\documentclass[twoside,11pt]{article}

\usepackage{amsthm}
\usepackage[abbrvbib]{jmlr2e}

\usepackage{amssymb}

\usepackage{float}

\usepackage{varwidth}

\usepackage{mathtools}

\usepackage{graphicx}

\usepackage{xcolor}

\newcommand{\tmop}[1]{\ensuremath{\operatorname{#1}}}

\newcommand{\Real}{\mathbb{R}} 
\DeclareMathOperator*{\argmin}{arg\,min}
\DeclareMathOperator{\rank}{rank}

\DeclareMathOperator{\diag}{diag}

\DeclareMathOperator{\KL}{KL}
\DeclareMathOperator{\var}{Var}
\DeclareMathOperator{\cov}{Cov}

\DeclareMathOperator{\tr}{tr}

\newcommand{\nonlin}{\phi}
\newcommand{\ex}{\operatorname{E}}

\newcommand{\Normal}{\mathcal{N}}

\newcommand{\bigO}{\mathcal{O}}

\newcommand{\B}{B} 

\newcommand{\ep}[1]{\,\text{#1}}

\newtheorem*{theorem*}{Theorem}
\newtheorem*{corollary*}{Corollary}

\usepackage{lastpage}
\jmlrheading{21}{2020}{1-\pageref{LastPage}}{11/17; Revised
6/20}{8/20}{17-678}{James Martens}
\ShortHeadings{New Insights and Perspectives on the Natural Gradient Method}{Martens}

\editor{L\'{e}on Bottou}

\begin{document}

\title{New Insights and Perspectives on the Natural Gradient Method}
%\title[New Insights and Perspectives on the Natural Gradient Method]{New Insights and Perspectives on the Natural Gradient Method}

%\author{James Martens\thanks{james.martens@gmail.com} }
%\affil{Department of Computer Science, University of Toronto}
%\affil{DeepMind, London}
%%\renewcommand\Authands{ and }
%%\date{Updated: \today}
%\date{}

\author{\name James Martens \email james.martens@gmail.com \\
      %\addr Department of Computer Science\\
      %University of Toronto\\
      %Toronto, Ontario, Canada \\
      %and \\
      \addr DeepMind \\
      London, United Kingdom
      }

\maketitle

%\ShortHeadings{New Insights and Perspectives on the Natural Gradient Method}{Martens}
%\firstpageno{1}

\begin{abstract}%
Natural gradient descent is an optimization method traditionally motivated from the perspective of information geometry, and works well for many applications as an alternative to stochastic gradient descent. In this paper we critically analyze this method and its properties, and show how it can be viewed as a type of 2nd-order optimization method, with the Fisher information matrix acting as a substitute for the Hessian. In many important cases, the Fisher information matrix is shown to be equivalent to the Generalized Gauss-Newton matrix, which both approximates the Hessian, but also has certain properties that favor its use over the Hessian. This perspective turns out to have significant implications for the design of a practical and robust natural gradient optimizer, as it motivates the use of techniques like trust regions and Tikhonov regularization. Additionally, we make a series of contributions to the understanding of natural gradient and 2nd-order methods, including: a thorough analysis of the convergence speed of stochastic natural gradient descent (and more general stochastic 2nd-order methods) as applied to convex quadratics, a critical examination of the oft-used ``empirical" approximation of the Fisher matrix, and an analysis of the (approximate) parameterization invariance property possessed by natural gradient methods (which we show also holds for certain other curvature matrices, but notably not the Hessian).

\end{abstract}

\begin{keywords}
  natural gradient methods, 2nd-order optimization, neural networks, convergence rate, parameterization invariance
\end{keywords}

\newpage
\setcounter{tocdepth}{1}
\tableofcontents
\newpage

\section{Introduction and Overview}

The natural gradient descent approach, pioneered by Amari and collaborators \citep[e.g.][]{natural_efficient}, is a popular alternative to traditional gradient descent methods which has received a lot of attention over the past several decades, motivating many new and related approaches.  It has been successfully applied to a variety of problems such as blind source separation \citep{amari1998adaptive}, reinforcement learning \citep{peters2008natural}, and neural network training \citep[e.g.][]{ng_adaptive,K-FAC,desjardins2015natural}.  

Natural gradient descent is generally applicable to the optimization of probabilistic models\footnote{This includes neural networks, which can be cast as conditional models.}, and involves the use of the so-called ``natural gradient", which is defined as the gradient times the inverse of the model's Fisher information matrix (aka ``the Fisher" ; see \textbf{Section \ref{sec:nat_grad}}), in place of the standard gradient.  In many applications, natural gradient descent seems to require far fewer iterations than gradient descent, making it a potentially attractive alternative method.  Unfortunately, for models with very many parameters such as large neural networks, computing the natural gradient is impractical due to the extreme size of the Fisher matrix. This problem can be addressed through the use of various approximations to the Fisher \citep[e.g][]{TONGA, ollivier2015riemannian, FANG, K-FAC} that are designed to be easier to compute, to store and finally to invert, than the exact Fisher.

Natural gradient descent is classically motivated as a way of implementing steepest descent\footnote{``Steepest descent" is a common synonym for gradient descent, which emphasizes the interpretation of the gradient as being the direction that descends down the loss surface most ``steeply".  Here ``steepness" is measured as the amount loss reduction per unit of distance traveled, where distance is measured according to some given metric. For standard gradient descent this metric is the Euclidean distance on the default parameter space.} in the space of realizable distributions\footnote{``Realizable distributions" means distributions which correspond to some setting of the model's parameters.} instead of the space of parameters, where distance in the distribution space is measured with a special ``Riemannian metric" \citep{amari2007methods}.  This metric depends only on the properties of the distributions themselves and not their parameters, and in particular is defined so that it approximates the square root of the Kullback–Leibler divergence within small neighborhoods.  Under this interpretation (discussed in detail in \textbf{Section \ref{sec:geom}}), natural gradient descent is invariant to any smooth and invertible reparameterization of the model, putting it in stark contrast to gradient descent, whose performance is parameterization dependent.  

In practice however, natural gradient descent still operates within the default parameter space, and works by computing directions in the space of distributions and then translating them back to the default space before taking a step. Because of this, the above discussed interpretation breaks down unless the step-size becomes arbitrarily small, and as discussed in \textbf{Section \ref{sec:role_of_damping}}, this breakdown has important implications for designing a natural gradient method that can work well in practice.  Another problem with this interpretation is that it doesn't provide any obvious reason why a step of natural gradient descent should make more progress reducing the objective than a step of standard gradient descent (assuming well-chosen step-sizes for both). Moreover, given a large step-size one also loses the parameterization invariance property of the natural gradient method, although it will still hold \emph{approximately} under certain conditions which are described in \textbf{Section \ref{sec:param_invar}}.

In \textbf{Section \ref{sec:role_of_damping}} we argue for an alternative view of natural gradient descent: as a type of 2nd-order method\footnote{By ``2nd-order method" we mean any iterative optimization method which generates updates as the (possibly approximate) solution of a non-trivial local quadratic model of the objective function. This extends well beyond the classical Newton's method, and includes approaches like (L-)BFGS, and methods based on the Gauss-Newton matrix.} which utilizes the Fisher as an alternative to the Hessian. As discussed in \textbf{Section \ref{sec:2nd-order}}, 2nd-order methods work by forming a local quadratic approximation to the objective around the current iterate, and produce the next iterate by optimizing this approximation within some region where the approximation is believed to be accurate.  According to this view, natural gradient descent ought to make more progress per step than gradient descent because it uses a local quadratic model/approximation of the objective function which is more detailed (which allows it to be less conservative) than the one implicitly used by gradient descent. 

In support of this view is the fact that the Fisher can be cast as an approximation of the Hessian in at least two different ways (provided the objective has the form discussed in \textbf{Section \ref{sec:stat_learn}}).  First, as discussed in \textbf{Section \ref{sec:nat_grad}}, it corresponds to the expected Hessian of the loss under the model's distribution over predicted outputs (instead of the usual empirical one used to compute the exact Hessian).  Second, as we establish in \textbf{Section \ref{sec:connections_to_GGN}}, it is very often equivalent to the so-called ``Generalized Gauss-Newton matrix" (GGN) (discussed in \textbf{Section \ref{sec:GGN}}), a generalization of the classical Gauss-Newton matrix \citep[e.g.][]{dennis_text, ortega_text, nocedal_book}, which is a popular alternative/approximation to the Hessian that has been used in various practical 2nd-order optimization methods designed specifically for neural networks \citep[e.g][]{schraudolph,HF,KSD}, and which may actually be a \emph{better} choice than the Hessian in the context of neural net training (see \textbf{Section \ref{sec:GGN_speculation}}).

Viewing natural gradient descent as a 2nd-order method is also \emph{prescriptive}, since it suggests the use of various damping/regularization techniques often used in the optimization literature to account for the limited accuracy of local quadratic approximations (especially over long distances). Indeed, such techniques have been successfully applied in 2nd-order methods designed for neural networks \citep[e.g.][]{HF, K-FAC}, where they proved crucial in achieving fast and robust performance in practice. And before that have had a long history of application in the context of practical non-linear regression procedures \citep{tikhonov1943stability, levenberg1944method, marquardt1963algorithm, levenberg_marquardt}.

The ``empirical Fisher", which is discussed in \textbf{Section \ref{sec:emp_fish}}, is an approximation to the Fisher whose computation is easier to implement in practice using standard automatic-differentiation libraries. The empirical Fisher differs from the usual Fisher in subtle but important ways, which as we show in \textbf{Section \ref{sec:empfish_comparison}}, make it considerably less useful as an approximation to the Fisher, or as a curvature matrix to be used in 2nd-order methods.  Using the empirical Fisher also breaks some of the theory justifying natural gradient descent, although it nonetheless preserves its (approximate) parameterization invariance (as we show in \textbf{Section \ref{sec:param_invar}}).  Despite these objections, the empirical Fisher has been used in many approaches, such as TONGA \citep{TONGA}, and the recent spate of methods that use the diagonal of this matrix such as RMSprop \citep{RMSprop} and Adam \citep{adam} (which we examine in \textbf{Section \ref{sec:diag_empfish_methods}}).

A well-known and oft quoted result about stochastic natural gradient descent is that it is asymptotically ``Fisher efficient" \citep{natural_efficient}.  Roughly speaking, this means that it provides an asymptotically unbiased estimate of the parameters with the lowest possible variance among all unbiased estimators (that see the same amount of data), thus achieving the best possible expected objective function value. Unfortunately, as discussed in \textbf{Section \ref{sec:Fisher_efficient}}, this result comes with several important caveats which significantly limit its applicability.  Moreover, even when it is applicable, it only provides an \emph{asymptotically} accurate characterization of the method, which may not usefully describe its behavior given a finite number of iterations.

To address these issues we build on the work of \citet{Murata} in \textbf{Section \ref{sec:speed_analysis}} and \textbf{Section \ref{sec:averaging}} to develop a more powerful convergence theory for stochastic 2nd-order methods (including natural gradient descent) as applied to convex \emph{quadratic objectives}. Our results provide a more precise expression for the convergence speed of such methods than existing results do\footnote{Alternate versions of several of our results have appeared in the literature before \citep[e.g.][]{Polyak_averaging, sgd-qn, moulines2011non}. These formulations tend to be more general than ours (we restrict to the quadratic case), but also less precise and interpretable. In particular, previous results tend to omit asymptotically negligible terms that are nonetheless important pre-asymptotically. Or when they do include these terms, their bounds are simultaneously looser and more complicated than our own, perhaps owing to their increased generality. We discuss connections to some prior work on convergence bounds in Sections \ref{sec:related_bound_results_1} and \ref{sec:related_bound_results_2}.}, and properly account for the effect of the starting point.  And as we discuss in \textbf{Section \ref{sec:consequences_asymptotic_main}} and \textbf{Section \ref{sec:consequences_averaging}} they imply various interesting consequences about the relative performance of various 1st and 2nd-order stochastic optimization methods. Perhaps the most interesting conclusion of this analysis is that, while stochastic gradient descent with Polyak-style parameter averaging achieves the same \emph{asymptotic} convergence speed as stochastic natural gradient descent (and is thus also ``Fisher efficient", as was first shown by \citet{Polyak_averaging}), stochastic 2nd-order methods can possess a much more favorable dependence on the starting point, which means that they can make much more progress given a limited iteration budget. Another interesting observation made in our analysis is that stochastic 2nd-order methods that use a decaying learning rate (of the form $\alpha_k = 1/(k+a+1)$) can, for certain problems, achieve an asymptotic objective function value that is better than that achieved in the same number of iterations by stochastic gradient descent (with a similar decaying learning rate), by a large constant factor. 

Unfortunately, these convergence theory results fail to explain why 2nd-order optimization with the GGN/Fisher works so much better than classical 2nd-order schemes based on the Hessian \emph{for neural network training} \citep{schraudolph, HF, KSD}. In \textbf{Section \ref{sec:future_work}} we propose several important open questions in this direction that we leave for future work.

\newpage
\section*{Table of Notation}
%\vspace{-0.1in}
%{\renewcommand{\arraystretch}{0.9}
\begin{table}[H]
\label{tab:notation}
\centering
\begin{tabular}{|l|l|}
\hline
Notation & Description \\
\hline
$[v]_i$ & $i$-th entry of a vector $v$\\
$[A]_{i,j}$ & $(i,j)$-th entry a matrix $A$\\
%$\diag(A)$ & vector consisting of the diagonal of the matrix $A$ \\
%$\diag(v)$ & diagonal matrix $A$ satisfying $[A]_{i,i} = [v]_i$ \\
%${\bf 1}_m$ & vector of length $m$ whose entries are 1\\
%$\sq(\cdot)$ & element-wise square of a vector or a matrix  \\ 
%$\vecc(A)$ & vectorization of a matrix $A$ \\
%$\gamma'$ & element-wise first derivative of a function $\gamma$\\
%$\gamma''$ & element-wise second derivative of a function $\gamma$\\
$\nabla \gamma$ & gradient of a scalar function $\gamma$\\
$J_\gamma$ & Jacobian of a vector-valued function $\gamma$\\
$H_\gamma$ & Hessian of a scalar function $\gamma$ (typically taken with \\
& respect to $\theta$ unless otherwise specified) \\
$H$ & Hessian of the objective function $h$ w.r.t. $\theta$ (i.e. $H_h$) \\
$\theta$ & vector of all the network's parameters \\
$W_i$ & weight matrix at layer $i$ \\
$\phi_i$ & activation function for layer $i$ \\
$s_i$ & unit inputs at layer $i$ \\
$a_i$ & unit activities at layer $i$ \\
%$\bar{a}_i$ & unit activities at layer $i$ with a homogeneous coordinate of value $1$ appended \\
$\ell$ & number of layers \\
$m$ & dimension of the network's output $f(x,\theta)$ \\
$m_i$ & number of units in $i$-th layer of the network  \\
%$W^{(\Delta t)}_{i,j}$ & matrix of weights from layer $j$ to $i$ across a time difference $\Delta t$ (for RNNs) \\
%$T$ & number of timesteps (for RNNS) \\
$f(x,\theta)$ & function mapping the neural network's inputs to its output\\
$L(y,z)$ & loss function \\
$h$ & objective function \\
%$h_i$ & The objective function on case $i$ \\
$S$ & training set \\
$k$ & current iteration \\
%$\theta_k$ & The parameter value produced by the $k$-th HF iteration \\
$n$ & dimension of $\theta$ \\
%$\delta_k$ & candidate update produced at iteration $k$\\
%$\Mth$ & local quadratic approximation of $f$ at $\theta_k$ \\
$M_k(\delta)$ & local quadratic approximation of $h$ at $\theta_k$\\
%$\hat M_k$ & A ``damped" version of the above \\ 
%$\B_k$ & The curvature matrix associated with $\Mth$ \\
%$\hat \B_k$ & The curvature matrix associated with $\hat M_k$ \\
%$\rho$ & The reduction ratio 
%$\displaystyle\frac{h(\theta_{k+1})-h(\theta_{k})}{M_{k}(\delta_k)}$ \\
%$p$ & The $p\equiv P(\theta)$ is a vector of predictions \\
%$D$ & damping matrix \\
%$P$ & preconditioning matrix \\
%$C$ & The matrix square root of $P$ \\
%$S$ & The training set \\
%$K_i(A,r_0)$ & subspace $\textrm{span}\{r_0,Ar_0,\ldots,A^{i-1}r_0 \}$\\
%$T$ & The number of time-steps of an RNN \\
$\lambda$ & strength constant for penalty-based damping  \\
$\lambda_j(A)$ & $j$-th largest eigenvalue of a symmetric matrix $A$ \\
%$x \odot y$ & element-wise (Hadamard) product between $x$ and $y$ \\
%$\deriv v$ & $\displaystyle \derivfrac{L( y, f(x,\theta) ) }{v}$ \\
%$dL( y, f(x,\theta) ) / dv$ \\
$G$ & generalized Gauss-Newton matrix (GGN) \\
$P_{x,y}(\theta)$ & model's distribution \\
$Q_{x,y}$ & data distribution \\
$\hat{Q}_{x,y}$ & training/empirical distribution \\
$R_{y|z}$ & predictive distribution used at network's output (so $P_{y|x}(\theta) = R_{y|f(x,\theta)}$)\\
$p$, $q$, $r$ & density functions associated with above $P$, $Q$, and $R$ (resp.) \\
$F$ & Fisher information matrix (typically associated with $P_{x,y}$) \\
$F_D$ & Fisher information matrix associated with parameterized distribution $D$\\
$\tilde{\nabla} h$ & The natural gradient for objective $h$.\\
\hline
\end{tabular}
\caption{\small A table listing some of the notation used throughout this document.}
\end{table}
%}

\newpage

\section{Neural Networks}
\label{sec:neural_networks}

%[[removed]] For the purposes of this document we will be interested in two basic types of neural networks:  feed-forward networks and recurrent networks.  Common to both is the basic idea of a collection of generic processing units or ``neurons" that compute scalar functions of their input, often called the ``activation function" or ``non-linearity", which act as inputs to other units according to ``connections", and are networked together to perform a complex sequential computation, similar to a classical circuit.  Typically the input to a given unit is computed according to some affine (linear + constant) function of its input units.  It is the determination of these affine functions, which may be parameterized by a weight vector and bias constant, that is usually the sole objective of training/learning, although sometimes the activation functions and other aspects of the network may be parameterized.

%The input/output relationships between the units of the neural network may by represented with a directed graph, where nodes in the graph represent units or collections of units, and directed edges represent the flow of information/input.  In most classical accounts, nodes represent single units and each incoming edge to a node corresponds to a particular entry of the weight vector parameterizing the unit's affine input function.   However we will adopt modern conventions here and relax these requirements so that nodes may represent entire groups of units, and edges simply indicate the flow of computation, similar to a computation graph.

Feed-forward neural networks are structured similarly to classical circuits.  They typically consist of a sequence of $\ell$ ``layers" of units, where each unit in a given layer receive inputs from the units in the previous layer, and computes an affine function of these followed by a scalar non-linear function called an ``activation function".  The input vector to the network, denoted by $x$, is given by the units of the first layer, which is called the ``input layer" (and is not counted towards the total $\ell$).  The output vector of the network is given by the units of the network's last layer (called the ``output layer"). The other layers are referred to as the network's ``hidden layers".  %Similarly to how deeper circuits can compute more interesting functions more efficiently, it is believed that 

Formally, given input $x \in \Real^{m_0}$, and parameters $\theta \in \Real^n$ which determine weight matrices $W_1 \in \Real^{m_1\times m_0}, W_2 \in \Real^{m_2\times m_1}, \ldots, W_\ell \in \Real^{m_\ell\times m_{\ell-1}}$ and biases $b_1 \in \Real^{m_1}, b_2 \in \Real^{m_2}, \ldots, b_\ell \in \Real^{m_\ell}$, the network computes its output $f(x, \theta) = a_\ell$ according to
\begin{align*}
s_i &= W_i a_{i-1} + b_i \\
a_i &= \nonlin_i(s_i) \ep{,}
\end{align*}
where $a_0 = x$.  %Here $\nonlin_i(\cdot)$ is some nonlinear differentiable function,
Here, $a_i$ is the vector of values (``activities") of the network's $i$-th layer, and $\nonlin_i(\cdot)$ is the vector-valued non-linear ``activation function" computed at layer $i$, and is often given by some simple scalar function applied coordinate-wise.

Note that most of the results discussed in this document will apply to the more general setting where $f(x,\theta)$ is an arbitrary differentiable function (in both $x$ and $\theta$).

\section{Supervised Learning Framework}
\label{sec:general_learning}

The goal of optimization in the context of supervised learning is to find some setting of $\theta$ so that, for each input $x$ in the training set, the output of the network (which we will sometimes call its ``prediction") matches the given target outputs as closely as possible, as measured by some loss.  In particular, given a training set $S$ consisting of pairs $(x,y)$, we wish to minimize the objective function
\begin{align}
\label{eqn:objective_def}
h(\theta) \equiv \frac{1}{|S|}\sum_{(x,y) \in S} L( y, f(x, \theta) ) ,
\end{align}
where $L(y, z)$ is a ``loss function" which measures the disagreement between $y$ and $z$.

The prediction $f(x, \theta)$ may be a guess for $y$, in which case $L$ might measure the inaccuracy of this guess (e.g. using the familiar squared error $\frac{1}{2}\| y - z \|^2$).  Or $f(x, \theta)$ could encode the parameters of some simple predictive distribution. For example, $f(x, \theta)$ could be the set of probabilities which parameterize a multinomial distribution over the possible discrete values of $y$, with $L(y,f(x, \theta))$ being the negative log probability of $y$ under this distribution. %[[give example notation - especially since we end up using cross-entropy error later on??  Also discuss distinction between ``soft" versions, which allow specifications of distributions over y, and "hard" version which insist that only a sample of y (a discrete 1 of k value) is given].

\section{KL Divergence Objectives}
\label{sec:stat_learn}

%[[\textbf{Need brief intro here.  Discuss aims of this section and briefly the main conclusions we reach regarding the connection to the GGN, and interpretations as an approximate Newton method which imply the use of the damping}]]

%In order to properly describe and analyze the natural gradient approach of \citet{natural_efficient} we must first adopt a statistical learning formulation and define the corresponding notation.  

The natural gradient method of \citet{natural_efficient} can potentially be applied to any objective function which measures the performance of some statistical model.  However, it enjoys richer theoretical properties when applied to objective functions based on the KL divergence between the model's distribution and the target distribution, or certain approximations/surrogates of these.  %This turns out to be due to the important relationship that the natural gradient has with this particular measure of ``distance" of a distribution, which we discuss in Section \ref{sec:geom}.
 In this section we will establish the basic notation and properties of these objective functions, and discuss the various ways in which they can be formulated.  Each of these formulations will be analogous to a particular formulation of the Fisher information matrix and natural gradient (as defined in Section \ref{sec:nat_grad}), which will differ in subtle but important ways.

In the idealized setting, input vectors $x$ are drawn independently from a \emph{target} distribution $Q_x$ with density function $q(x)$, and the corresponding (target) outputs $y$ from a \emph{conditional target} distribution $Q_{y|x}$ with density function $q(y|x)$. 

We define the goal of learning as the minimization of the KL divergence from the target joint distribution $Q_{x,y}$, whose density is $q(y,x) = q(y|x)q(x)$, to the learned distribution $P_{x,y}(\theta)$, whose density is $p(x,y|\theta) = p(y|x, \theta) q(x)$.  (Note that the second $q(x)$ is not a typo here, since we are not learning the distribution over $x$, only the conditional distribution of $y$ given $x$.)  Our objective function is thus
\begin{align*}
\KL( Q_{x,y} \| P_{x,y}(\theta) ) = \int q(x,y) \log \frac{q(x,y)}{p(x,y|\theta)} dxdy \ep{.}
\end{align*}
This is equivalent to the expected KL divergence $\ex_{Q_x}[ \KL( Q_{y|x} \| P_{y|x}(\theta) ) ]$ as can be seen by
\begin{align}
    \KL( Q_{x,y} \| P_{x,y}(\theta) ) \nonumber &= \int q(x,y) \log \frac{q(y|x) q(x)}{p(y|x,\theta) q(x)}dxdy \nonumber \\
    &= \int q(x) \int q(y|x) \log \frac{q(y|x)}{p(y|x,\theta)}dy dx \nonumber \\
    &= \ex_{Q_x}\left[ \KL( Q_{y|x} \| P_{y|x}(\theta) ) \right] \label{eqn:KL_obj_cond}
\end{align}

It is often the case that we only have samples from $Q_x$ and no direct knowledge of its density function.  Or the expectation w.r.t.~$Q_x$ in eqn.~\ref{eqn:KL_obj_cond} may be too difficult to compute.  In such cases, we can substitute an empirical \emph{training} distribution $\hat{Q}_x$ in for $Q_x$, which is given by a set $S_x$ of samples from $Q_x$.  This results in the objective
\begin{align*}
\ex_{\hat{Q}_x}\left[ \KL( Q_{y|x} \| P_{y|x}(\theta) ) \right] = \frac{1}{|S|} \sum_{x \in S_x} \KL( Q_{y|x} \| P_{y|x}(\theta) ) \ep{.}
\end{align*}

Provided that $q(y|x)$ is known for each $x$ in $S_x$, and that $\KL( Q_{y|x} \| P_{y|x}(\theta) )$ can be efficiently computed, we can use the above expression as our objective. Otherwise, as is often the case, we might only have access to a single sample $y$ from $Q_{y|x}$ for each $x \in S_x$, giving an empirical \emph{training} distribution $\hat{Q}_{y|x}$. Substituting this in for $Q_{y|x}$ gives the objective function
\begin{align*}
\ex_{\hat{Q}_x}\left[ \KL( \hat{Q}_{y|x} \| P_{y|x}(\theta) ) \right] \propto -\frac{1}{|S|} \sum_{(x,y) \in S} \log {p(y|x,\theta)} \ep{,}
\end{align*}
where we have extended $S_x$ to a set $S$ of the $(x,y)$ pairs (which agrees with how $S$ was defined in Section \ref{sec:general_learning}). Here, the proportionality is with respect to $\theta$, and it hides an additive constant which is technically infinity\footnote{The constant corresponds to the differential entropy of the Dirac delta distribution centered at $y$. One can think of this as approaching infinity under the limit-based definition of the Dirac.}. This is \emph{effectively} the same objective that is minimized in standard maximum likelihood learning.

This kind of objective function fits into the general supervised learning framework described in Section \ref{sec:general_learning} as follows.  We define the learned conditional distribution $P_{y|x}(\theta)$ to be the composition of the deterministic prediction function $f(x,\theta)$ (which may be a neural network), and an ``output" conditional distribution $R_{y|z}$ (with associated density function $r(y|z)$), so that 
\begin{align*}
P_{y|x}(\theta) = R_{y|f(x,\theta)} \ep{.}
\end{align*}
We then define the loss function as $L(y,z) = -\log r(y|z)$.

Given a loss function $L$ which is not explicitly defined this way one can typically still find a corresponding $R$ to make the definition apply.  In particular, if $\exp(-L (y, z))$ has the same finite integral w.r.t.~$y$ for each $z$, then one can define $R$ by taking $r(y | z) \propto \exp(-L (y, z))$, where the proportion is w.r.t.~both $y$ and $z$.

\section{Various Definitions of the Natural Gradient and the Fisher Information Matrix}
\label{sec:nat_grad}

%Adopting this statistical learning interpretation allows us to make use of natural gradient methods \citep{natural_efficient}.

The Fisher information matrix $F$ of $P_{x,y}(\theta)$ w.r.t.~$\theta$ (aka the ``Fisher") is given by
\begin{align}
\label{eqn:F_expression_1}
F &= \ex_{P_{x,y}}\left[ \nabla \log p(x,y|\theta) \nabla \log p(x,y|\theta)^\top \right] \\
\label{eqn:F_expression_2}
&= -\ex_{P_{x,y}}\left[ H_{\log p(x,y | \theta)} \right ] \ep{.}
\end{align}
where gradients and Hessians are taken w.r.t.~$\theta$.  It can be immediately seen from the first of these expressions for $F$ that it is positive semi-definite (PSD) (since it's the expectation of something which is trivially PSD, a vector outer-product).  And from the second expression we can see that it also has the interpretation of being the negative expected Hessian of $\log p(x,y|\theta)$.

The usual definition of the natural gradient \citep{natural_efficient} which appears in the literature is
\begin{align*}
\tilde{\nabla} h = F^{-1} \nabla h \ep{,}
\end{align*}
where $F$ is the Fisher and $h$ is the objective function.

Because $p(x,y|\theta) = p(y|x, \theta) q(x)$, where $q(x)$ doesn't depend on $\theta$, we have 
\begin{align*}
\nabla \log p(x,y | \theta) = \nabla \log p(y | x, \theta) + \nabla \log q(x) = \nabla \log p(y | x, \theta) \ep{,}
\end{align*} 
and so $F$ can also be written as the expectation (w.r.t.~$Q_x$) of the Fisher information matrix of $P_{y|x}(\theta)$ as follows:
\begin{align*}
F &= \ex_{Q_x}\left[\ex_{P_{y|x}}\left[ \nabla \log p(y|x,\theta) \nabla \log p(y|x,\theta)^\top \right]\right] \quad \: \mbox{or} \quad \:  F = -\ex_{Q_x}\left[\ex_{P_{y|x}}\left[ H_{\log p(y | x, \theta)} \right ]\right] \ep{.}
\end{align*}

In \citet{natural_efficient}, this version of $F$ is computed explicitly for a basic perceptron model (basically a neural network with 0 hidden layers) in the case where $Q_x = N(0,I)$. However, in practice the real $q(x)$ may not be directly available, or it may be difficult to integrate $H_{\log p(y | x, \theta)}$ over $Q_x$.  For example, the conditional Hessian $H_{\log p(y | x, \theta)}$ corresponding to a multi-layer neural network may be far too complicated to be analytically integrated, even for a very simple $Q_x$.  In such situations $Q_x$ may be replaced with its empirical version $\hat{Q}_x$, giving
\begin{align*}
F = \frac{1}{|S|} \sum_{x \in S_x} \ex_{P_{y|x}} \left [\nabla \log p(y|x,\theta) \nabla \log p(y|x,\theta)^\top \right] \quad \: \mbox{or} \quad \: F = -\frac{1}{|S|} \sum_{x \in S_x} \ex_{P_{y|x}}\left[H_{\log p(y | x, \theta)}\right] \ep{.}
\end{align*}
This is the version of $F$ considered in \citet{ng_adaptive}.

From these expressions we can see that when $L(y,z) = -\log r(y|z)$ (as in Section \ref{sec:stat_learn}), the Fisher has the interpretation of being the expectation under $P_{x,y}$ of the Hessian of $L(y,f(x,\theta))$:
\begin{align*}
    F = \frac{1}{|S|} \sum_{x \in S_x} \ex_{P_{y|x}}\left[H_{L(y,f(x,\theta))}\right] \ep{.}
\end{align*}
Meanwhile, the Hessian $H$ of $h$ is also given by the expected value of the Hessian of $L(y,f(x,\theta))$, except under the distribution $\hat{Q}_{x,y}$ instead of $P_{x,y}$ (where $\hat{Q}_{x,y}$ is given by the density function $\hat{q}(x,y) = \hat{q}(y|x) \hat{q}(x)$).  In other words
\begin{align*}
    H = \frac{1}{|S|} \sum_{x \in S_x} \ex_{\hat{Q}_{x,y}}\left[H_{L(y,f(x,\theta))}\right] \ep{.}
\end{align*}
Thus $F$ and $H$ can be seen as approximations of each other in some sense. % Moreover, we can see that by computing $F$ using $\hat{Q}_x$ instead of $Q_x$ as described above, the quality of this approximation will arguably be better since $H$ is also computed using $\hat{Q}_x$ instead of $Q_x$.

\section{Geometric Interpretation}

\label{sec:geom}

The negative gradient $-\nabla h$ can be interpreted as the steepest descent direction for $h$ in the sense that it yields the greatest instantaneous rate of reduction in $h$ per unit of change in $\theta$, where change in $\theta$ is measured using the standard Euclidean norm $\|\cdot\|$. More formally we have
\begin{align*}
\frac{-\nabla h}{\|\nabla h\|} = \lim_{\epsilon \to 0} \frac{1}{\epsilon} \argmin_{d : \|d\| \leq \epsilon} h(\theta + d) \ep{.}
\end{align*}
This interpretation highlights the strong dependence of the gradient on the Euclidean geometry of the parameter space (as defined by the norm $\|\cdot\|$).  

One way to motivate the natural gradient is to show that it (or more precisely its negation) can be viewed as a steepest descent direction, much like the negative gradient can be, except with respect to a metric that is intrinsic to the distributions being modeled, as opposed to the default Euclidean metric which is tied to the given parameterization.  In particular, the natural gradient can be derived by adapting the steepest descent formulation to use an alternative definition of (local) distance based on the ``information geometry" \citep{information_geom} of the space of probability distributions.  The particular distance function\footnote{Note that this is not a formal ``distance" function in the usual sense since it is not symmetric.} which gives rise to the natural gradient turns out to be
\begin{align*}
\KL(P_{x,y}(\theta + d) \| P_{x,y}(\theta) ) \ep{.}
\end{align*}
%Or equivalently
%\begin{align*}
%( \ex_{Q_x}[ \KL( P_{y|x}(\theta + d) \| P_{y|x}(\theta) ) ] )^{1/2}
%\end{align*}

To formalize this, one can use the well-known connection between the KL divergence and the Fisher, given by the Taylor series approximation
\begin{align*}
\KL( P_{x,y}(\theta + d) \| P_{x,y}(\theta) ) = \frac{1}{2} d^\top F d + O(d^3) \ep{,}
\end{align*}
where ``$O(d^3)$" is short-hand to mean terms that are order 3 or higher in the entries of $d$.  Thus, $F$ defines the local quadratic approximation of this distance, and so gives the mechanism of \emph{local} translation between the geometry of the space of distributions, and that of the original parameter space with its default Euclidean geometry.  %[talk about $F$ defining a metric here (or maybe later)?]  

To make use of this connection, \citet{IGO} proves for general PSD matrices $A$ that
\begin{align*}
\frac{- A^{-1} \nabla h}{\| \nabla h \|_{A^{-1}} } &= \lim_{\epsilon \to 0} \frac{1}{\epsilon} \argmin_{d : \|d\|_A \leq \epsilon} h(\theta + d) \ep{,}% \\
%&= \lim_{\epsilon \to 0} \frac{1}{\epsilon} \argmin_{d : \|d\|_A \leq 1} h(\theta + \epsilon d) \ep{,}
\end{align*}
where the notation $\|v\|_B$ is defined by $\|v\|_B = \sqrt{ v^\top B v}$. 
Taking $A = \frac{1}{2} F$ and using the above Taylor series approximation to establish that
\begin{align*}
\KL( P_{x,y}(\theta + d) || P_{x,y}(\theta) ) \to \frac{1}{2} d^\top F d = \frac{1}{2} \|d\|^2_{F}   
\end{align*}
as $\epsilon \to 0$, \citep{IGO} then proceed to show that
\begin{align*}
-\sqrt{2} \frac{\tilde{\nabla} h} {\| \nabla h \|_{F^{-1}} } = \lim_{\epsilon \to 0} \frac{1}{\epsilon} \argmin_{d \: : \: \KL( P_{x,y}(\theta + d) \| P_{x,y}(\theta)  ) \leq \epsilon^2} h(\theta + d) \ep{,}
\end{align*}
(where we recall the notation $\tilde{\nabla} h = F^{-1} \nabla h$).
%where the denominator can also be written as $\sqrt{\nabla h^\top \tilde{\nabla} h}$ [[\textbf{why is this important?}]].  

Thus the negative natural gradient is indeed the steepest descent direction in the space of distributions where distance is measured in small local neighborhoods by the KL divergence.  %While this might seem impossible since the KL divergence is in general not symmetric in its two arguments, it turns out that $\KL(P_{x,y}(\theta + d) \| P_{x,y}(\theta) )$ is locally/asymptotically symmetric as $d$ goes to zero, and so will be (approximately) symmetric in a local neighborhood \footnote{This follows from the fact the second order term of the Taylor series of $\KL(P_{x,y}(\theta) \| P_{x,y}(\theta + d) )$ is \emph{also} given by $\frac{1}{2} d^\top F d$.}.  %, and so the natural gradient can be just as easily derived as above using this function as a measure of distance. 

Note that both $F$ and $\tilde{\nabla} h$ are defined in terms of the standard basis in $\theta$-space, and so obviously depend on the parameterization of $h$.  But the KL divergence does not, and instead only depends on the form of the predictive distribution $P_{y|x}$.  Thus, the direction in distribution space defined implicitly by $\tilde{\nabla} h$ will be invariant to our choice of parameterization (whereas the direction defined by $\nabla h$ will not be, in general).  

By using the smoothly varying PSD matrix $F$ to locally define a metric tensor at every point in parameter space, a Riemannian manifold can be generated over the space of distributions.  Note that the associated metric of this space won't be the square root of the KL divergence (this isn't even a valid metric), although it will be ``locally equivalent" to it in the sense that the two functions will approximate each other within a small local neighborhood.
 %It can be shown [recent Amari ref] that a broad class distance functions for probability distributions (called f-divergences) all generate the same Riemannian manifold structure.

%%REMOVED based on email with Sebastian Goldt
%When we use the KL divergence objective function discussed in Section \ref{sec:stat_learn}, the geometric interpretation of the natural gradient becomes particularly nice. This is because the objective function will locally measure distance in distribution space the same way that it is locally measured in the steepest descent interpretation of the natural gradient.  In this view, smoothly following the natural gradient is equivalent to following the geodesic path in the Riemannian manifold from the current distribution towards the target distribution (which may never be reached due to the presence of singularities).% or the non-realizability of the target distribution by the model).  %As explored by Amari and collaborators \citep{information_geom}, following this path can have important advantages over following the path given by the standard gradient, in addition to merely being invariant to the original parameterization.

\section{2nd-order Optimization}

\label{sec:2nd-order}

The basic idea in 2nd-order optimization is to compute the update $\delta$ to $\theta \in \Real^n$ by minimizing some local quadratic approximation or ``model" $M_k(\delta)$ of $h(\theta_k + \delta)$ centered around the current iterate $\theta_k$.  That is, we compute $\delta_k^* = \argmin_{\delta} M_k(\delta)$ and then update $\theta$ according to $\theta_{k+1} = \theta_k + \alpha_k \delta_k^*$, where $M_k(\delta)$ is defined by
\begin{align*}
M_k(\delta) = \frac{1}{2} \delta^\top \B_k \delta + \nabla h(\theta_k)^\top \delta + h(\theta_k) \ep{,}
\end{align*}
and where $\B_k \in \Real^{n \times n}$ is the ``curvature matrix", which is symmetric.  The ``sub-problem" of minimizing $M_k(\delta)$ can be solved exactly by solving the $n \times n$ dimensional linear system $\B_k\delta = -\nabla h$, whose solution is $\delta^* = -\B_k^{-1} \nabla h$ when $\B_k$ is positive definite.

Gradient descent, the canonical 1st-order method, can be viewed in the framework of 2nd-order methods as making the choice $B_k = \beta I$ for some $\beta$, resulting in the update $\delta_k^* = -\frac{1}{\beta} \nabla h(\theta_k)$.  In the case where $h$ is convex and Lipschitz-smooth\footnote{By this we mean that $\|\nabla h(\theta) - \nabla h(\theta')\| \leq \mathcal{L} \|\theta - \theta'\|$ for all $\theta$ and $\theta'$.} with constant $\mathcal{L}$, a safe/conservative choice that will ensure convergence with $\alpha_k = 1$ is $\beta = \mathcal{L}$ \citep[e.g.][]{nesterov2013introductory}. The intuition behind this choice is that $\B$ will act as a \emph{global} upper bound on the curvature of $h$, in the sense that $\B_k = \mathcal{L} I \succeq H(\theta)$\footnote{Here we define $A \succeq C$ to mean that $A-C$ is PSD.} for all $\theta$, so that $\delta_k^*$ never extends past the point that would be safe in the worst-case scenario where the curvature is at its upper bound $\mathcal{L}$ the entire way along $\delta^*$. More concretely, one can show that given this choice of $\beta$, $M_k(\delta)$ upper bounds $h(\theta_k + \delta)$, and will therefore never predict a reduction in $h(\theta_k + \delta)$ where there is actually a sharp increase (e.g. due to $h$ curving unexpectedly upward on the path from $\theta_k$ to $\theta_k + \delta$).  Minimizing $M_k(\delta)$ is therefore guaranteed not to increase $h(\theta_k + \delta)$ beyond the current value $h(\theta_k)$ since $M_k(0) = h(\theta_k)$.  But despite these nice properties, this choice will almost always overestimate the curvature in most directions, leading to updates that move unnecessarily slowly along directions of consistent low curvature. %Moreover, $\mathcal{L}$ is often difficult to precisely determine in practice.

While neural networks haven't been closely studied by optimization researchers until somewhat recently, many of the local optimization issues related to neural network learning can be seen as special cases of problems which arise more generally in continuous optimization.  For example, tightly coupled parameters with strong local dependencies, and large variations in scale along different directions in parameter space (which may arise due to the ``vanishing gradient" phenomenon \citep{vanishing}), are precisely the sorts of issues for which 2nd-order optimization is well suited.  Gradient descent on the other hand is well known to be very sensitive to such issues, and in order to avoid large oscillations and instability must use a learning rate which is inversely proportional to  $\mathcal{L}$.  2nd-order optimization methods provide a much more powerful and elegant solution to the problem of variations in scale/curvature along different directions by selectively re-scaling the gradient along different eigen-directions of the curvature matrix $\B_k$ according to their associated curvature (eigenvalue), instead of employing a one-size-fits-all curvature estimate.

In the classical Newton's method we take $\B_k = H(\theta_k)$, in which case $M_k(\delta)$ becomes the 2nd-order Taylor-series approximation of $h$ centered at $\theta_k$. This choice gives us the most accurate \emph{local} model of the curvature possible, and allows for rapid exploration of low-curvature directions and thus faster convergence. 

Unfortunately, naive implementations of Newton's method can run into numerous problems when applied to neural network training objectives, such as $H$ being sometimes indefinite (and thus $M_k(\delta)$ being unbounded below in directions of negative curvature) and related issues of ``model trust", where the method implicitly trusts its own local quadratic model of the objective too much, causing it to propose very large updates that may actually increase the $h$. These problems are usually not encountered with first order methods, but only because they use a very conservative local model that is intrinsically incapable of generating large updates. Fortunately, using the Gauss-Newton approximation to the Hessian (as discussed in Section \ref{sec:GGN}), and/or applying various update damping/trust-region techniques (as discussed in Section \ref{sec:role_of_damping}), the issue of model trust issue in 2nd-order methods can be mostly overcome.

Another important issue preventing the naive application of 2nd-order methods to neural networks is the typically very high dimensionality of the parameter space ($n$), which prohibits the calculation/storage/inversion of the $n^2$-entry curvature matrix $\B_k$.  To address this, various approximate Newton methods have been developed within the optimization and machine learning communities.  These methods work by approximating $\B_k$ with something easier to compute/store/invert such as a low-rank or diagonal matrix, or by performing only approximate/incomplete optimization of $M_k(\delta)$. A survey of such methods is outside the scope of this report, but many good references and reviews are available \citep[e.g.][]{nocedal_book, fletcher2013practical}. \citet{martens_thesis} reviews these approaches specifically in the context of neural networks.

Finally, it is worth observing that the \emph{local} optimality of the Hessian-based 2nd-order Taylor series approximation to $h$ won't necessarily yield the fastest possible optimization procedure, as it is possible to imagine quadratic models that take a ``longer view" of the objective. (As an extreme example, given knowledge of a global minimizer $\theta^*$ of $h$, one could construct a quadratic model whose minimizer is exactly $\theta^*$ but which is a very poor local approximation to $h(\theta)$.)  It is possible that the Fisher might give rise to such quadratic models, which would help explain its observed superiority to the Hessian in neural network optimization \citep{schraudolph,HF,KSD}. We elaborate more on this speculative theory in Section \ref{sec:GGN_speculation}.

\section{The Generalized Gauss-Newton Matrix}
\label{sec:GGN}
This section discusses the Generalized Gauss-Newton matrix of \citet{schraudolph}, and justifies its use as an alternative to the Hessian. Its relevance to our discussion of natural gradient methods will be made clear later in Section \ref{sec:connections_to_GGN}, where we establish a correspondence between this matrix and the Fisher.

The classical Gauss-Newton matrix (or more simply the Gauss-Newton matrix) is the curvature matrix $G$ which arises in the Gauss-Newton method for non-linear least squares problems \citep[e.g.][]{dennis_text, ortega_text, nocedal_book}.  It is applicable to our standard neural network training objective $h$ in the case where $L(y,z) = \frac{1}{2}\|y - z\|^2$, and is given by 
\begin{align*}
G = \frac{1}{|S|} \sum_{(x,y) \in S} J_f^\top J_f \ep{,}
\end{align*}
where $J_f$ is the Jacobian of $f(x,\theta)$ w.r.t.~the parameters $\theta$. %, evaluated at the current value $\theta_i$ of $\theta$.  
It is usually defined as a modified version of the Hessian $H$ of $h$ (w.r.t.~$\theta$), obtained by dropping the second term inside the sum in the following expression for $H$:
\begin{align*}
H = \frac{1}{|S|} \sum_{(x,y) \in S} \left( J_f^\top J_f - \sum_{j=1}^m [y-f(x,\theta)]_j H_{[f]_j} \right) \ep{,}
\end{align*}
where $H_{[f]_j}$ is the Hessian (w.r.t.~$\theta$) of the $j$-th component of $f(x,\theta)$. We can see from this expression that $G=H$ when $y = f(x,\theta)$. And more generally, if the $y$'s are well-described by the model $f(x,\theta) + \epsilon$ for i.i.d. noise $\epsilon$ then $G=H$ will hold approximately.

An alternative way to derive the classical Gauss-Newton is to simply replace the non-linear function $f(x,\theta)$ by its own local linear approximation, centered at the current iterate $\theta_k$.  In particular, we replace $f$ by $\tilde{f}(x, \theta) = J_f \cdot (\theta - \theta_k)  + f(x, \theta_k)$ so that $h$ becomes a quadratic function of $\theta$, with derivative $\nabla h (\theta_k)$ and Hessian given by $G$.

Beyond the fact that the resulting matrix is PSD and has other nice properties discussed below, there doesn't seem to be any obvious justification for linearizing $f$ (or equivalently, dropping the corresponding term from the Hessian). It's likely that the reasonableness of doing this depends on problem-specific details about $L$ and $f$, and how the curvature matrix will be used by the optimizer. In Subsection \ref{sec:ggn_insights_other_work} we discuss how linearizing $f$ might be justified, specifically for wide neural networks, by some recent theoretical analyses.

\citet{schraudolph} showed how the idea of the Gauss-Newton matrix can be generalized to the situation where $L(y,z)$ is \emph{any} loss function which is convex in $z$.  The generalized formula for $G$ is
\begin{align}
\label{eqn:GGN}
G = \frac{1}{|S|}\sum_{(x,y) \in S} J_f^\top H_L J_f \ep{,}
\end{align}
where $H_L$ is the Hessian of $L(y, z)$ w.r.t.~$z$, evaluated at $z = f(x,\theta)$.  
Because $L(y,z)$ is convex in $z$, $H_L$ will be PSD for each $(x,y)$, and thus so will $G$.  We will call this $G$ the Generalized Gauss-Newton matrix (GGN). (Note that this definition is sensitive to where we draw the dividing line between the loss function $L$ and the network itself (i.e. $z$), in contrast to the definition of the Fisher, which is invariant to this choice.)

%Note that the GGN doesn't actually depend on the values of the $y$'s in the training set $S$, as neither $J_f$ nor $H_L$ depend on $y$.
Analogously to the case of the classical Gauss-Newton matrix (which assumed $L(y,z) = \frac{1}{2}\|y - z\|^2$), the GGN can be obtained by dropping the second term inside the sum of the following expression for the Hessian $H$:
\begin{align}
\label{eqn:H_decomp}
H = \frac{1}{|S|} \sum_{(x,y) \in S} \left( J_f^\top H_L J_f + \sum_{j=1}^m \left[\left.\nabla_z L(y,z)\right|_{z=f(x,\theta)}\right]_j H_{[f]_j} \right) \ep{.}
\end{align}
Here $\left.\nabla_z L(y,z)\right|_{z=f(x,\theta)}$ is the gradient of $L(y,z)$ w.r.t.~$z$, evaluated at $z = f(x,\theta)$.  Note if we have for some local optimum $\theta^*$ that $\left[\left.\nabla_z L(y,z)\right|_{z=f(x,\theta^*)}\right]_j \approx 0$ for each $(x,y)$ and $j$, which corresponds to the network making an optimal prediction for each training case over each dimension, then $G(\theta^*) = H(\theta^*)$.  In such a case, the behavior of a 2nd-order optimizer using $G$ will approach the behavior of standard Newton's method as it converges to $\theta^*$. A weaker condition implying equivalence is that $\frac{1}{S_y(x)} \sum_{y \in S_y(x)} \left[\left.\nabla_z L(y,z)\right|_{z=f(x,\theta^*)}\right]_j \approx 0$ for all $x \in S_x$ and $j$, where $S_y(x)$ denotes the set of $y$'s s.t.~$(x,y) \in S$, which corresponds to the network making an optimal prediction for each $x$ in the presence of intrinsic uncertainty about the target $y$. (This can be seen by noting that $H_{[f]_j}$ doesn't depend on $y$.)

Like the Hessian, the GGN can be used to define a local quadratic model of $h$, as given by:
\begin{align*}
M_k(\delta) = \frac{1}{2} \delta^\top G(\theta_k) \delta + \nabla h(\theta_k)^\top \delta + h(\theta_k) \ep{.}
\end{align*}
In 2nd-order methods based on the GGN, parameter updates are computed by minimizing $M_k(\delta)$ w.r.t. $\delta$.  The exact minimizer\footnote{This formula assumes $G(\theta_k)$ is invertible. If it's not, the problem will either be unbounded, or the solution can be computed using the pseudo-inverse instead.} $\delta^* = -G(\theta_k)^{-1} \nabla h(\theta_k)$ is often too difficult to compute, and so practical methods will often only approximately minimize $M_k(\delta)$ \citep[e.g][]{dembo1982inexact, steihaug1983conjugate,dennis_text,HF,KSD}.

%A key property of $G$ which is not shared by the Hessian $H$ is that it is PSD, and can thus be used to define a local quadratic model of the objective $h$ which is bounded below.  While the unboundedness of local quadratic models defined by the Hessian can be worked around by imposing a trust region, it has nevertheless been observed by various researchers \citep{schraudolph,HF,KSD} that $G$ works much better in practice for neural network optimization. We will explore this in much more detail in Section \ref{sec:GGN_speculation}.

Since computing the whole matrix explicitly is usually too expensive, the GGN is typically accessed via matrix-vector products.  To compute such products efficiently one can use the method of \citet{schraudolph}, which is a generalization of the well-known method for computing such products with the classical Gauss-Newton (and is also related to the TangentProp method of \citet{simard1992tangent}).  The method is similar in cost and structure to standard backpropagation, although it can sometimes be tricky to implement (see \citet{HF_chapter}).

As pointed out in \citet{hf-rnn}, the GGN can also be derived by a generalization of the previously described derivation of the classical Gauss-Newton matrix to the situation where $L$ is an arbitrary convex loss.  In particular, if we substitute the linearization $\tilde{f}$ for $f$ in $h$ as before (where $\tilde{f}(x, \theta) = J_f \cdot (\theta - \theta_i)  + f(x, \theta_i)$ is the linearization of $f$), it is not difficult to see that the Hessian of the resulting modified $h$ will be equal to the GGN.% $\frac{1}{|S|}\sum_{(x,y) \in S} J_f^\top H_L J_f$.

%\citet{hf-rnn} showed that the GGN matrix can also be viewed as the Hessian of a particular approximation of $h$ constructed by replacing $f(x,\theta)$ with its 1st-order approximation w.r.t. $\theta$.  Consider the following local convex approximation $\hat h$ to $h$ at $\theta_k$ that is obtained by taking the first-order approximation $f(x,\theta) \approx f(x,\theta_k) + \Fp \delta$ (where $\delta = \theta - \theta_k$):
%\begin{equation}
%\hat h(\delta) = \frac1{|S|} \sum_{(x,y)\in S} L( y, f(x, \theta_k) + \Fp \delta ) 
%\label{eqn:GGN_lin}
%\end{equation}
%The approximation $\hat h$ is convex because it is a composition of a convex function and an affine function. It is easy to see that $\hat h$ and $h$ have the same derivative at $\theta = \theta_k$, because
%\begin{equation*}
%%\frac{\partial \hat f_\theta(\delta)}{\partial \delta}\Big|_{\delta=0} = \Fp(\theta) ^\top L'(F(\theta)) = \Fp^\top L' 
%\hatgrad = \frac1{|S|} \sum_{(x,y)\in S} \Fp^\top \nabla_z L(y,z)
%\end{equation*}
%which is precisely the gradient of $h$ at $\theta_k$. And the Hessian of $\hat h$ at $\theta = \theta_k$ is precisely the GGN matrix:
%\begin{equation*}
%%\frac{\partial^2 \hat f_\theta(\delta)}{\partial \delta^2}\Big|_{\delta=0} = {\Fp}^\top H_L \Fp
%H_{\hat h} = \frac1{|S|} \sum_{(x,y)\in S} \Fp^\top H_L \Fp = \GGN
%\end{equation*}
%%So the Gauss-Newton matrix can be seen as the Hessian of a local convex approximation to the objective $f$ based on a local 1st-order approximation to $F(\theta)$.

\citet{schraudolph} advocated that when computing the GGN, $L$ and $f$ be redefined so that as much as possible of the network's computation is performed within $L$ instead of $f$, while maintaining the convexity of $L$.  This is because, unlike $f$, $L$ is not linearly approximated in the GGN, and so its associated second-order derivative terms are faithfully captured.  What this almost always means in practice is that what is usually thought of as the final non-linearity of the network (i.e. $\nonlin_\ell$) is folded into $L$, and the network itself just computes the identity function at its final layer.  Interestingly, in many natural situations which occur in practice, doing this gives a much simpler and more elegant expression for $H_L$.  Exactly when and why this happens will be made clear in Section \ref{sec:connections_to_GGN}. %when we discuss the connection between the GGN and the Fisher information matrix that for certain ``matching" choices of $L$ and $\nonlin_{\ell}$.

\subsection{Speculation on Possible Advantages of the GGN Over the Hessian}
\label{sec:GGN_speculation}

\subsubsection{Qualitative Observations}

Unlike the Hessian, the GGN is positive semi-definite (PSD). This means that it never models the curvature as negative in any direction.  The most obvious problem with negative curvature is that the quadratic model will predict an unbounded improvement in the objective for moving in the associated directions. Indeed, without the use of some kind of trust-region or damping technique (as discussed in Section \ref{sec:role_of_damping}), or pruning/modification of negative curvature directions \citep{KSD, dauphin2014identifying}, or self-terminating Newton-CG scheme \citep{steihaug1983conjugate}, the update produced by minimizing the quadratic model will be infinitely large in such directions. 

However, attempts to use such methods in combination with the Hessian have yielded lackluster results for neural network optimization compared to methods based on the GGN \citep{HF, HF_chapter, KSD}. So what might be going on here? While the true curvature of $h(\theta)$ can indeed be negative in a local neighborhood (as measured by the Hessian), we know it must quickly become non-negative as we travel along any particular direction, given that our loss $L(y,z)$ is convex in $z$ and bounded below. Meanwhile, positive curvature predicts a quadratic penalty, and in the worst case merely underestimates how badly the objective will eventually increase along a particular direction. We can thus say that negative curvature is somewhat less ``trustworthy" than positive curvature for this reason, and speculate that a 2nd-order method based on the GGN won't have to rely as much on trust-regions etc (which restrict the size of the update and slow down performance) to produce reliable updates.

There is also the issue of estimation from limited data. Because contributions made to the GGN for each training case and each individual component of $f(x,\theta)$ are PSD, there can be no cancellation between positive and negative/indefinite contributions. This means that the GGN can be more robustly estimated from subsets of the training data than the Hessian. (By analogy, consider how much harder it is to estimate the scale of the mean value of a variable when that variable can take on both positive and negative values, and has a mean close to $0$.) This property also means that positive curvature from one case or component will never be cancelled out by negative curvature from another case or component. And if we believe that negative curvature is less trustworthy than positive curvature over larger distances, this is probably a good thing.

Despite these nice properties, the GGN is notably \emph{not} an upper bound on the Hessian (in the PSD sense), as it fails to model \emph{all} of the positive curvature contained in the latter.  But crucially, it only fails to model the (positive or negative) curvature coming from the network function $f(x,\theta)$, as opposed to the curvature coming from the loss function $L(y,z)$.  (To see this, recall the decomposition of the Hessian from eqn.~\ref{eqn:H_decomp}, noting that the term dropped from the Hessian depends only on the gradients of $L$ and the Hessian of components of $f$.)  Curvature coming from $f$, whether it is positive or negative, is arguably less trustworthy/stable across long distance than curvature coming from $L$, as argued below.

\subsubsection{A More Detailed View of the Hessian vs the GGN}

Consider the following decomposition of the Hessian, which is a generalization of the one given in eqn.~\ref{eqn:H_decomp}:
\begin{align*}
H = \frac{1}{|S|} \sum_{(x,y) \in S} \left( J_f^\top H_L J_f + C + C^\top + \sum_{i=1}^\ell \sum_{j=1}^{m_i} \left[\nabla_{a_i} L(y,f)\right]_j J_{s_i}^\top H_{[\phi_i(s_i)]_j} J_{s_i} \right) \ep{.}
\end{align*}
Here, $a_i$, $\phi_i$ and $s_i$ are defined as in Section \ref{sec:neural_networks}, $\nabla_{a_i} L(y,f)$ is the gradient of $L(y,f)$ w.r.t.~$a_i$, $H_{[\phi_i(s_i)]_j}$ is the Hessian of $\phi_i(s_i)$ (i.e. the function which computes $a_i$) w.r.t.~$s_i$, $J_{s_i}$ is the Jacobian of $s_i$ (viewed as a function of $\theta$ and $x$) w.r.t.~$\theta$, and $C$ is given by
\begin{align*}
C &= \left[\begin{array}{c}
     J_{a_0} \otimes \nabla_{s_{1}} L (y, f)\\
     J_{a_1} \otimes \nabla_{s_{2}} L (y, f)\\
     \vdots\\
     J_{a_{\ell}} \otimes \nabla_{s_{\ell-1}} L (y, f)
   \end{array}\right] \ep{,}
\end{align*}
where $\otimes$ denotes the Kronecker product.
%\[ s_i = W_i a_{i - 1} + b_i \]
%\[ a_i = \phi_i (s_i) \]

The $C+C^\top$ term represents the contribution to the curvature resulting from the interaction of the different layers. Even in a linear network, where each $\phi_i$ computes the identity function, this is non-zero, since $f$ will be an order $\ell$ multilinear function of the parameters. We note that since $a_i$ doesn't depend on $W_j$ for $j > i$, $C$ will be block lower-triangular, with blocks corresponding to the $W_i$'s.

Aside from $C + C^\top$, the contribution to the Hessian made by $f$ comes from a sum of terms of the form $\left[\nabla_{a_i} L(y,f)\right]_j J_{s_i}^\top H_{[\phi_i(s_i)]_j} J_{s_i}$. These terms represent the curvature of the activation functions, and will be zero in a linear network (since we would have $H_{[\phi_i(s_i)]_j} = 0$). It seems reasonable to suspect that the sign of these terms will be subject to rapid and unpredictable change, resulting from sign changes in both $H_{[\phi_i(s_i)]_j}$ and $\left[\nabla_{a_i} L(y,f)\right]_j$. The former is the ``local Hessian" of $\phi_i$, and will change signs as the function $\phi_i$ enters its different convex and concave regions ($\phi_i$ is typically non-convex). $\left[\nabla_{a_i} L(y,f)\right]_j$ meanwhile is the loss derivative w.r.t. that unit's output, and depends on the behavior of all of the layers above $a_i$, and on which ``side" of the training target the network's current prediction is (which may flip back and forth at each iteration).  

This is to be contrasted with the term $J_f^\top H_L J_f$, which represents the curvature of the loss function $L$, and which remains PSD everywhere (and for each individual training case). Arguably, this term will be more stable w.r.t. changes in the parameters, especially when averaged over the training set.  %Note that this term, and only this term, is included in the GGN.

%As curvature tends to change more slowly than the gradient we would expect a term like $J_f^\top H_L J_f$ to remain more stable over long distances than $\left[\nabla_{a_i} L(y,f)\right]_j$, and we would also expect $\left[\nabla_{a_i} L(y,f)\right]_j$ to be a highly sensitive and unpredictable quantity for early layers, due to the amount of non-linear 

%This means that the GGN may act as an upper bound on the curvature over a larger local region than the Hessian does, and thus produce updates that are less likely to overshoot along directions whose curvature starts low close to $\theta_k$ but increases sharply.  Note that this is the same intuition behind the use of step-size $1/\mathcal{L}$ within gradient descent to ensure robust convergence on convex Lipschitz-smooth objectives (as discussed in Section \ref{sec:2nd-order}).

%It should be noted however that the GGN also drops terms from the Hessian which can be positive, and thus it is not guaranteed to be an upper bound on the Hessian in all cases.

\subsubsection{Some Insights From Other Works}
\label{sec:ggn_insights_other_work}

In the case of the squared error loss $L(y,z) = \frac{1}{2}\|y - z\|^2$ (which means that the GGN reduces to the standard Gauss-Newton matrix) with $m=1$, \citet{chen2011hessian} established that the GGN is the unique matrix which gives rise to a local quadratic approximation of $L(y,f(x,\theta))$ which is both non-negative (as $L$ itself is), and vanishing on a subspace of dimension $n-1$ (which is the dimension of the \emph{manifold} on which $L$ itself vanishes). Notably, the quadratic approximation produced using the Hessian need not have either of these properties. By summing over output components and averaging over the training set $S$, one should be able to generalize this result to the entire objective $h(\theta)$ with $m \geq 1$. Thus, we see that the GGN gives rise to a quadratic approximation which shares certain global characteristics with the true $h(\theta)$ that the 2nd-order Taylor series doesn't, despite being a less precise approximation to $h(\theta)$ in a strictly \emph{local} sense.

\citet{botev2017practical} observed that for networks with piece-wise linear activation functions, such as the popular RELUs (given by $[\phi_i(s_i)]_j = \max([s_i]_j,0)$), the GGN and the Hessian will coincide on the diagonal blocks whenever the latter is well-defined. This can be seen from the above decomposition of $H$ by noting that $C + C^\top$ is zero on the diagonal blocks, and that for piece-wise linear activation functions we have $H_{[\phi_i(s_i)]_j} = 0$ everywhere that this quantity exists (i.e. everywhere except the ``kinks" in the activation functions).

%Finally, it is worth noting that for networks with piece-wise linear activation functions, such as the popular RELUs (given by $[\phi_i(s_i)]_j = \max([s_i]_j,0)$), we have $H_{[\phi_i(s_i)]_j} = 0$ when $[\phi_i(s_i)]_j \neq 0$, and so the corresponding terms in our decomposition of the Hessian will be zero almost everywhere. This would almost make the Hessian and GGN coincide, if not for the $C + C^\top$ term. But because $C$ is zero on diagonal blocks the correspondence will hold there at least. (Note that the correspondence between the Hessian and the GGN on diagonal blocks for RELU networks was first shown by \citet{botev2017practical} using a somewhat different argument.)

Finally, under certain realistic assumptions on the network architecture and initialization point, and a lower bound on the width of the layers, recent results have shown that the $f$ function for a neural network behaves very similarly to its local linear approximation (taken at the initial parameters) throughout the entirety of optimization. This happens both for gradient descent \citep{du2018gradient, jacot2018neural, lee2019wide}, and natural gradient descent / GGN-based methods \citep{zhang2019fast, cai2019gram}, applied to certain choices for $L$. Not only does this allow one to prove strong global convergence guarantees for these algorithms, it lends support to the idea that modeling the curvature in $f$ (which is precisely the part of the Hessian that the GGN throws out) may be pointless for the purposes of optimization in neural networks, and perhaps even counter-productive.

\section{Computational Aspects of the Natural Gradient and Connections to the Generalized Gauss-Newton Matrix}
\label{sec:connections_to_GGN}

\subsection{Computing the Fisher (and Matrix-Vector Products With It)}

Note that
\begin{align*}
\nabla \log p(y|x,\theta) = J_f^\top \nabla_z \log r(y|z) \ep{,}
\end{align*}
where $J_f$ is the Jacobian of $f(x,\theta)$ w.r.t.~$\theta$, and $\nabla_z \log r(y|z)$ is the gradient of $\log r(y|z)$ w.r.t.~$z$, evaluated at $z = f(x,\theta)$ (with $r$ defined as near the end of Section \ref{sec:stat_learn}).  

As was first shown by \citet{ng_adaptive}, the Fisher information matrix is thus given by
\begin{align*}
F &= \ex_{Q_x}\left [\ex_{P_{y|x}} \left [\nabla \log p(y|x,\theta) \nabla \log p(y|x,\theta)^\top \right]\right] \\
&= \ex_{Q_x}[ \ex_{P_{y|x}}[ J_f^\top \nabla_z \log r(y|z) \nabla_z \log r(y|z)^\top J_f ] ] \\
&= \ex_{Q_x}[ J_f^\top \ex_{P_{y|x}} [\nabla_z \log r(y|z) \nabla_z \log r(y|z)^\top] J_f ] = \ex_{Q_x}[ J_f^\top F_R J_f ] \ep{,}
\end{align*}
where $F_R$ is the Fisher information matrix of the predictive distribution $R_{y|z}$ at $z = f(x,\theta)$. $F_R$ is itself given by
\begin{align*}
F_R = \ex_{P_{y|x}} [\nabla_z \log r(y|z) \nabla_z \log r(y|z)^\top] = \ex_{R_{y|f(x,\theta)}} [\nabla_z \log r(y|z) \nabla_z \log r(y|z)^\top]
\end{align*}
or
\begin{align*}
F_R = -\ex_{R_{y|f(x,\theta)}} [H_{\log r(y|z)}] \ep{,}
\end{align*}
where $H_{\log r(y|z)}$ is the Hessian of $\log r(y|z)$ w.r.t.~$z$, evaluated at $z = f(x,\theta)$.  

Note that even if $Q_x$'s density function $q(x)$ is known, and is relatively simple, only for certain choices of $R_{y|z}$ and $f(x,\theta$) will it be possible to analytically evaluate the expectation w.r.t.~$Q_x$ in the above expression for $F$.  For example, if we take $Q_x = \Normal(0,I)$, $R_{y|z} = \Normal(z,\sigma^2)$, and $f$ to be a simple neural network with no hidden units and a single tan-sigmoid output unit, then both $F$ and its inverse can be computed efficiently \citep{natural_efficient}. This situation is exceptional however, and for even slightly more complex models, such as neural networks with one or more hidden layers, it has never been demonstrated how to make such computations feasible in high dimensions.

%Notably, $Q_{y|x}$ being difficult to work with or unknown for $x$'s outside of the training set won't have any bearing on whether it is practical to use $Q_x$ when computing $F$ (versus $\hat Q_x$) , since $F$, and the distance function which defines it does not depend on $Q_{y|x}$ (much unlike the case for $h$).  Thus it is conceivable that in some situations it may be practical to use $Q_x$ to compute $F$ despite it being impractical to use $Q_x$ to compute $h$.  [[\textbf{what is the point being made here?}]]

Fortunately the situation improves significantly if $Q_x$ is replaced by $\hat Q_x$, as this gives
\begin{align}
\label{eqn:fisher_nice_exp}
F = \ex_{\hat Q_x}[ J_f^\top F_R J_f ] = \frac{1}{|S|} \sum_{x \in S_x} J_f^\top F_R J_f \ep{,}
\end{align}
which is easy to compute assuming $F_R$ is easy to compute.  Moreover, this is essentially equivalent to the expression in eqn.~\ref{eqn:GGN} for the generalized Gauss-Newton matrix (GGN), except that we have the Fisher $F_R$ of the predictive distribution ($R_{y|z}$) instead of Hessian $H_L$ of the loss ($L$) as the ``inner" matrix.

Eqn.~\ref{eqn:fisher_nice_exp} also suggests a straightforward and efficient way of computing matrix-vector products with $F$, using an approach similar to the one in \citet{schraudolph} for computing matrix-vector products with the GGN.  In particular, one can multiply by $J_f$ using a linearized forward pass (aka forward-mode automatic differentiation), then multiply by $F_R$ (which will be easy if $R_{y|z}$ is sufficiently simple), and then finally multiply by $J_f^\top$ using standard backprop.

\subsection{Qualified Equivalence of the GNN and the Fisher}
As we shall see in this subsection, the connections between the GGN and Fisher run deeper than just similar expressions and algorithms for computing matrix-vector products.

In \citet{ng_adaptive} it was shown that if the density function of $R_{y|z}$ has the form $r(y|z) = \prod_{j=1}^m c( y_j - z_j )$ where $c(a)$ is some univariate density function over $\Real$, then $F$ is equal to a re-scaled\footnote{The re-scaling constant will be determined by the properties of $c(a)$.} version of the classical Gauss-Newton matrix for non-linear least squares, with regression function given by $f$.  And in particular, the choice $c(a) = \exp(-a^2 / 2)$ turns the learning problem into non-linear least squares, and $F$ into the classical Gauss-Newton matrix.

\citet{heskes} showed that the Fisher and the classical Gauss-Newton matrix are equivalent in the case of the squared error loss, and proposed using the Fisher as an alternative to the Hessian in more general contexts. Concurrently with this work, \citet{razvan} showed that for several common loss functions like cross-entropy and squared error, the GGN and Fisher are equivalent.

%Other comparisons between the classical Gauss-Newton matrix and the Fisher matrix have also been made \citep[e.g.][]{bottou_very_large}, and it has also been observed \citep[e.g.][]{Le-Roux} that for single-output networks and squared error objectives, the \emph{empirical} Fisher and Gauss-Newton matrices differ only in that the former is multiplied by the error associated with each case. %\footnote{Note: this does \emph{not} mean that the matrices are identical up to a constant after averaging over the training $x$'s}. 

We will show that in fact there is a much more general equivalence between the two matrices, starting from the observation that the expressions for the GGN in eqn.~\ref{eqn:GGN} and Fisher in eqn.~\ref{eqn:fisher_nice_exp} are identical up to the equivalence of $H_L$ and $F_R$.

First, note that $L(y,z)$ might not even be convex in $z$, so that it wouldn't define a valid GGN matrix. But even if $L(y,z)$ is convex in $z$, it won't be true in general that $F_R = H_L$, and so the GGN and Fisher will differ.  However, there is an important class of $R_{y|z}$'s for which $F_R = H_L$ will hold, provided that we have $L(y,z) = -\log r(y|z)$ (putting us in the framework of Section \ref{sec:stat_learn}).

Notice that $F_R = -\ex_{R_{y|f(x,\theta)}} [H_{\log r(y|z)}]$, and $H_L = -H_{\log r(y|z)}$ (which follows from $L(y,z) = -\log r(y|z)$).  Thus, the two matrices being equal is equivalent to the condition
\begin{align}
\label{eqn:equiv_cond}
\ex_{R_{y|f(x,\theta)}} [H_{\log r(y|z)}] = H_{\log r(y|z)} \ep{.}
\end{align}

While this condition may seem arbitrary, it is actually very natural and holds in the important case where $R_{y|z}$ corresponds to an exponential family model with ``natural" parameters given by $z$. Stated in terms of equations this condition is
\begin{align*}
\log r(y|z) = z^\top T(y) - \log Z(z)
\end{align*}
for some function $T(y)$, where $Z(z)$ is the normalizing constant/partition function.  In this case we have $H_{\log r(y|z)} = -H_{\log Z}$ (which doesn't depend on $y$), and so eqn.~\ref{eqn:equiv_cond} holds trivially.

Examples of such $R_{y|z}$'s include:
\begin{itemize}
\item multivariate normal distributions where $z$ parameterizes only the mean $\mu$
\item multivariate normal distributions where $z$ is the concatenation of $\Sigma^{-1}\mu$ and the vectorization of $\Sigma^{-1}$
\item multinomial distributions where the softmax of $z$ is the vector of probabilities for each class
\end{itemize}

%[[maybe part of the discussion below should be moved to the GN section, and then just referenced here]]
Note that the loss function $L$ corresponding to the multivariate normal is the familiar squared error, and the loss corresponding to the multinomial distribution is the familiar cross-entropy.

Interestingly, the relationship observed by \citet{ollivier2018online} between natural gradient descent and methods based on the extended Kalman filter for neural network training relies on precisely the same condition on $R_{y|z}$. This makes intuitive sense, since the extended Kalman filter is derived by approximating $f$ as affine and then applying the standard Kalman filter for linear/Gaussian systems (which implicitly involves computing a Hessian of a linear model under a squared loss), which is the same approximation that can be used to derive the GGN from the Hessian (see Section \ref{sec:GGN}).

As discussed in Section \ref{sec:GGN}, when constructing the GGN one must pay attention to how $f$ and $L$ are defined with regards to what parts of the neural network's computation are performed by each function (this choice is irrelevant to the Fisher). For example, the softmax computation performed at the final layer of a classification network is usually considered to be part of the network itself and hence to be part of $f$.  The output $f(x,\theta)$ of this computation are normalized probabilities, which are then fed into a cross-entropy loss of the form $L(y,z) = -\sum_j y_j \log z_j$.  But the other way of doing it, which \citet{schraudolph} recommends, is to have the softmax function be part of $L$ instead of $f$, which results in a GGN which is slightly closer to the Hessian due to ``less" of the computational pipeline being linearized before taking the 2nd-order Taylor series approximation.  The corresponding loss function is $L(y,z) = -\sum_j y_j z_j + \log( \sum_j \exp(z_j) )$ in this case.  As we have established above, doing it this way also has the nice side effect of making the GGN equivalent to the Fisher, provided that $R_{y|z}$ is an exponential family model with $z$ as its natural parameters.

This (qualified) equivalence between the Fisher and the GGN suggests how the GGN can be generalized to cases where it might not otherwise be well-defined.  In particular, it suggests formulating the loss as the negative log density for some distribution, and then taking the Fisher of this distribution.  Sometimes, this might be as simple as defining $r(y | z) \propto \exp(-L (y, z))$ as per the discussion at the end of Section \ref{sec:stat_learn}.

For example, suppose our loss is defined as the negative log probability of a multi-variate normal distribution $R_{y|z} = N(\mu, \sigma^2)$ parameterized by $\mu$ and $\gamma = \log \sigma^2$ (so that $z = \begin{bmatrix}\mu \\ \gamma \end{bmatrix}$).  In other words, suppose that
\begin{align*}
L(y,z) = -\log r(y|z) \propto \frac{1}{2} \gamma + \frac{1}{2\exp(\gamma)}(x-\mu)^2 \ep{.}
\end{align*}
In this case the loss Hessian is equal to
\begin{align*}
H_L = \frac{1}{\exp(\gamma)} \begin{bmatrix} 1 & x-\mu \\ x-\mu & \frac{1}{2}(x-\mu)^2 \end{bmatrix} \ep{.}
\end{align*}
It is not hard to verify that this matrix is indefinite for certain settings of $x$ and $z$ (e.g. $x = 2$, $\mu = \gamma = 0$).  Therefore, $L$ is not convex in $z$ and we cannot define a valid GGN matrix from it.

To resolve this problem we can use the Fisher $F_R$ in place of $H_L$ in the formula for the GGN, which by eqn.~\ref{eqn:fisher_nice_exp} yields $F$.  Alternatively, we can insert reparameterization operations into our network to transform $\mu$ and $\gamma$ into the natural parameters $\frac{\mu}{\sigma^2} = \frac{\mu}{\exp(\gamma)}$ and $-\frac{1}{2\sigma^2} = -\frac{1}{2\exp(\gamma)}$, and then proceed to compute the GGN as usual, noting that $H_L = F_R$ in this case, so that $H_L$ will be PSD.  Either way will yield the same curvature matrix, due to the above discussed equivalence of the Fisher and GGN matrix for natural parameterizations.

\section{Constructing Practical Natural Gradient Methods, and the Critical Role of Damping}
\label{sec:role_of_damping}

%breakdowns:
%
%1st order (making update direction fail to reduce(increase?) h as predicted)
%
%2nd order approx of KL fails, meaning even if the 1st order still holds we stop getting the most improvement per unit KL
%
%
%To fix this we want to ``damp" update, making it smaller / more conservative so that these approximations are closer to holding
%
% Can exploit connection of Fisher to GGN, which gives that some of the 2nd-order curvature of $h$ will be modeled by $F$, so that the damping has the more modest goal of supplementing what is already much better than just a 1st-order approximation to $h$

Assuming that it is easy to compute, the simplest way to use the natural gradient in optimization is to substitute it in place of the standard gradient within a basic gradient descent approach.  This gives the iteration
\begin{align}
\label{eqn:ng_simple_opt}
\theta_{k+1} = \theta_k - \alpha_k \tilde{\nabla}  h(\theta_k) \ep{,}
\end{align}
where $\{\alpha_k\}_k$ is a schedule of step-sizes/learning-rates.%, which may be fixed or adaptive.

Choosing the step-size schedule can be difficult.  There are adaptive schemes which are largely heuristic in nature \citep{natural_efficient} and some non-adaptive prescriptions such as $\alpha_k = \rho/k$ for some constant $\rho$, which have certain theoretical convergence guarantees in the stochastic setting, but which won't necessarily work well in practice.  

In principle, we could apply the natural gradient method with infinitesimally small steps and produce a smooth idealized path through the space of realizable distributions%(where by realizable we mean ones that are contained in the class of distributions generated from all possible $\theta$'s) 
. %, which would be completely invariant to any smooth and invertible reparameterization.  
But since this is usually impossible in practice, and we don't have access to any other simple description of the class of distributions parameterized by $\theta$ that we could work with more directly, our only option is to take non-negligible discrete steps in the given parameter space\footnote{In principle, we could move to a much more general class of distributions, such as those given by some non-parametric formulation, where we could work directly with the distributions themselves.  But even assuming such an approach would be practical from a computational efficiency standpoint, we would lose the various advantages that we get from working with powerful parametric models like neural networks.  In particular, we would lose their ability to generalize to unseen data by modeling the ``computational process" which explains the data, instead of merely using smoothness and locality to generalize.}.

%Unlike approximate Newton methods which optimize a local quadratic model of the objective to produce both a direction and a default distance along that direction, the natural gradient is only a direction.  In most realizations of the natural gradient approach there is no serious attempt to deal with break-downs of the various 1st and 2nd-order approximations used to derive the natural gradient which inevitably occur when we use it to take large discrete steps in the original parameter space.
 
%gradient descent could be performed directly over the space of distributions [actually, does this even make sense?], but since we don't have direct access to this, any scheme which uses the natural gradient for optimization must still makes its updates in the original space, via translating [...]

The fundamental problem with simple schemes such as the one in eqn.~\ref{eqn:ng_simple_opt} is that they implicitly assume that the natural gradient is a good direction to follow over non-negligible distances in the original parameter space, which will not be true in general.  Traveling along a straight line in the original parameter space will not yield a straight line in distribution space, and so the resulting path may instead veer far away from the target that the natural gradient originally pointed towards.  This is illustrated in Figure \ref{fig:geo_diag}.

\begin{figure}%[H]
\begin{center}
\includegraphics[width=.9\columnwidth]{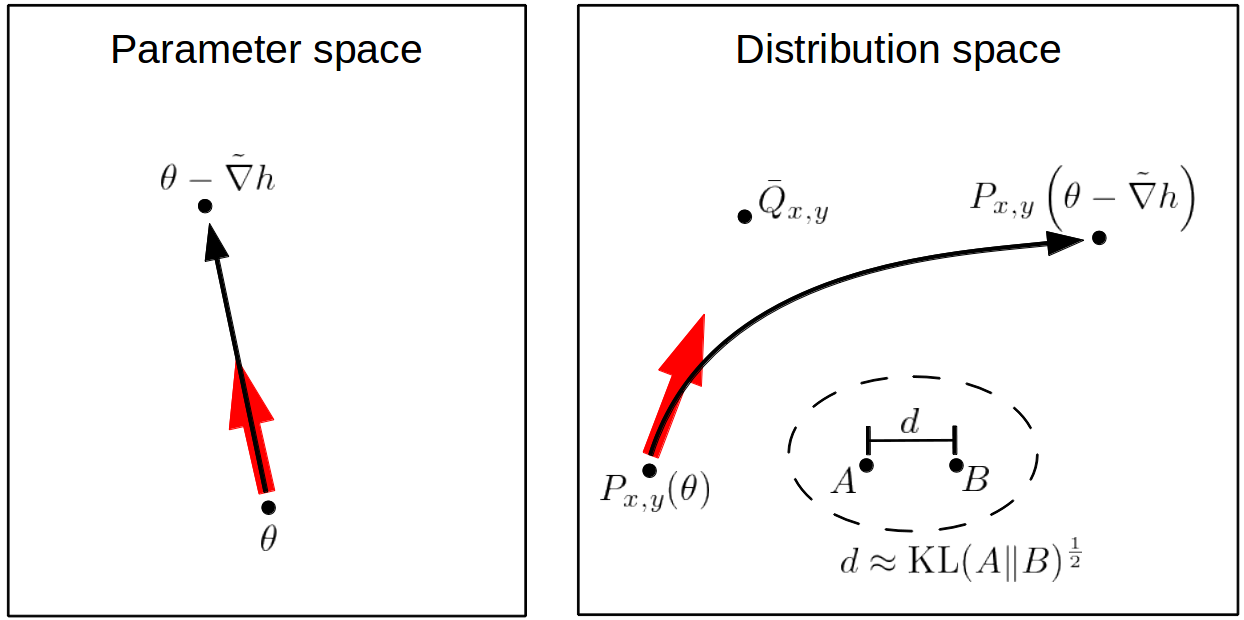}
\caption{\small A typical situation encountered when performing large discrete updates in the original parameter space.  The red arrow is the natural gradient direction (given by the vector $\tilde{\nabla} h$ in parameter space) and the black arrow is the path generated by taking $\theta - \alpha \tilde{\nabla} h$ for $\alpha \in [0,1]$. \label{fig:geo_diag} }
\end{center}
\end{figure}

Fortunately, we can exploit the (qualified) equivalence between the Fisher and the GGN in order to produce natural gradient-like updates which will often be appropriate to take with $\alpha_k = 1$.  In particular, we know from the discussion in Section \ref{sec:GGN} that the GGN matrix $G$ can serve as a reasonable proxy for the Hessian $H$ of $h$, and may even be superior in certain contexts. Meanwhile, the update $\delta$ produced by minimizing the GGN-based local quadratic model $M_k(\delta) = \frac{1}{2} \delta^\top G(\theta_k) \delta + \nabla h(\theta_k)^\top \delta + h(\theta_k)$ is given by $-G(\theta_k)^{-1} \nabla h(\theta_k)$, which will be equal to the negative natural gradient when $F = G$.  Thus, the (negative) natural gradient, with scaling factor $\alpha = 1$, can be seen as the optimal update according to a particular local 2nd-order approximation of $h$. And just as in the case of other 2nd-order methods, the break-down in the accuracy of this quadratic approximation over long distances, combined with the potential for the natural gradient to be very large (e.g. when $F$ contains some very small eigenvalues), can often lead to very large and very poor update proposals. Simply re-scaling the update by reducing $\alpha$ may be too crude a mechanism to deal with this subtle problem, as it will affect all eigen-directions (of $F$) equally, including those in which the natural gradient is already sensible, or even overly conservative.  

Instead, the connection between natural gradient descent and 2nd-order methods motivates the use of ``update damping" techniques that have been developed for the latter, which work by constraining or penalizing the solution for $\delta$ in various ways during the optimization of $M_k(\delta)$.  Examples include Tikhonov regularization/damping and the closely related trust-region method \citep[e.g.][]{tikhonov1943stability, more1983computing, conn2000trust, nocedal_book}, and other ones such as the ``structural damping" approach of \citet{hf-rnn}, or the approach present in Krylov Subspace Descent \citep{KSD}. See \citet{HF_chapter} for an in-depth discussion of these and other damping techniques in the context of neural network optimization.

This idea is supported by practical experience in neural network optimization. For example, the Hessian-free optimization approach of \citet{HF} generates its updates using a Tikhonov damping scheme applied to the exact GGN matrix (which was equivalent to the Fisher in that work). These updates, which can be applied with a step-size of 1, make a lot more progress optimizing the objective than updates computed without any damping (which must instead rely on a carefully chosen step-size to even be feasible).

It is worth pointing out that other interpretations of natural gradient descent can also motivate the use of damping/regularization terms. In particular, \citet{ollivier2018online} has shown that online natural gradient descent, with a particular flavor of Tikhonov regularization, closely resembles a certain type of extended Kalman filter-based training algorithm for neural networks \citep{singhal1989training, ruck1992comparative}, where $\theta$ is treated as an evolving hidden state that is estimated by the filter (using training targets as noisy observations and inputs as control signals).

\section{The Empirical Fisher}
\label{sec:emp_fish}

%\subsection{Definition}

An approximation of the Fisher known as the ``empirical Fisher" \citep{schraudolph}, which we denote by $\bar{F}$, is commonly used in practical natural gradient methods. It is obtained by taking the inner expectation of eqn.~\ref{eqn:F_expression_1} over the target distribution $Q_{x,y}$ (or its empirical surrogate $\hat{Q}_{x,y}$) instead of the model's distribution $P_{x,y}$. 

%This often simplifies the Fisher computations considerably.  

%In the case of $Q_{y|x}$, we can use the fact that for many problems the conditional $Q_{y|x}$ has a very simple form and is easy to work with, such as when $y$ can take on only a small finite set of values [[\textbf{isn't this also easy for case of $P_{y|x}$ too?}]].  

In the case where one uses $\hat{Q}_{x,y}$, this yields the following simple form:
\begin{align*}
\bar{F} &= \ex_{\hat{Q}_{x,y}}\left[ \nabla \log p(x,y|\theta) \nabla \log p(x,y|\theta)^\top \right] \\
&= \ex_{\hat{Q}_x}\left[\ex_{\hat{Q}_{y|x}}\left[\nabla \log p(y|x,\theta) \nabla \log p(y|x,\theta)^\top\right]\right] \\
&= \frac{1}{|S|} \sum_{(x,y) \in S} \nabla \log p(y|x,\theta) \nabla \log p(y|x,\theta)^\top \ep{.}
\end{align*}

This matrix is often incorrectly referred to as the Fisher, or even the Gauss-Newton, even though it is not equivalent to either of these matrices in general.

%This approximation to $F$, which we will denote $\bar{F}$, is what is sometimes called the ``empirical Fisher", and has been used in various works to define an approximation to the natural gradient.  

\subsection{Comparisons to the Standard Fisher}
\label{sec:empfish_comparison}

Like the Fisher $F$, the empirical Fisher $\bar{F}$ is PSD.  But unlike $F$, it is essentially free to compute, provided that one is already computing the gradient of $h$.  And it can also be applied to objective functions which might not involve a probabilistic model in any obvious way.  

Compared to $F$, which is of rank $\leq |S|\rank(F_R)$, $\bar{F}$ has a rank of $\leq |S|$, which can make it easier to work with in practice.  For example, the problem of computing the diagonal (or various blocks) is easier for the empirical Fisher than it is for higher rank matrices like the standard Fisher \citep{curvprop}.  This has motivated its use in optimization methods such as TONGA \citep{TONGA}, and as the diagonal preconditioner of choice in the Hessian-free optimization method \citep{HF}.  Interestingly however, there are stochastic estimation methods \citep{chapelle, curvprop} which can be used to efficiently estimate the diagonal (or various blocks) of the standard Fisher $F$, and these work quite well in practice. (These include the obvious method of sampling $y$'s from the model's conditional distribution and computing gradients from them, but also includes methods based on matrix factorization and random signs. See \citet{curvprop} for comparative analysis of the variance of these methods.)

Despite the various practical advantages of using $\bar{F}$, there are good reasons to use true Fisher $F$ instead of $\bar{F}$ whenever possible.  In addition to Amari's extensive theory developed for the exact natural gradient (which uses $F$), perhaps the best reason for using $F$ over $\bar{F}$ is that $F$ turns out to be a reasonable approximation/substitute to the Hessian $H$ of $h$ in certain important special cases, which is a property that $\bar{F}$ lacks in general.  

For example, as discussed in Section \ref{sec:nat_grad}, when the loss is given by $-\log p(y|x)$ (as in Section \ref{sec:stat_learn}), $F$ can be seen as an approximation of $H$, because both matrices have the interpretation of being the expected Hessian of the loss under some distribution.  Due to the similarity of the expression for $F$ in eqn.~\ref{eqn:F_expression_1} and the one above for $\bar{F}$, it might be tempting to think that $\bar{F}$ is given by the expected Hessian of the loss under $\hat{Q}_{x,y}$ (which is actually the formula for $H$) in the same way that $F$ is given by eqn.~\ref{eqn:F_expression_2}. But this is not the case in general.

And as we saw in Section \ref{sec:connections_to_GGN}, given certain assumptions about how the GGN is computed, and some additional assumptions about the form of the loss function $L$, $F$ turns out to be equivalent to the GGN.  This is very useful since the GGN can be used to define a local quadratic approximation of $h$, whereas $F$ normally doesn't have such an interpretation.  Moreover, \citet{schraudolph} and later \citet{HF} compared $\bar{F}$ to the GGN and observed that the latter performed much better as a curvature matrix within various neural network optimization methods.  %Moreover, the cost of computing matrix-vector products is roughly the same for either matrix. %, despite the former having a much lower rank for problems where $z$ is high dimensional (such as training auto-encoders).

As concrete evidence for why the empirical Fisher is, at best, a questionable choice for the curvature matrix, we will consider the following example.  Set $n = 1$, $f(x, \theta) = \theta$, $R_{y|z} = \Normal(z, 1)$, and $S = \{(0,0)\}$, so that $h(\theta)$ is a simple convex quadratic function of $\theta$, given by $h(\theta) = \frac{1}{2} \theta^2$.  In this example we have that $\nabla h = \theta$, $\bar{F} = \theta^2$, while $F = 1$.  If we use $\bar{F}^\xi$ as our curvature matrix for some exponent $\frac1{2} \leq \xi \leq 1$, then it is easy to see that an iteration of the form
\begin{align*}
\theta_{k+1} &= \theta_k - \alpha_k (\bar{F}(\theta_k)^\xi)^{-1} \nabla h(\theta_k) = \theta_k - \alpha_k (\theta_k^2)^{-\xi} \theta_k = ( 1 - \alpha_k |\theta_k|^{-2\xi} )\theta_k
\end{align*}
will fail to converge to the minimizer (at $\theta = 0$) unless $\xi < 1$ and the step-size $\alpha_k$ goes to $0$ sufficiently fast.  And even when it does converge, it will only be at a rate comparable to the speed at which $\alpha_k$ goes to $0$, which in typical situations will be either $\bigO(1/k)$ or $\bigO(1/\sqrt{k})$.  Meanwhile, a similar iteration of the form
\begin{align*}
\theta_{k+1} &= \theta_k - \alpha_k F^{-1} \nabla h(\theta_k) = \theta_k - \alpha_k \theta_k = (1-\alpha_k) \theta_k \ep{,}
\end{align*}
which uses the exact Fisher $F$ as the curvature matrix, will experience very fast linear convergence\footnote{Here we mean ``linear" in the classical sense that $|\theta_k - 0| \leq |\theta_0 - 0| |1-\alpha|^k$.} with rate $|1 - \alpha|$, for any fixed step-size $\alpha_k = \alpha$ satisfying $0 < \alpha < 2$.

It is important to note that this example uses a noise-free version of the gradient, and that this kind of linear convergence is (provably) impossible in most realistic stochastic/online settings.  Nevertheless, we would argue that a highly desirable property of any stochastic optimization method should be that it can, in principle, revert to an optimal (or nearly optimal) behavior in the deterministic setting.  This might matter a lot in practice, since the gradient may end up being sufficiently well estimated in earlier stages of optimization from only a small amount of data (which is a common occurrence in our experience), or in later stages provided that larger mini-batches or other variance-reducing procedures are employed \citep[e.g.][]{SAG, varreduction}.  More concretely, the pre-asymptotic convergence rate of stochastic 2nd-order optimizers can still depend strongly on the choice of the curvature matrix, as we will show in Section \ref{sec:asymptotic_speed}.

\subsection{A Discussion of Recent Diagonal Methods Based on the Empirical Fisher}

\label{sec:diag_empfish_methods}

Recently, a spate of stochastic optimization methods have been proposed that are all based on diagonal approximations of the empirical Fisher $\bar{F}$.  These include the diagonal version of AdaGrad \citep{ADAGRAD}, RMSProp \citep{RMSprop}, Adam \citep{adam}, etc.  Such methods use iterations of the following form (possibly with some slight modifications):
\begin{align}
\label{eqn:diag_iter}
\theta_{k+1} = \theta_k - \alpha_k (\B_k + \lambda I)^{-\xi} g_k(\theta_k) \ep{,}
\end{align}
where the curvature matrix $\B_k$ is taken to be a diagonal matrix $\diag(u_k)$ with $u_k$ adapted to maintain some kind of estimate of the diagonal of $\bar{F}$ (possibly using information from previous iterates/mini-batches), $g_k(\theta_k)$ is an estimate of $\nabla h(\theta_k)$ produced from the current mini-batch, ${\alpha_k}_k$ is a schedule of step-sizes, and $0 < \lambda$ and $0 < \xi \leq 1$ are hyperparameters (discussed later in this section).

There are also slightly more sophisticated methods \citep{pesky,adadelta} which use preconditioners that combine the diagonal of $\bar{F}$ with other quantities (such as an approximation of the diagonal of the Gauss-Newton/Fisher in the case of \citet{pesky}) in order to correct for how the empirical Fisher doesn't have the right ``scale" (which is ultimately the reason why it does poorly in the example given at the end of Section \ref{sec:empfish_comparison}).

A diagonal preconditioner \citep{nash1985preconditioning} of the form used in eqn.~\ref{eqn:diag_iter} was also used by \citep{HF} to accelerate the conjugate gradient (CG) sub-optimizations performed within a truncated-Newton method (using the GGN matrix). In the context of CG, the improper scale of $\bar{F}$ is not as serious an issue due to the fact that CG is invariant to the overall scale of its preconditioner (since it computes an optimal ``step-size" at each step which automatically adjusts for the scale).  However, it still makes more sense to use the diagonal of the true Fisher $F$ as a preconditioner, and thanks to the method proposed by \citet{chapelle}, this can be estimated efficiently and accurately.
%[[\textbf{actually, is this really the full explanation?  Unlike with SGD, the preconditioner remains fixed throughout CG's optimization, and won't ``go to 0"}]]

The idea of using the diagonal of $F$, $\bar{F}$, or the Gauss-Newton as a preconditioner for stochastic gradient descent (SGD) and was likely first applied to neural networks with the work of Lecun and collaborators \citep{diag_lecun,lecun_tricks}, who proposed an iteration of the form in eqn.~\ref{eqn:diag_iter} with $\xi = 1$ where $u_k$ approximates the diagonal of the Hessian or the Gauss-Newton matrix (which as shown in Section \ref{sec:connections_to_GGN}, is actually equivalent to $F$ for the common squared-error loss). Following this work, various neural network optimization methods have been developed over the last couple of decades that use diagonal, block-diagonal, low-rank, or Krylov-subspace based approximations of $F$ or $\bar{F}$ as a curvature matrix/preconditioner. In addition to methods based on diagonal approximations already mentioned, some methods based on non-diagonal approximations include the method of \citet{ng_adaptive}, TONGA \citep{TONGA}, Natural Newton \citep{naturalnewton}, HF \citep{HF}, KSD \citep{KSD} and many more.

The idea of computing an estimate of the (empirical) Fisher using a history of previous iterates/mini-batches also appeared in various early works.  The particular way of doing this proposed \citet{ADAGRAD}, which is to use an equally weighted average of all past gradients, was motivated from a regret-based asymptotic convergence analysis and tends not to work well in practice \citep{RMSprop}.
%and which was done in order to make it possible to prove a particular regret bound%\footnote{The total regret at iteration $K$ is defined as $\sum_{k=1}^K (h(\theta_k) - h(\theta^*)$ for the optimal $\theta^*$, and provides a measure of the speed of convergence of an online optimization algorithm which is popular in the convex optimization community.}.
The traditional and more intuitive approach of using an exponentially decayed running average \citep[e.g.][]{lecun_tricks,ng_adaptive} works better, at least pre-asymptotically, as it is able to naturally ``forget" very old contributions to the estimate (which are based on stale parameter values).

It is important to observe that the way $\bar{F}$ is estimated can affect the convergence characteristics of an iteration like eqn.~\ref{eqn:diag_iter} in subtle but important ways.  For example, if $\bar{F}$ is estimated using gradients from previous iterations, and especially if it is the average of \emph{all} past gradients (as in AdaGrad), it may shrink sufficiently slowly that the convergence issues seen in the example at the end of Section \ref{sec:empfish_comparison} are avoided.  Moreover, for reasons related to this phenomenon, it seems likely that the proofs of regret bounds in \citet{ADAGRAD} and the related work of \citet{Hazan_newton} could \emph{not} be modified to work if the exact $\bar{F}$, computed only at the current $\theta$, were used.  Developing a better understanding of this issue, and the relationship between methods developed in the online learning literature (such as AdaGrad), and classical stochastic 2nd-order methods based on notions of curvature, remains an interesting direction for future research.

%As argued in Section \ref{sec:emp_fish}, this is arguably a much better matrix to use than the diagonal of $\bar{F}$.  Moreover, with the recent work of \citep{chapelle, curvprop} one can forgo the biased approximation of the diagonal of $F$ used in [[cite]] in favor of an unbiased stochastic estimate.

\subsection{The Constants $\lambda$ and $\xi$}

The constants $\lambda$ and $\xi$ present in eqn.~\ref{eqn:diag_iter} are often thought of as fudge factors designed to correct for the ``poor conditioning" \citep{diag_lecun} of the curvature matrix, or to guarantee boundedness of the updates and prevent the optimizer from ``blowing up" \citep{lecun_tricks}.  However, these explanations are oversimplifications that reference the symptoms instead of the cause.   A more compelling and functional explanation, at least in the case of $\lambda$, comes from viewing the update in eqn.~\ref{eqn:diag_iter} as being the minimizer of a local quadratic approximation $M_k(\delta) = \frac{1}{2} \delta^\top \B_k \delta + \nabla h(\theta_k)^\top \delta + h(\theta_k)$ to $h(\theta_k + \delta)$, as discussed in Section \ref{sec:role_of_damping}.  In this view, $\lambda$ plays the role of a Tikhonov damping parameter \citep{tikhonov1943stability, conn2000trust, nocedal_book, HF_chapter} which is added to $\B_k$ in order to ensure that the proposed update stays within a certain region around zero in which $M_k(\delta)$ remains a reasonable approximation to $h(\theta_k + \delta)$.  Note that this explanation implies that no single fixed value of $\lambda$ will be appropriate throughout the entire course of optimization, since the local properties of the objective will change, and so an adaptive adjustment scheme, such as the one present in HF \citep{HF} (which is based on the Levenberg-Marquardt method), should be used.

The use of the exponent $\xi = 3/4$ first appeared in HF as part of its diagonal preconditioner for CG, and was justified as a way of making the curvature estimate ``more conservative" by making it closer to a multiple of the identity, to compensate for the diagonal approximation being made (among other things).  Around the same time, \citet{ADAGRAD} proposed to use $\xi = 1/2$ within an update of the form of eqn.~\ref{eqn:diag_iter}, which was required in order to prove certain regret bounds for non-strongly-convex objectives.

%(although this had nothing to do with the diagonal approximation being made, since $\xi = 1/2$ is also used in the non-diagonal version of AdaGrad, where it analogously corresponds to the exponent in a \emph{matrix power} of $\bar{F}$). 
%However, it is noteworthy that while the use of $\xi = 1$ would invalidate this particular bound (or at least its existing proof), it is relatively simple to prove a different bound for the $\xi = 1$ case by a minor modification of the original argument of \citet{ADAGRAD}, provided that one appropriately modifies the schedule for the step-size $\alpha_k$ so that the updates scale as $1/\sqrt{k}$.

To shed some light on the question of $\xi$, we can consider the work of \citet{Hazan_newton}, who like \citet{ADAGRAD}, developed and analyzed an online approximate Newton method within the framework of online convex optimization. Like the non-diagonal version of AdaGrad, the method proposed by \citet{Hazan_newton} uses an estimate of the empirical Fisher $\bar{F}$ computed as the average of gradients from all previous iterations.  While impractical for high dimensional problems like any non-diagonal method is (or at least, one that doesn't make some other strong approximation of the curvature matrix), this method achieves a better bound on the regret to what \citet{ADAGRAD} was able to show for AdaGrad ($\bigO(\log(k))$ instead of $\bigO(\sqrt{k})$, where $k$ is the total number of iterations), which was possible in part due to the use of stronger hypotheses about the properties of $h$ (e.g. that for each $x$ and $y$, $L(y,f(x,\theta))$ is a strongly convex function of $\theta$).  Notably, this method uses $\xi = 1$, just as in standard natural gradient descent, which provides support for such a choice, especially since the $h$ used in neural networks will typically satisfy these stronger assumptions in a local neighborhood of the optimum, at least when standard $\ell_2$ regularization is used.  

However, it is important to note that \citet{Hazan_newton} also proves a $\bigO(\log(k))$ bound on the regret for a basic version of SGD, and that what actually differentiates the various methods they analyze is the constant hidden in the big-O notation, which is much larger for the version of SGD they consider than for their approximate Newton method.  In particular, the former depends on a quantity which grows with the condition number of the Hessian $H$ at $\theta^*$ while the latter does not, in a way that echos the various analyses performed on stochastic gradient descent and stochastic approximations of Newton's method in the more classical ``local-convergence" setting \citep[e.g.][]{Murata, bottou_very_large}.

\section{A Critical Analysis of Parameterization Invariance}
\label{sec:param_invar}

One of the main selling points of the natural gradient method is its invariance to reparameterizations of the model.  In particular, the smooth path through the space of distributions generated by the idealized natural gradient method with infinitesimally small steps will be invariant to any smooth invertible reparameterization of the $f$. 

More precisely, it can be said that this path will be the same whether we use the default parameterization (given by $P_{y|x}(\theta)$), or parameterize our model as $P_{y|x}(\zeta(\gamma))$, where $\zeta : \Real^n \rightarrow \Real^n$ is a smooth invertible ``reparameterization function" which relates $\theta$ to $\gamma$ as $\theta = \zeta(\gamma)$.% (and where we assume an equivalent starting point in the alternative parameterization given by $\gamma_0 = g^{-1}(\theta_0)$)

%[[in a future version of this report, point out that just because $H$ doesn't give rise to an ``intrinsic" metric, doesn't mean that it won't nessesarily be parameterization invariant.  Indeed, there are many non-intrinsic choices that are parameterization invariant]]

In this section we will examine this ``smooth path parameterization invariance" property more closely in order to answer the following questions:
\begin{itemize} 
\item How can we characterize it using only basic properties of the curvature matrix?
\item Is there an elementary proof that can be applied in a variety of settings?
\item What other kinds of curvature matrices give rise to it, and is the Hessian included among these? 
\item Will this invariance property imply that \emph{practical} optimization algorithms based on the natural gradient (i.e. those that use large steps) will behave in a way that is invariant to the parameterization?
\end{itemize} 

%First note that this smooth path invariance is achieved if and only if the directions produced using either parameterization, when viewed as vectors in the same space, are proportional to each other.

Let $\zeta$ be as above, and let $d_\theta$ and $d_\gamma$ be updates given in $\theta$-space and $\gamma$-space (resp.).  Additively updating $\gamma$ by $d_\gamma$ and translating it back to $\theta$-space via $\zeta$ gives $\zeta(\gamma + d_\gamma)$.  Measured by some non-specific norm $\| \cdot \|$, this differs from $\theta + d_\theta$ by:
\begin{align*}
\| \zeta(\gamma + d_\gamma) - (\theta + d_\theta) \| \ep{.}
\end{align*}

This can be rewritten and bounded as
\begin{align}
\label{eqn:invar_bound}
\| (\zeta(\gamma + d_\gamma) - (\zeta(\gamma) + J_\zeta d_\gamma)) + (J_\zeta d_\gamma -  d_\theta) \| \leq \| \zeta(\gamma + d_\gamma) - (\zeta(\gamma) + J_\zeta d_\gamma) \| + \| J_\zeta d_\gamma -  d_\theta \| \ep{,}
\end{align}
where $J_\zeta$ is the Jacobian of $\zeta$, and we have used $\theta = \zeta(\gamma)$.

The first term on the RHS of eqn.~\ref{eqn:invar_bound} measures the extent to which $\zeta( \gamma + d_\gamma)$ fails to be predicted by the first-order Taylor series approximation of $\zeta$ centered at $\gamma$ (i.e. the local affine approximation of $\zeta$ at $\gamma$).  This quantity will depend on the size of $d_\gamma$, and the amount of curvature in $\gamma$. In the case where $\zeta$ is affine, it will be exactly $0$. We can further bound it by applying Taylor's theorem for each component of $\zeta$, which gives
\begin{align}
\label{eqn:smoothbnd_taylor}
\| \zeta(\gamma + d_\gamma) - (\zeta(\gamma) + J_\zeta d_\gamma) \| \leq \frac{1}{2} \left \|
\begin{bmatrix}
d_\gamma^\top H_{[\zeta]_1}( \gamma + c_1 d_\gamma ) d_\gamma \\
d_\gamma^\top H_{[\zeta]_2}( \gamma + c_2 d_\gamma ) d_\gamma \\
\vdots \\
d_\gamma^\top H_{[\zeta]_n}( \gamma + c_n d_\gamma ) d_\gamma
\end{bmatrix}
\right \|
\end{align}
for some $c_i \in (0,1)$.  If we assume that there is some $C > 0$ so that for all $i$ and $\gamma$, $\|H_{[\zeta]_i}( \gamma )\|_2 \leq C$, then using the fact that $|d_\gamma^\top H_{[\zeta]_i}( \gamma + c_n d_\gamma ) d_\gamma| \leq \frac{1}{2} \| H_{[\zeta]_i}( \gamma + c_i d_\gamma ) \|_2 \|d_\gamma\|^2$, we can further upper bound this by $\frac{1}{2} C \sqrt{n} \|d_\gamma\|^2$.

The second term on the RHS of eqn.~\ref{eqn:invar_bound} will be zero when
\begin{align}
\label{eqn:invar_absol}
J_\zeta d_\gamma = d_\theta \ep{,}
\end{align}
which (as we will see) is a condition that is satisfied in certain natural situations. A slightly weakened version of this condition is that $J_\zeta d_\gamma \propto d_\theta$.  Because we have
\begin{align*}
\lim_{\epsilon \rightarrow 0} \frac{\zeta(\gamma + \epsilon d_\gamma) - \zeta(\gamma)}{\epsilon} = J_\zeta d_\gamma
\end{align*}
this condition can thus be interpreted as saying that $d_\gamma$, when translated appropriately via $\zeta$, points in the same direction away from $\theta$ that $d_\theta$ does.  In the smooth path case, where the optimizer only moves an infinitesimally small distance in the direction of $d_\gamma$ (or $d_\theta$) at each iteration before recomputing it at the new $\gamma$ (or $\theta$), this condition is sufficient to establish that the path in $\gamma$ space, when mapped back to $\theta$ space via the $\zeta$ function, will be the same as the path which would have been taken if the optimizer had worked directly in $\theta$ space.  

However, for a practical update scheme where we move the entire distance of $d_\gamma$ or $d_\theta$ before recomputing the update vector, such as the one in eqn.~\ref{eqn:ng_simple_opt}, this kind of invariance will not strictly hold even when $J_\zeta d_\gamma = d_\theta$.  But given that $J_\zeta d_\gamma = d_\theta$, the per-iteration error will be bounded by the first term on the RHS of eqn.~\ref{eqn:invar_bound}, and will thus be small provided that $d_\gamma$ is sufficiently small and $\zeta$ is sufficiently smooth (as shown above).

Now, suppose we generate the updates $d_\theta$ and $d_\gamma$ from curvature matrices $\B_\theta$ and $\B_\gamma$ according to $d_\theta = -\alpha\B_\theta^{-1} \nabla h$ and $d_\gamma = -\alpha\B_\gamma^{-1} \nabla_\gamma h$, where $\nabla_\gamma h$ is the gradient of $h(\zeta(\gamma))$ w.r.t.~$\gamma$.  Then noting that $\nabla_\gamma h = J_\zeta^\top \nabla h$, the condition in eqn.~\ref{eqn:invar_absol} becomes equivalent to
\begin{align*}
J_\zeta \B_\gamma^{-1} J_\zeta^\top \nabla h = \B_\theta^{-1} \nabla h \ep{.}
\end{align*}
For this to hold, a \emph{sufficient} condition is that $\B_\theta^{-1} = J_\zeta \B_\gamma^{-1} J_\zeta^\top$.  Since $J_\zeta$ is invertible (because $\zeta$ is) an equivalent condition is
\begin{align}
\label{eqn:sufficient_cond_invar}
J_\zeta^\top \B_\theta J_\zeta = \B_\gamma \ep{.}
\end{align}

The following theorem summarizes our results so far.

\begin{theorem}
\label{thm:invar}
Suppose that $\theta = \zeta(\gamma)$ and $B_\theta$ and $B_\gamma$ are invertible matrices satisfying
\begin{align*}
J_\zeta^\top B_\theta J_\zeta = B_\gamma
\end{align*}
Then we have that additively updating $\theta$ by $d_\theta = -\alpha B_\theta^{-1} \nabla h$ is \textbf{approximately} equivalent to additively updating $\gamma$ by $d_\gamma = -\alpha B_\gamma^{-1} \nabla_{\gamma} h$, in the sense that $\zeta(\gamma + d_\gamma) \approx \theta + d_\theta$, with error bounded according to
\begin{align*}
\|\zeta(\gamma + d_\gamma) - (\theta + d_\theta)\| \leq \| \zeta(\gamma + d_\gamma) - (\zeta(\gamma) + J_\zeta d_\gamma) \| \ep{.}
\end{align*}
Moreover, this error can be further bounded as in eqn.~\ref{eqn:smoothbnd_taylor}, and will be exactly $0$ if $\zeta$ is affine. And if there is a $C \geq 0$ such that $\|H_{[\zeta]_i}( \gamma )\|_2 \leq C$ for all $i$ and $\gamma$, then we can even further bound this as $\frac{1}{2} C \sqrt{n} \|d_\gamma\|^2$.
\end{theorem}

Because the error bound is zero when $\zeta$ is affine, this result will trivially extend to entire sequences of arbitrary number of steps for such $\zeta$'s. And in the more general case, since the error scales as $\alpha^2$, we can obtain equivalence of sequences of $T/\alpha$ steps in the limit as $\alpha \to 0$.  Because the length of the updates scale as $\alpha$, and we have $T/\alpha$ of them, the sequences converges to smooth paths of fixed length in the limit. The following corollary establishes this result, under a few additional (mild) hypotheses. Its proof is in Appendix \ref{app:cor_invar_proof}.
\begin{corollary}
\label{cor:invar}
Suppose that $B_\theta$ and $B_\gamma$ are invertible matrices satisfying
\begin{align*}
J_\zeta^\top B_\theta J_\zeta = B_\gamma
\end{align*}
for all values of $\theta$.  Then the path followed by an iterative optimizer working in $\theta$-space and using additive updates of the form $d_\theta = -\alpha B_\theta^{-1} \nabla h$ is the same as the path followed by an iterative optimizer working in $\gamma$-space and using additive updates of the form $d_\gamma = -\alpha B_\gamma^{-1} \nabla_{\gamma} h$, provided that the optimizers use equivalent starting points (i.e. $\theta_0 = \zeta(\gamma_0)$), and that either
\begin{itemize}
    \item $\zeta$ is affine,
    \item or $d_\theta / \alpha$ is uniformly continuous as a function of $\theta$, $d_\gamma / \alpha$ is uniformly bounded (in norm), there is a $C$ as in the statement of Theorem \ref{thm:invar}, and $\alpha \to 0$.
\end{itemize}
Note that in the second case we allow the number of steps in the sequences to grow proportionally to $1/\alpha$ so that the continuous paths they converge to have non-zero length as $\alpha \to 0$.
\end{corollary}

%Note that as the hypothesis of $\zeta$ being affine is violated we will have that the error grows in a way that is bounded by the ``degree of violation" of this hypothesis.  In particular, we will have that
%\begin{align*}
%\| \zeta(\gamma + d_\gamma) - (\theta + d_\theta) \| \leq \| \zeta(\gamma + d_\gamma) - (\zeta(\gamma) + J_\zeta d_\gamma) \|
%\end{align*}
%and so in the non-infinitesimal case, the two optimizers may diverge at a rate which grows in rough proportion to the typical value of this error.  

So from these results we see that natural gradient-based methods that take finite steps will \emph{not} be invariant to smooth invertible reparameterizations $\zeta$, although they will be \emph{approximately} invariant, and in a way that depends on the degree of curvature of $\zeta$ and the size $\alpha$ of the step-size.

\subsection{When is the Condition $J_\zeta^\top B_\theta J_\zeta = B_\gamma$ Satisfied?}

Suppose the curvature matrix $B_\theta$ has the form
\begin{align*}
\B_\theta = \ex_{D_{x,y}}[ J_f^\top A J_f ] \ep{,}
\end{align*}
where $D_{x,y}$ is some arbitrary distribution over $x$ and $y$ (such as the training distribution), and $A \in \Real^{m \times m}$ is an invertible matrix-valued function of $x$, $y$ and $\theta$, whose value is parameterization invariant (i.e. its value depends only on the value of $\theta$ that a given $\gamma$ maps to under the $\gamma$ parameterization).  Note that this type of curvature matrix includes as special cases the Generalized Gauss-Newton (whether or not it's equivalent to the Fisher), the Fisher, and the empirical Fisher (discussed in Section \ref{sec:emp_fish}). 

To obtain the analogous curvature matrix $\B_\gamma$ for the $\gamma$ parameterization we replace $f$ by $f \circ \zeta$ which gives
\begin{align*}
\B_\gamma &= \ex_{D_{x,y}}[ J_{f\circ \zeta}^\top \, A \, J_{f\circ \zeta} ] \ep{.}
\end{align*}
Then noting that $J_{f\circ \zeta} = J_f J_\zeta$, where $J_\zeta$ is the Jacobian of $\zeta$, we have
\begin{align*}
\B_\gamma &= \ex_{D_{x,y}}[ (J_f J_\zeta)^\top A (J_f J_\zeta) ] = J_\zeta^\top \ex_{D_{x,y}}[ J_f^\top A J_f ] J_\zeta = J_\zeta^\top \B_\theta J_\zeta \ep{.}
\end{align*}
(Here we have used the fact that the reparameterization function $\zeta$ is independent of $x$ and $y$.) Thus, this type of curvature matrix satisfies the sufficient condition in eqn.~\ref{eqn:sufficient_cond_invar}.   %For example, when $A = I$, the update $\B_\theta^{-1} \nabla h$ is the steepest descent direction w.r.t.~the metric that looks at the norm of the change in $z$.

The Hessian on the other hand does not satisfy this sufficient condition, except in certain special cases.  To see this, note that taking the curvature matrix to be the Hessian gives
\begin{align*}
\B_\gamma = J_\zeta^\top H J_\zeta + \frac{1}{|S|} \sum_{(x,y) \in S} \sum_{j=1}^n [\nabla h]_j H_{[\zeta]_j} \ep{,}
\end{align*}
where $H = \B_\theta$ is the Hessian of $h$ w.r.t.~$\theta$.  Thus, when the curvature matrix is the Hessian, the sufficient condition $J_\zeta^\top \B_\theta J_\zeta = J_\zeta^\top H J_\zeta \propto \B_\gamma$ holds if and only if
\begin{align*}
\frac{1}{|S|} \sum_{(x,y) \in S}  \sum_{j=1}^n [\nabla h]_j H_{[\zeta]_j} = J_\zeta^\top H J_\zeta \ep{,}
\end{align*}
where $\nabla L$ is the gradient of $L(y,z)$ w.r.t.~$z$ (evaluated at $z = f(x,\theta)$), and we allow a proportionality constant of $0$.  Rearranging this gives
\begin{align*}
\frac{1}{|S|} \sum_{(x,y) \in S} \sum_{j=1}^n [\nabla h]_j J_\zeta^{-\top} H_{[\zeta]_j} J_\zeta^{-1} = H \ep{.}
\end{align*}

This relation is unlikely to be satisfied unless the left hand side is equal to $0$.  One situation where this will occur is when $H_{[\zeta]_j} = 0$ for each $j$, which holds when $[\zeta]_j$ is an affine function of $\gamma$.  Another situation is where we have $\nabla h = 0$ for each $(x,y) \in S$. %, which is also equivalent to assuming that the Hessian and the GGN are equal.

%Thus the practical natural gradient iteration \ref{eqn:ng_simple_opt} is invariant to affine reparameterizations, just as Newton's method is.  But unlike Newton's method, which fails to be more generally invariant even in the smooth path case, iteration \ref{eqn:ng_simple_opt} fails only in the large-$\alpha$ case, and then only insofar as the local 1st-order approximation to the reparameterization function $\zeta$ breaks down.  Thus, for sufficiently small $\alpha$'s and sufficiently smooth reparameterizations, invariance of iteration \ref{eqn:ng_simple_opt} holds approximately in some sense.

\section{A New Interpretation of the Natural Gradient}

%In \citet{HF_chapter} it was shown that the quadratic model $M(\delta)$ corresponding to the GGN matrix has a particularly nice interpretation as a squared error objective function.  Given the connection between the Fisher and the GGN, and the interpretation of the natural gradient as the minimizer of the GGN-based quadratic approximation $M(\delta)$ to $h$, this result gives us an interesting interpretation of the natural gradient.

%J_f^\top \nabla \log r
%need sums here:

As discussed in Section \ref{sec:role_of_damping}, the negative natural gradient is given by the minimizer of a local quadratic approximation $M(\delta)$ to $h$ whose curvature matrix is the Fisher $F$.  And if we have that the gradient $\nabla h$ and $F$ are computed on the same set $S$ of data points, $M(\delta$) can be written as
\begin{align*}
M(\delta) &= \frac{1}{2} \delta^\top F \delta + \nabla h^\top \delta + h(\theta) \\
&= \frac{1}{|S|} \sum_{(x,y) \in S} \left[ \frac{1}{2} \delta^\top J_f^\top F_R J_f \delta  +  (J_f^\top \nabla_z \log r(y|z))^\top \delta \right] +  h(\theta) \\
&= \frac{1}{|S|} \sum_{(x,y) \in S} \left[ \frac{1}{2} (J_f \delta)^\top F_R (J_f \delta) +  \nabla_z \log r(y|z)^\top {F_R}^{-1} F_R (J_f \delta) \right. \\
&\hspace{52.5mm} + \frac{1}{2} (\nabla_z \log r(y|z))^\top F_R^{-1} F_R F_R^{-1} \nabla_z \log r(y|z) \\
&\hspace{52.5mm} \left. - \frac{1}{2} (\nabla_z \log r(y|z))^\top F_R^{-1} F_R F_R^{-1} \nabla_z \log r(y|z) \right] + h(\theta)\\
&= \frac{1}{|S|} \sum_{(x,y) \in S} \frac{1}{2} (J_f \delta + F_R^{-1} \nabla_z \log r(y|z))^\top F_R (J_f \delta + F_R^{-1} \nabla_z \log r(y|z)) + c \\
&= \frac{1}{|S|} \sum_{(x,y) \in S} \frac{1}{2} \| J_f \delta + F_R^{-1} \nabla_z \log r(y|z) \|_{F_R}^2 + c \ep{,}
\end{align*}
where %${L''}^{-1}$ is the pseudo-inverse of ${L''}$ and 
$F_R$ is the Fisher of the predictive distribution $R_{y|z}$ (as originally defined in Section \ref{sec:connections_to_GGN}), $\| v \|_{F_R} = \sqrt{v^\top F_R v}$, and  $c = h(\theta) - \frac{1}{2}(\sum_{(x,y) \in S} \nabla_z \log r(y|z)^\top F_R^{-1} \nabla_z \log r(y|z))/|S|$ is a constant (independent of $\delta$).

Note that for a given $(x,y) \in S$, $F_R^{-1} \nabla_z \log r(y|z)$ can be interpreted as the natural gradient direction in $z$-space for an objective corresponding to the KL divergence between the predictive distribution $R_{y|z}$ and a delta distribution on the given $y$.  In other words, it points in the direction which moves $R_{y|z}$ most quickly towards to said delta distribution, as measured by the KL divergence (see Section \ref{sec:geom}).  And assuming that the GGN interpretation of $F$ holds (as discussed in Section \ref{sec:connections_to_GGN}), we know that it also corresponds to the optimal change in $z$ according to the 2nd-order Taylor series approximation of the loss function $L(y,z)$.

Thus, $M(\delta)$ can be interpreted as the sum of squared distances (as measured using the Fisher metric tensor) between these ``optimal" changes in the $z$'s, and the changes in the $z$'s which result from adding $\delta$ to $\theta$, as predicted using 1st-order Taylor-series approximations to $f$.

In addition to giving us a new interpretation for the natural gradient, this expression also gives us an easy-to-compute bound on the largest possible improvement to $h$ (as predicted by $M(\delta)$).  In particular, since the squared error terms are non-negative, we have
\begin{align*}
M(\delta) - h(\theta) \geq -\frac{1}{2|S|}\sum_{(x,y) \in S} \nabla_z \log r(y|z)^\top F_R^{-1} \nabla_z \log r(y|z) \ep{.}
\end{align*}
Given $F_R = H_L$, this quantity has the simple interpretation of being the optimal improvement in $h$ (as predicted by a 2nd-order order model of $L(y,z)$ for each case in $S$) achieved in the hypothetical scenario where we can change the $z$'s independently for each case.  %This bound can also be used to argue that computing $F$ and $\nabla h$ on the same data will tend to make the natural gradient more stable and less prone to ``exploding", as was done by \citet{HF_chapter} for updates based on the GGN.

The existence of this bound shows that the natural gradient can be meaningfully defined even when $F^{-1}$ may not exist, provided that we compute $F$ and $\nabla h$ on the \emph{same data}, and that each $F_R$ is invertible.  In particular, it can be defined as the minimizer of $M(\delta)$ that has minimum norm (which must exist since $M(\delta)$ is bounded below), which in practice could be computed by using the pseudo-inverse of $F$ in place of $F^{-1}$. Other choices are possible, although care would have to be taken to ensure invariance of the choice with respect to parameterization. %Note that if $\nabla h$ is computed on different data this argument breaks down, and unbounded decrease in $M(\delta)$ could be obtained in some direction     %And based on the bound it is also possible to argue that the natural gradient will be less prone to ``exploding". %, as \citet{HF_chapter}

\section{Asymptotic Convergence Speed}
\label{sec:asymptotic_speed}

\subsection{Amari's Fisher Efficiency Result}
\label{sec:Fisher_efficient}

A property of natural gradient descent which is frequently referenced in the literature is that it is ``Fisher efficient".  In particular, \citet{natural_efficient} showed that an iteration of the form
\begin{align}
\label{eqn:general_iter_old}
\theta_{k+1} = \theta_k - \alpha_k \tilde{g}_k(\theta_k)
\end{align}
when applied to an objective of the form discussed in Section \ref{sec:stat_learn}, with $\alpha_k$ shrinking as $1/k$, and with $\tilde{g}_k(\theta_k) = F^{-1} g_k(\theta_k)$ where $g_k(\theta_k)$ is a stochastic estimate of $\nabla h(\theta_k)$ (from a single training case), will produce an estimator $\theta_k$ which is asymptotically ``Fisher efficient".  This means that $\theta_k$ will tend to an unbiased estimator of the global optimum $\theta^*$ of $h(\theta)$, and that its expected squared error matrix (which tends to its variance) will satisfy
\begin{align}
\label{eqn:fisher_efficiency}
\ex[ (\theta_k - \theta^*)(\theta_k - \theta^*)^\top ] = \frac{1}{k} F(\theta^*)^{-1} + \bigO\left(\frac{1}{k^2}\right) \ep{,}
\end{align}
which is (asymptotically) the smallest\footnote{With the usual definition of $\preceq$ for matrices: $A \preceq C$ iff $C-A$ is PSD.} possible variance matrix that any unbiased estimator computed from $k$ training cases can have, according to the Cram\'er-Rao lower bound\footnote{Note that to apply the Cram\'er-Rao lower bound in this context one must assume that the training data set, on which we compute the objective (and which determines $\theta^*$), is infinitely large, or more precisely that its conditional distribution over $y$ has a density function. For finite training sets one can easily obtain an estimator with exactly zero error for a sufficiently large $k$ (assuming a rich enough model class), and so these requirements are not surprising. If we believe that there is a true underlying distribution of the data that has a density function, and from which the training set is just a finite collection of samples, then Cram\'er-Rao can be thought of as applying to the problem of estimating the true parameters of this distribution from said samples (with $F$ computed using the true distribution), and will accurately bound the rate of convergence to the true parameters until we start to see samples repeat. After that point, convergence to the true parameters will slow down and eventually stop, while convergence on the training objective may start to beat the bound. This of course implies that any convergence on the training set that happens faster than the Cram\'er-Rao bound will necessarily correspond to over-fitting (since this faster convergence cannot happen for the test loss).}.

%\JM{Be mindful of how this result may not apply in the finite training set case. i.e. because the Fisher will degenerate for a delta distribution.}

This result can also be straightforwardly extended to handle the case where $g_k(\theta_k)$ is computed using a mini-batch of size $m$ (which uses $m$ independently sampled cases at each iteration), in which case the above asymptotic variance bound becomes
\begin{align*}
\frac{1}{mk} F(\theta^*)^{-1} + \bigO\left(\frac{1}{k^2}\right) \ep{,}
\end{align*}
which again matches the Cram\'er-Rao lower bound.

\begin{center}
\fbox{\begin{varwidth}{\dimexpr\textwidth-2\fboxsep-2\fboxrule\relax}
Note that all expectations in this section will be taken with respect to all random variables, both present and historical (i.e. from previous iterations). So for example, $\ex[\theta_k]$ is computed by marginalizing over the distribution of $g_i$ for all $i<k$. If ``$\theta$" appears inside an expectation without a subscript then it is not an iterate of the optimizer and is instead just treated as fixed non-stochastic value.
\end{varwidth}}
\end{center}

This result applies to the version of natural gradient descent where $F$ is computed using the training distribution $\hat{Q}_x$ and the model's conditional distribution $P_{y|x}$ (see Section \ref{sec:nat_grad}).  If we instead consider the version where $F$ is computed using the true data distribution $Q_x$, then a similar result will still apply, provided that we sample $x$ from $Q_x$ and $y$ from $Q_{y|x}$ when computing the stochastic gradient $g_k(\theta_k)$, and that $\theta^*$ is defined as the minimum of the idealized objective $\KL( Q_{x,y} \| P_{x,y}(\theta) )$ (see Section \ref{sec:stat_learn}).

While this Fisher efficiency result would seem to suggest that natural gradient descent is the best possible optimization method in the stochastic setting, it unfortunately comes with several important caveats and conditions, which we will discuss. (Moreover, as we will later, it is also possessed by much simpler methods, and so isn't a great justification for the use of natural gradient descent by itself.)

Firstly, the proof assumes that the iteration in eqn.~\ref{eqn:general_iter_old} eventually converges to the global optimum $\theta^*$ (at an unspecified speed).  While this assumption can be justified when the objective $h$ is convex (provided that $\alpha_k$ is chosen appropriately), it won't be true in general for non-convex objectives, such as those encountered in neural network training.  In practice however, a reasonable local optimum $\theta^*$ might be a good enough surrogate for the global optimum, in which case a property analogous to Fisher efficiency may still hold, at least approximately. %eqn.~\ref{eqn:fisher_efficiency} may still hold.

%Secondly, it is assumed that $F$ is computed using the true target distribution $Q_x$.  Even if we had access to some finite description of $Q_x$, such a computation is seemingly impossible to perform for deep neural networks.  Moreover, it will often be impractical to even approximate this computation using the training distribution $\hat{Q}_x$, as the training set $S$ can often be very large.  While one usually approximates $F$ using samples from of $x$ from $S$, and while this may work reasonable well in practice (depending on how much data is used, and whether the estimate can be ``accumulated" across multiple iterations), a Fisher efficiency result like the one proved by \citet{natural_efficient} will likely no longer hold in general.  The degree to which such a result may hold approximately for practical methods remains an interesting avenue for future research.

Secondly, it is assumed in Amari's proof that $F$ is computed using the full training distribution $\hat{Q}_x$, which in the case of neural network optimization usually amounts to an entire pass over the training set $S$.  So while the proof allows for the gradient $\nabla h$ to be stochastically estimated from a mini-batch, it doesn't allow this for the Fisher $F$.  This is a serious challenge to the idea that (stochastic) natural gradient descent gives an estimator which makes optimal use of the training data that it sees.   And note that while one can approximate $F$ using minibatches from $S$, which is a solution that often works well in practice (especially when combined with a decayed-averaging scheme\footnote{By this we mean a scheme which maintains an estimate where past contributions decay exponentially at some fixed rate. In other words, we estimate $F$ at each iteration as $(1-\beta) F_{\tmop{new}} + \beta F_{\tmop{old}}$ for some $0<\beta<1$ where $F_{\tmop{new}}$ is the Fisher as computed on the current mini-batch (for the current setting of $\theta$), and $F_{\tmop{old}}$ is the old estimate (which will be based on stale $\theta$ values).}), a Fisher efficiency result like the one proved by \citet{natural_efficient} will likely no longer hold.  Investigating the manner and degree in which it may hold \emph{approximately} when $F$ is estimated in this way is an interesting direction for future research.

A third issue with Amari's result is that it is given in terms of the convergence of $\theta_k$ (as measured by the Euclidean norm) instead of the objective function value, which is arguably much more relevant.  Fortunately, it is straightforward to obtain the former from the latter.  In particular, by applying Taylor's theorem and using $\nabla h(\theta^*) = 0$ we have
\begin{align}
h(\theta_k) - h(\theta^*) &= \frac{1}{2} (\theta_k - \theta^*)^\top {H^*} (\theta_k - \theta^*) + \nabla h(\theta^*)^\top (\theta_k - \theta^*) + \bigO \left( (\theta_k - \theta^*)^3 \right) \notag \\
&= \frac{1}{2} (\theta_k - \theta^*)^\top {H^*} (\theta_k - \theta^*) + \bigO \left ( (\theta_k - \theta^*)^3 \right ) \ep{,}
\label{eqn:Taylor_h}
\end{align}
where ${H^*} = H(\theta^*)$ and $\bigO \left ( (\theta_k - \theta^*)^3 \right )$ is short-hand to mean a function which is cubic in the entries of $\theta_k - \theta^*$. From this it follows\footnote{The last line of this derivation uses $\ex \left[\bigO \left ( (\theta_k - \theta^*)^3 \right ) \right] = o( 1/k )$, which is an (unjustified) assumption that is used in Amari's proof. % and in certain related works \citep[e.g.][]{sgd-qn}.   
This assumption has intuitive appeal since $\ex \left[\bigO \left ( (\theta_k - \theta^*)^2 \right ) \right] = \bigO( 1/k )$, and so it makes sense that $\ex \left[\bigO \left ( (\theta_k - \theta^*)^3 \right ) \right]$ would shrink faster.  However, extreme counterexamples are possible which involve very heavy-tailed distributions on $\theta_k$ over unbounded regions.  By adding some mild hypotheses such as $\theta_k$ being restricted to some bounded region, which is an assumption frequently used in the convex optimization literature, it is possible to justify this assumption rigorously.  Rather than linger on this issue we will refer the reader to \citet{bottou_very_large}, which provides a more rigorous treatment of these kinds of asymptotic results, using various generalizations of the big-O notation.} that
%[[\textbf{actually, the bounding of the cubic term doesn't seem rigorous.  Why should it be $1/k^{3/2}$??  Well, the answer may be due to the fact that the 1st-order convergence theory provides stronger guarantees about the difference between $\theta_k$ and $\theta^*$, and in particular that their distance is bounded (not just the expected distance).  Perhaps we need to qualify these statements with an ``almost surely" clause, or some such thing.  Also, we need to use something better than the cube of norms.  Perhaps whatever notation is used for high-dimensional Taylor's theorem would do.}]]
\begin{align}
\ex[h(\theta_k)] - h(\theta^*) &= \frac{1}{2} \ex \left[ (\theta_k - \theta^*)^\top {H^*} (\theta_k - \theta^*) \right ] + \ex \left[\bigO \left ( (\theta_k - \theta^*)^3 \right ) \right] \notag \\
&= \frac{1}{2} \tr\left( {H^*} \ex \left[ (\theta_k - \theta^*)(\theta_k - \theta^*)^\top \right ] \right) + \ex \left[\bigO \left ( (\theta_k - \theta^*)^3 \right ) \right] \notag \\
&= \frac{1}{2k} \tr\left( {H^*} F(\theta^*)^{-1} \right) + \ex \left[\bigO \left ( (\theta_k - \theta^*)^3 \right ) \right] \notag \\
&= \frac{n}{2k} + o\left( \frac{1}{k} \right) \ep{,} \label{eqn:expected_h_simple}
\end{align}
where we have used ${H^*} = F(\theta^*)$, which follows from the ``realizability" hypothesis used to prove the Fisher efficiency result (see below). Note that while this is the same convergence rate ($\bigO(1/k)$) as the one which appears in \citet{Hazan_newton} (see our Section \ref{sec:emp_fish}), the constant is much better.  However, the comparison is slightly unfair, as \citet{Hazan_newton} doesn't require that the curvature matrix be estimated on the entire data set (as discussed above).  %But it does suggest that the \citet{Hazan_newton} result may be weaker than it otherwise could be and that a stronger 

The fourth and final caveat of Amari's Fisher efficiency result is that Amari's proof assumes that the training distribution $\hat{Q}_{x,y}$ and the optimal model distribution $P_{x,y}(\theta^*)$ coincide, a condition called ``realizability" (which is also required in order for the Cram\'er-Rao lower bound to apply).  This essentially means that the model perfectly captures the training distribution at $\theta = \theta^*$.  This assumption is used in Amari's proof of the Fisher efficiency result to show that the Fisher $F$, when evaluated at $\theta = \theta^*$, is equal to both the empirical Fisher $\bar{F}$ and the Hessian $H$ of $h$. (These equalities follow immediately from $\hat{Q}_{x,y} = P_{x,y}(\theta^*)$ using the forms of the Fisher presented in Section \ref{sec:nat_grad}.) Note that realizability is a subtle condition. It can fail to hold if the model isn't powerful enough to capture the training distribution. But also if the training distribution is a finite set of pairs $(x,y)$ and the model is powerful enough to perfectly capture this (as a Delta distribution), in which case $F(\theta^*)$ is no longer well-defined (because its associated density function isn't), and convergence faster than the Cram\'er-Rao bound becomes possible.

It is not clear from Amari's proof what happens when this correspondence fails to hold at $\theta = \theta^*$, and whether a (perhaps) weaker asymptotic upper bound on the variance might still be provable.  Fortunately, various authors \citep{Murata, bottou_very_large, sgd-qn} building on early work of \citet{amari_oldtheory}, provide some further insight into this question by studying asymptotic behavior of general iterations of the form\footnote{Note that some authors define $\B_k$ to be the matrix that multiplies the gradient, instead of its inverse (as we do instead).}
\begin{align}
\label{eqn:general_iter}
\theta_{k+1} = \theta_k - \alpha_k \B_k^{-1} g_k(\theta_k) \ep{,}
\end{align}
where $\B_k = \B$ is a fixed\footnote{Note that for a non-constant $\B_k$ where $\B_k^{-1}$ converges sufficiently quickly to a fixed $\B^{-1}$ as $\theta_k$ converges to $\theta^*$, these analyses will likely still apply, at least approximately.} curvature matrix (which is independent of $\theta_k$ and $k$), and where $g_k(\theta_k)$ is a stochastic estimate of $\nabla h(\theta_k)$. %(which must be unbiased, have finite variance, and have the property that $\{\epsilon_k = g_k(\theta_k) - \nabla h(\theta_k)\}_k$ are i.i.d.~variables). %, and where $\alpha_k$ shrinks as $\frac{1}{k}$ (\citet{Murata} also does an analysis for the case where $\alpha_k$ remains fixed).  

In particular, \citet{Murata} gives exact (although implicit) expressions for the asymptotic mean and variance of $\theta_k$ in the above iteration for the case where $\alpha_k = 1/(k+1)$ or $\alpha_k$ is constant.  These expressions describe the (asymptotic) behavior of this iteration in cases where the curvature matrix $\B$ is not the Hessian $H$ or the Fisher $F$, covering the non-realizable case, as well as the case where the curvature matrix is only an approximation of the Hessian or Fisher.  \citet{sgd-qn} meanwhile gives expressions for $\ex[h(\theta_k)]$ in the case where $\alpha_k$ shrinks as $1/k$, thus generalizing eqn.~\ref{eqn:expected_h_simple} in a similar manner.  

In the following subsections we will examine these results in more depth, and improve on those of \citet{sgd-qn} (at least in the \emph{quadratic} case) by giving an \emph{exact} asymptotic expression for $\ex[h(\theta_k)]$. We will also analyze iterate averaging (aka Polyak averaging ; see Section \ref{sec:averaging}) in the same setting.%, and in particular show that it too achieves Fisher efficiency.

Some interesting consequences of this analysis are discussed in Sections \ref{sec:consequences_asymptotic_main} and \ref{sec:consequences_averaging}. Of particular note is that for any choice of $\B$, $\ex[h(\theta_k)] - h(\theta_0)$ can be expressed as a sum of two terms: one that scales as $\bigO(1/k)$ and doesn't depend on the starting point $\theta_0$, and one that does depend on the starting point and scales as $\bigO(1/k^2)$ or better.  Moreover, the first term, which is asymptotically dominant, carries all the dependence on the noise covariance, and crucially isn't improved by the use of a non-trivial choices for $\B$ such as $F$ or $H$, assuming the use of Polyak averaging. Indeed, if Polyak averaging is used, this term matches the Cram\'er-Rao lower bound, and thus even plain stochastic gradient descent becomes Fisher efficient! (This also follows from the analysis of \citet{Polyak_averaging}.) Meanwhile, if learning rate decay is used instead of Polyak averaging, one \emph{can} improve the constant on this term by using 2nd-order methods, although not the overall $1/k$ rate.

While these results strongly suggest that 2nd-order methods like natural gradient descent won't be of much help \emph{asymptotically} in the stochastic setting, we argue that the constant on the starting point dependent $\bigO(1/k^2)$ term can still be improved significantly through the use of such methods, and this term may matter more in practice given a limited iteration budget (despite being negligible for very large $k$).

\subsection{Some New Results Concerning Asymptotic Convergence Speed of General Stochastic 2nd-order Methods}
\label{sec:speed_analysis}

%In this subsection we build on the work of \citet{Murata} in order to prove a result that gives detailed expressions for the convergence speed of stochastic 1st and 2nd-order methods based on iterations of the form in eqn.~\ref{eqn:general_iter}. Along the way, we will develop techniques for computing the asymptotic mean and covariance of $\theta_k$.% (as produced by the stochastic iteration in eqn.~\ref{eqn:general_iter}).% using tools from control theory.  %Readers only interested in the results themselves may wish to skip directly to the statement of the theorem and the subsequent discussion of consequences, which appear at the end of this subsection.

In this subsection we will give two results which characterize the asymptotic convergence of the stochastic iteration in eqn.~\ref{eqn:general_iter} as applied to the convex quadratic objective $h(\theta) = \frac{1}{2} (\theta - \theta^*)^\top {H^*} (\theta - \theta^*)$ (whose minimizer is $\theta^*$). The proofs of both results are in Appendix \ref{app:convergence_proofs}.

We begin by defining
\begin{align*}
\Sigma_g(\theta) = \var(g(\theta)) = \ex\left[ (g(\theta)-\ex[g(\theta)]) (g(\theta)-\ex[g(\theta)])^\top \right] \ep{,}
\end{align*}
where $g(\theta)$ denotes the random variable whose distribution coincides with the conditional distribution of $g_k(\theta_k)$ given $\theta_k$. Note that this notation is well-defined as long as $g_k(\theta_k)$ depends only on the value of $\theta_k$, and not on $k$ itself. (This will be true, for example, if the $g_k(\theta_k)$'s are generated by sampling a fixed-size mini-batch of iid training data.)

To simplify our analysis we will assume that $\Sigma_g(\theta)$ is \emph{constant} with respect to $\theta$, allowing us to write it simply as $\Sigma_g$. While somewhat unrealistic, one can reasonably argue that this assumption will become approximately true as $\theta$ converges to the optimum $\theta^*$. It should be noted that convergence of stochastic optimization methods can happen faster than in our analysis, and indeed than is allowed by the Cram\'er-Rao lower bound, if $\Sigma_g(\theta)$ approaches $0$ sufficiently fast as $\theta$ goes to $\theta^*$ . This can happen if the model can obtain zero error on all cases in the training distribution \citep[e.g][]{loizou2017momentum}, a situation which is ruled out by the hypothesis that this distribution has a density function (as is required by Cram\'er-Rao). But it's worth observing that this kind of faster convergence can't happen on the test data distribution (assuming it satisfies Cram\'er-Rao's hypotheses), and thus will only correspond to faster over-fitting.

% \JM{In the convergence section maybe include a discussion of convergence in the situation where the noise covariance goes to zero as one approaches the optimum and how this can change the convergence characteristics. See for example Section 2.5 of https://arxiv.org/pdf/1712.09677.pdf. We can also cite this paper and the ones they cite such as ``Analysis and design of optimization algorithms via integral quadratic constraints" (convex quadratic case?) for examples of analyses that consider vanishing error (which is a stronger condition than realizability).}

Before stating our first result we will define some additional notation. We denote the variance of $\theta_k$ by
\begin{align*}
V_k = \var(\theta_k) = \ex \left[ (\theta_k - \ex[\theta_k])(\theta_k - \ex[\theta_k])^\top \right ] \ep{.}
\end{align*}
And we define the following linear operators\footnote{Note that these operators are \emph{not} $n \times n$ matrices themselves, although they can be represented as $n^2 \times n^2$ matrices if we vectorize their $n \times n$ matrix arguments.  Also note that such operators can be linearly combined and composed, where we will use the standard $\pm$ notation for linear combination, multiplication for composition, and where $I$ will be the identity operator.  So, for example, $(I + \Xi^2)(X) = X + \Xi(\Xi(X))$.} that map square matrices to matrices of the same size:
\begin{align*}
\Xi (X) &= \B^{-1} H^* X + \left(\B^{-1} H^* X\right)^\top = \B^{-1} H^* X + X H^* \B^{-1} \ep{, and} \\
\Psi_\beta(X) &= \left(I - \beta \B^{-1} H^*\right) X \left(I - \beta \B^{-1} H^*\right)^\top \ep{.}
\end{align*}
We also define
\begin{align*}
U = \B^{-1} \Sigma_g \B^{-1}
\end{align*}
for notational brevity, as this is an expression which will appear frequently. Finally, for an $n$-dimensional symmetric matrix $A$ we will denote its $i$-th largest eigenvalue by $\lambda_i(A)$, so that $\lambda_1(A) \geq \lambda_2(A) \geq \ldots \geq \lambda_n(A)$.  

The following theorem represents a more detailed and rigorous treatment of the type of asymptotic expressions for the mean and variance of $\theta_k$ given by \citet{Murata}, although specialized to the quadratic case. Note that symbols like $V_\infty$ have a slightly different meaning here than in \citet{Murata}.
\begin{theorem}
\label{thm:murata_replace}
Suppose that $\theta_k$ is generated by the stochastic iteration in eqn.~\ref{eqn:general_iter} while optimizing a quadratic objective $h(\theta) = \frac{1}{2} (\theta - \theta^*)^\top {H^*} (\theta - \theta^*)$. 

If $\alpha_k$ is equal to a constant $\alpha$ satisfying $\alpha \lambda_1(\B^{-1} H^*) \leq 1$, then the mean and variance of $\theta_k$ are given by
\begin{align*}
\ex[\theta_k] &= \theta^* + (I - \alpha B^{-1} {H^*})^k (\theta_0 - \theta^*) \\
V_k &= \left(I - \Lambda^k\right) (V_\infty) \ep{,}%+ \Lambda^k \left((\theta_0 - \theta^*)(\theta_0 - \theta^*)^\top\right) \ep{,}
\end{align*}
where $\Lambda = \Psi_{\alpha}$ and $V_\infty = \alpha^2 \left(I - \Lambda\right)^{-1} (U)$.

If on the other hand we have $\alpha_k = 1/(k+a+1)$ for some $a \geq 1$, with $b \equiv \lambda_n\left( B^{-1} {H^*} \right) > \frac{1}{2}$ and $\lambda_1\left( B^{-1} {H^*} \right) \leq a + 1$, then the mean and variance of $\theta_k$ are given by
\begin{align*}
\ex[\theta_k] &= \theta^* + \prod_{j=0}^{k-1} \left(I - \alpha_j B^{-1} {H^*} \right) (\theta_0 - \theta^*) \\
%V_k &= \frac{1}{k+1} (\Xi_{B^{-1} {H^*}} - I)^{-1} \Omega_{B^{-1} {H^*}} {H^*}^{-1} \bar{F}(\theta^*) {H^*}^{-1}
%V_k &= \frac{1}{k} \left(\Xi_{B^{-1} {H^*}} - I \right)^{-1} (B^{-1} \Sigma_g B^{-1}) + \bigO\left(\frac{1}{k^2}\right)
V_k &= \frac{1}{k+a} \left( \Xi - I\right)^{-1}\left(U\right) + E_k \ep{,}
\end{align*}
where $E_k$ is a matrix valued ``error" that shrinks as $1/k^2$.
\end{theorem}
\begin{remark}
Due to the properties of quadratic functions, and the assumptions of constant values for $B$ and $\Sigma_g$, one could state and/or prove this theorem (and the ones that follow) while taking one of ${H^*}$, $B$, or $\Sigma_g$ to be the identity matrix, without any loss of generality. This is due to the fact that optimizing with a preconditioner is equivalent to optimizing a linearly-reparameterized version of the objective with plain stochastic gradient descent.
\end{remark}

%While powerful and very general, Theorem \ref{thm:murata_replace} is difficult to interpret.  

One interesting observation that we can immediately make from Theorem \ref{thm:murata_replace} is that, at least in the case where the objective is a convex quadratic, $\ex[\theta_k]$ progresses in a way that is fully independent of the distribution of noise in the gradient estimate (which is captured by the $\Sigma_g$ matrix).  Indeed, it proceeds as $\theta_k$ itself would in the case of fully deterministic optimization.  It is only the variance of $\theta_k$ around $\ex[\theta_k]$ that depends on the gradient estimator's noise.  

To see why this happens, note that if $h(\theta)$ is quadratic then $\nabla h(\theta)$ will be an affine function, and thus will commute with expectation. This allows us to write
\begin{align*}
\ex[ g(\theta_k) ] = \ex[ \nabla h(\theta_k) ] = \nabla h(\ex[\theta_k]) \ep{.}
\end{align*}
Provided that $\alpha_k$ doesn't depend on $\theta_k$ in any way (as we are implicitly assuming), we then have
\begin{align*}
\ex[\theta_{k+1}] = \ex[ \theta_k - \alpha_k \B^{-1} g(\theta_k) ] = \ex[\theta_k] - \alpha_k \B^{-1} \nabla h(\ex[\theta_k]) \ep{,}
\end{align*}
which is precisely the deterministic version of eqn.~\ref{eqn:general_iter}, where we treat $\ex[\theta_k]$ as the parameter vector being optimized].

While Theorem \ref{thm:murata_replace} provides a detailed picture of how well $\theta^*$ is estimated by $\theta_k$, it doesn't tell us anything directly about how quickly progress is being made on the objective, which is arguably a much more relevant concern in practice.  Fortunately, as observed by \citet{Murata}, we have the basic identity (proved for completeness in Appendix \ref{app:basic_ident_proof}):
\begin{align*}
\ex\left[ (\theta_k - \theta^*)(\theta_k - \theta^*)^\top \right] &= V_k + (\ex[\theta_k] - \theta^*)(\ex[\theta_k] - \theta^*)^\top \ep{.}
\end{align*}
And it thus follows that
\begin{align}
\ex[h(\theta_k)] - h(\theta^*) &= \frac{1}{2} \tr\left( {H^*} \ex \left[ (\theta_k - \theta^*)(\theta_k - \theta^*)^\top \right ] \right) %+ \ex \left[\bigO \left ( (\theta_k - \theta^*)^3 \right ) \right] 
\notag \\
&= \frac{1}{2} \tr \left( {H^*} \left(V_k + (\ex[\theta_k] - \theta^*)(\ex[\theta_k] - \theta^*)^\top \right) \right) %+ \ex \left[\bigO \left ( (\theta_k - \theta^*)^3 \right ) \right] 
\notag \\
&= \frac{1}{2} \tr \left( {H^*} V_k \right) + \frac{1}{2} \tr \left( {H^*} (\ex[\theta_k] - \theta^*)(\ex[\theta_k] - \theta^*)^\top \right) %+ \ex \left[\bigO \left ( (\theta_k - \theta^*)^3 \right ) \right] \ep{,}
\label{eqn:obj_expression} 
\end{align}
which allows us to relate the convergence of $E[\theta_k]$ (which behaves like $\theta_k$ in the deterministic version of the algorithm) and the size/shape of the variance of $\theta_k$ to the convergence of $\ex[h(\theta_k)]$.  In particular, we see that in this simple case where $h(\theta)$ is quadratic, $\ex[h(\theta_k)] - h(\theta^*)$ neatly decomposes as the sum of two independent terms that quantify the roles of these respective factors in the convergence of $\ex[h(\theta_k)]$ to $h(\theta^*)$.

In the proof of the following theorem (which is in the appendix), we will use the above expression and Theorem \ref{thm:murata_replace} to precisely characterize the asymptotic convergence of $\ex[h(\theta_k)]$.  Note that while \citet{Murata} gives expressions for this as well, they cannot be directly evaluated except in certain special cases (such as when $B = H$), and only include the asymptotically dominant terms.

\begin{theorem}
\label{thm:asymptotic_main}
Suppose that $\theta_k$ is generated by the stochastic iteration in eqn.~\ref{eqn:general_iter} while optimizing a quadratic objective $h(\theta) = \frac{1}{2} (\theta - \theta^*)^\top {H^*} (\theta - \theta^*)$. 

If $\alpha_k$ is equal to a constant $\alpha$ satisfying $\alpha \lambda_1(\B^{-1} H^*) \leq 1$, then the expected objective $\ex[h(\theta_k)]$ satisfies
\begin{align*}
l(k) \leq  \ex[h(\theta_k)] - h(\theta^*) \leq u(k) \ep{,}
\end{align*}
where
\begin{align*}
u(k) &= \left[ 1 - \left(1 - \epsilon_1 \right)^{2k} \right] \frac{\alpha}{4} \tr\left( \left(\B - \frac{\alpha}{2} {H^*}\right)^{-1} \Sigma_g \right) + \left(1 - \epsilon_2 \right)^{2k} h(\theta_0)
\end{align*}
and
\begin{align*}
l(k) &= \left[ 1 - \left(1 - \epsilon_2 \right)^{2k} \right] \frac{\alpha}{4} \tr\left( \left(\B - \frac{\alpha}{2} {H^*}\right)^{-1} \Sigma_g \right) + \left(1 - \epsilon_1 \right)^{2k} h(\theta_0) \ep{,}
\end{align*}
with $\epsilon_1 = \lambda_1(C) = \alpha \lambda_1\left(B^{-1} {H^*}\right)$ and $\epsilon_2 = \lambda_n(C) = \alpha \lambda_n\left(B^{-1} {H^*}\right)$.

\vspace{0.3in}
If on the other hand we have $\alpha_k = 1/(k+a+1)$ for some $a \geq 1$, with $b \equiv \lambda_n\left( B^{-1} {H^*} \right) > \frac{1}{2}$ and $\lambda_1\left( B^{-1} {H^*} \right) \leq a + 1$, then the expected objective satisfies
\begin{align*}
l(k) \leq  \ex[h(\theta_k)] - h(\theta^*) \leq u(k) \ep{,}
\end{align*}
where
\begin{align*}
u(k) = \frac{1}{4(k+a)} \tr\left(\left(\B^{-1} - \frac{1}{2} {H^*}^{-1} \right)^{-1} U\right) &+ \frac{\nu(a) k}{4(k+a)^3} \tr\left(\left(\B^{-1} - \frac{1}{2} {H^*}^{-1} \right)^{-1} \Psi_1(U) \right) \\
&\quad + h(\theta_0) \left(\frac{1+a}{k+a}\right)^{2 b}
\end{align*}
and
\begin{align*}
l(k) = \frac{1}{4(k+a)} \tr\left(\left(\B^{-1} - \frac{1}{2} {H^*}^{-1} \right)^{-1} U\right) - \frac{1}{4a} \left(\frac{1+a}{k+a}\right)^{2b} \tr\left(\left(\B^{-1} - \frac{1}{2} {H^*}^{-1} \right)^{-1} U\right) \ep{,}
\end{align*}
where $\nu(a) = (a+2)^3/(a(a+1)^2)$.

\end{theorem}

% \begin{remark}
% As with the theorem on which it is based (Theorem \ref{thm:murata_replace}), the above theorem can likely be extended to handle non-quadratic objectives (at least in the case where $\alpha_k = 1/(k+1)$).
% \end{remark}

\subsubsection{Some Consequences of Theorem \ref{thm:asymptotic_main}} 
\label{sec:consequences_asymptotic_main}

In the case of a fixed step-size $\alpha_k = \alpha$, Theorem \ref{thm:asymptotic_main} shows that $\ex[h(\theta_k)]$ will tend to the constant
\begin{align*}
h(\theta^*) + \frac{\alpha}{4} \tr\left( \left(\B - \frac{\alpha}{2} {H^*}\right)^{-1} \Sigma_g \right) \ep{.}
\end{align*}
The size of this extra additive factor is correlated with the step-size $\alpha$ and gradient noise covariance $\Sigma_g$.  If the covariance or step-sizes are small enough, it may not be very large in practice.

Moreover, one can use the fact that the iterates $\{\theta_k\}_{k=1}^\infty$ are (non-independent) asymptotically unbiased estimators of $\theta^*$ to produce an asymptotically unbiased estimator with shrinking variance by averaging them together. This is done in the Polyak Averaging method \citep[e.g.][]{Polyak_averaging}, which we analyze in Section \ref{sec:averaging}.  % (see \citet{kevin} for a good discussion of this).

In the scenario where $\alpha_k = 1/(k+a+1)$, if one performs stochastic 2nd-order optimization with $\B = H^*$ (so that $b=1$) and any $a \geq 1$, Theorem \ref{thm:asymptotic_main} gives that
\begin{align*}
\ex[h(\theta_k)] - h(\theta^*) \leq \frac{1}{2(k+a)} \tr\left({H^*}^{-1} \Sigma_g\right) + h(\theta_0) \left(\frac{1+a}{k+a}\right)^2
\end{align*}
(where we have used the fact that $\Psi_1 = 0$ when $B = H^*$). And if one considers the scenario corresponding to 1st-order optimization where we take $B = \lambda_n(H^*) I$ (so that $b = 1$) and $a = \kappa(H^*)$, where $\kappa(H^*) = \lambda_1({H^*}) / \lambda_n({H^*})$ is the condition number of $H^*$, we get
\begin{align*}
\ex[h(\theta_k)] - h(\theta^*) \leq &\left(\frac{1}{4(k+\kappa(H^*))} + \bigO\left(\frac{1}{(k+\kappa(H^*))^2}\right)\right) \frac{1}{\lambda_n({H^*})}\\
& \quad \cdot \tr\left(\left(I - \frac{\lambda_n({H^*})}{2} {H^*}^{-1} \right)^{-1} \Sigma_g \right)
+ h(\theta_0) \left(\frac{1+\kappa(H^*)}{k+\kappa(H^*)}\right)^2 \ep{.}
\end{align*}
In deriving this we have applied Lemma \ref{lemma:PSDtrace_bound} while exploiting the fact that $\Psi_1(U) \preceq U$ (i.e. $U - \Psi_1(U)$ is PSD).

These two bounds can be made more similar looking by choosing $a = \kappa(H^*)$ in the first one. (However this is unlikely to be the optimal choice in general, and the extra freedom in choosing $a$ seems to be one of the advantages of using $\B = H^*$.) In this case, the starting-point dependent terms (which are noise independent) seem to exhibit the same asymptotics, although this is just an artifact of the analysis, and it is possible to obtain tighter bounds on these terms by considering the entire spectrum instead of just the most extreme eigenvalue (as was done in eqn.~\ref{eqn:crude_bound_step} while applying Lemma \ref{lemma:wang_basic_bound}).

The noise-dependent terms, which are the ones that dominate asymptotically as $k \to \infty$ and were derived using a very tight analysis (and indeed we have a matching lower bound for these), exhibit a more obvious difference in the above expressions.  To compare their sizes we can apply Lemma \ref{lemma:wang_basic_bound} to obtain the following bounds (see Appendix \ref{app:extra_deriv2}):
\begin{align*}
\frac{1}{2\lambda_1 \left({H^*} \right)}\tr(\Sigma_g) \leq \frac{1}{2}\tr\left( {H^*}^{-1} \Sigma_g\right) \leq \frac{1}{2\lambda_n \left({H^*} \right)}\tr(\Sigma_g)
\end{align*}
and
\begin{align*}
\frac{1}{4\lambda_n({H^*})} \tr(\Sigma_g) \leq  \frac{1}{4\lambda_n({H^*})}\tr\left(\left(I - \frac{\lambda_n({H^*})}{2} {H^*}^{-1} \right)^{-1} \Sigma_g \right) \leq \frac{1}{2\lambda_n({H^*})} \tr(\Sigma_g) \ep{.}
\end{align*}
Because the lower bound is much smaller in the $B = H^*$ case, these bounds thus allow for the possibility that the noise dependent term will be much smaller in that case. A necessary condition for this to happen is that $H^*$ is ill-conditioned (so that $\lambda_1 \left({H^*} \right) \gg \lambda_n \left({H^*} \right)$), although this alone is not sufficient.  

To provide an actual concrete example where the noise-dependent term is smaller, we must make further assumptions about the nature of the gradient noise covariance matrix $\Sigma_g$.  As an important example, we consider the scenario where the stochastic gradients are computed using (single) randomly sampled cases from the training set $S$, and where we are in the realizable regime (so that $H^* = \bar{F^*} = \Sigma_g$; see Section \ref{sec:Fisher_efficient}). When $B = H^*$, the constant on the noise dependent term will thus scale as
\begin{align*}
\frac{1}{2}\tr\left( {H^*}^{-1} \Sigma_g\right) = \frac{1}{2}\tr\left( {H^*}^{-1} {H^*}\right) = \frac{n}{2} \ep{,}
\end{align*}
while if $B = \lambda_n({H^*}) I$ it will scale as
\begin{align*}
\frac{1}{4\lambda_n({H^*})} \tr\left(\left(I - \frac{\lambda_n({H^*})}{2} {H^*}^{-1} \right)^{-1} \Sigma_g \right) &= \frac{1}{4\lambda_n({H^*})} \tr\left(\left(I - \frac{\lambda_n({H^*})}{2} {H^*}^{-1} \right)^{-1} {H^*} \right) \\
&= \frac{1}{4\lambda_n({H^*})} \sum_{i=1}^n \frac{\lambda_i({H^*})}{1 - \frac{\lambda_n({H^*})}{2\lambda_i({H^*})} } = \frac{1}{4} \sum_{i=1}^n \frac{r_i}{1 - \frac{1}{2r_i} } \ep{,}
\end{align*}
where we have defined $r_i = \lambda_i({H^*}) / \lambda_n({H^*})$. (To go from the first to the second line we have used the general fact that for any rational function $g$ the $i$-th eigenvalue of $g(H^*)$ is given by $g(\lambda_i)$. And also that the trace is the sum of the eigenvalues.)

Observing that $1 \leq r_i$, we thus have $r_i \leq r_i / (1-1/(2 r_i))$, from which it also follows that
\begin{align*}
\frac{1}{4} \frac{\tr(H^*)}{\lambda_n(H^*)} \: = \: \frac{1}{4}\sum_{i=1}^n r_i \: \leq  \: \frac{1}{4}\sum_{i=1}^n \frac{r_i}{1 - \frac{1}{2r_i} } \ep{.}
\end{align*}
and
\begin{align*}
\frac{n}{2} \: \leq  \: 2 \cdot \frac{1}{4}\sum_{i=1}^n \frac{r_i}{1 - \frac{1}{2r_i} } \ep{.}
\end{align*}
From these bounds we can see that noise scaling in the $B = H^*$ case is no worse than twice that of the $B = \lambda_n({H^*}) I$ case.  And it has the potential to be much smaller, such as when $\frac{\tr(H^*)}{\lambda_n(H^*)} \gg n$, or when the spectrum of ${H^*}$ covers a large range. For example, if $\lambda_i({H^*}) = n-i+1$ then the noise scales as $\Omega(n^2 / k)$ in the $B = \lambda_n({H^*}) I$ case.

\subsubsection{Related Results}

\label{sec:related_bound_results_1}

The related result most directly comparable to Theorem \ref{thm:asymptotic_main} is Theorem 1 of \citet{sgd-qn}, which provides upper and lower bounds for $\ex[h(\theta_k)] - h(\theta^*)$ in the case where $\alpha_k = 1/(k+a+1)$ and $\lambda_n\left(B^{-1} {H^*}\right) > 1/2$.   In particular, using a different technique from our own, \citet{sgd-qn} show that\footnote{Note that the notation `$B$' as it is used by \citet{sgd-qn} means the \emph{inverse} of the matrix $B$ as it appears in this paper.  And while \citet{sgd-qn} presents their bounds with $\bar{F}$ in place of $\Sigma_g$, these are the same matrix when evaluated at $\theta = \theta^*$ as we have $\ex[g(\theta^*)] = 0$ (since $\theta^*$ is a local optimum).}
\begin{align*}
\frac{1}{k} \cdot \frac{\tr({H^*} U)}{4 \left(\lambda_1(\B^{-1} {H^*}) - \frac{1}{2}\right)}  + o\left( \frac{h(\theta_0)}{k} \right) \: \leq \: \ex[h(\theta_k)] - h(\theta^*) \: \leq \: \frac{1}{k} \cdot & \frac{\tr({H^*} U)}{4 \left(\lambda_n( \B^{-1} {H^*}) - \frac{1}{2}\right)} \\ 
&+ o\left( \frac{h(\theta_0)}{k} \right) \ep{.}
\end{align*}

This result is more general in the sense that it doesn't assume the objective is quadratic.  However, it is also far less detailed than our result, and in particular doesn't describe the asymptotic value of $\ex[h(\theta_k)] - h(\theta^*)$, instead only giving (fairly loose) upper and lower bounds on it. It can be obtained from our Theorem \ref{thm:asymptotic_main}, in the quadratic case, by a straightforward application of Lemma \ref{lemma:wang_basic_bound}.

There are other relevant results in the vast literature on general strongly convex functions, such as \citet{kushner2003stochastic} and the references therein, and \citet{moulines2011non}. These results, while usually only presented for standard stochastic gradient descent, can be applied to the same setting considered in Theorem \ref{thm:asymptotic_main} by performing a simple linear reparameterization.  A comprehensive review of such results would be outside the scope of this report, but to the best of our knowledge there is no result which would totally subsume Theorem \ref{thm:asymptotic_main} for convex quadratics. The bounds in \citet{moulines2011non} for example, are more general in a number of ways, but also appear to be less detailed than ours, and harder to interpret.

\subsection{An Analysis of Averaging}
\label{sec:averaging}

%of the case where $\alpha_k = \alpha$ and $2 \alpha \lambda_1( B^{-1} {H^*} ) < 1$

In this subsection we will extend the analysis from Subsection \ref{sec:speed_analysis} to incorporate basic iterate averaging of the standard type \citep[e.g.][]{Polyak_averaging}.  In particular, we will bound $\ex[h(\bar{\theta}_k)]$ where
\begin{align*}
\bar{\theta}_k = \frac{1}{k+1} \sum_{i = 0}^k \theta_i \ep{.}
\end{align*}

Note that while this type of averaging leads to elegant bounds (as we will see), a form of averaging based on an exponentially-decayed moving average typically works much better in practice.  This is given by
\begin{align*}
\bar{\theta}_k = (1-\beta_k) \theta_k + \beta_{k}\bar{\theta}_{k-1} \quad\quad\quad\quad \bar{\theta}_0 = \theta_0
\end{align*}
for $\beta_k = \min\{ 1-1/k, \beta_{\max}\}$ with $0 < \beta_{\max} < 1$ close to 1 (e.g. $\beta_{\max} = 0.99$).  This type of averaging has the advantage that it more quickly ``forgets" the very early $\theta_i$'s, since their ``weight" in the average decays exponentially.  However, the cost of doing this is that the variance will never converge exactly to zero, which arguably matters more in theory than it does in practice.

The main result of this subsection, which is proved in Appendix \ref{app:convergence_proofs}, is stated as follows:
\begin{theorem}
\label{thm:averaging}
Suppose that $\theta_k$ is generated by the stochastic iteration in eqn.~\ref{eqn:general_iter} with constant step-size $\alpha_k = \alpha$ while optimizing a quadratic objective $h(\theta) = \frac{1}{2} (\theta - \theta^*)^\top {H^*} (\theta - \theta^*)$. Further, suppose that $\alpha \lambda_1(B^{-1} {H^*}) < 1$, and define $\bar{\theta}_k = \frac{1}{k+1} \sum_{i = 0}^k \theta_i$. Then we have the following bound:
\begin{align*}
\ex[h(\bar{\theta}_k)] - h(\theta^*) &\leq \min \left \{\frac{1}{2(k+1)} \tr\left({H^*}^{-1} \Sigma_g\right) , \:\: \frac{\alpha}{4} \tr\left( \left(\B - \frac{\alpha}{2} {H^*}\right)^{-1} \Sigma_g \right) \right\}  \\
&\quad\quad\quad\quad + \min \left \{ \frac{1}{2(k+1)^2\alpha^2} \left\| {H^*}^{-1/2} B (\theta_0 - \theta^*) \right\|^2, \right. \\
&\quad\quad\quad\quad \quad\quad\quad \quad\!\!\! \left. \frac{1}{2(k+1) \alpha} \left\| B^{1/2} (\theta_0 - \theta^*) \right\|^2, \:\: h(\theta_0) \right \} \ep{.}
\end{align*}
\end{theorem}

Note that this bound can be written asymptotically ($k \to \infty$) as 
\begin{align*}
    \ex[h(\bar{\theta}_k)] - h(\theta^*) \leq \bigO \left( \frac{1}{2(k+1)}\tr \left({H^*}^{-1} \Sigma_g \right) \right) \ep{,}
\end{align*}
which notably doesn't depend on either $\alpha$ or $\B$. 

This is a somewhat surprising property of averaging to be sure, and can be intuitively explained as follows.  Increasing the step-size along any direction $d$ (as measured by $\alpha d^\top \B^{-1} d$) will increase the variance in that direction for each iterate (since the step-size multiplies the noise in the stochastic gradient estimate), but will also cause the iterates to decorrelate faster in that direction (as can be seen from eqn.~\ref{eqn:avg_decorr}).  Increased decorrelation in the iterates leads to lower variance in their average, which counteracts the aforementioned increase in the variance. As it turns out, these competing effects will exactly cancel in the limit, which the proof of Theorem \ref{thm:averaging} rigorously establishes.

\subsubsection{Some Consequences of Theorem \ref{thm:averaging}} 
\label{sec:consequences_averaging}

In the case of stochastic 2nd-order optimization of a convex quadratic objective where we take $B = {H^*}$ (which allows us to use an $\alpha$ close to $1$) this gives
\begin{align*}
\ex[h(\bar{\theta}_k)] - h(\theta^*) &\leq \frac{\tr\left( {H^*}^{-1} \Sigma_g \right)}{2(k+1)} + \frac{h(\theta_0)}{(k+1)^2\alpha^2} \ep{.}
\end{align*}
Then choosing the maximum allowable value of $\alpha$ this becomes
\begin{align*}
\ex[h(\bar{\theta}_k)] - h(\theta^*) &\leq \frac{\tr\left( {H^*}^{-1} \Sigma_g \right)}{2(k+1)} + \frac{h(\theta_0)}{(k+1)^2} \ep{,}
\end{align*}
which is a similar bound to the one described in Section \ref{sec:consequences_asymptotic_main} for stochastic 2nd-order optimization (with $\B = H^*$) using an annealed step-size $\alpha_k = 1/(k+1)$.

For the sake of comparison, applying Theorem \ref{thm:averaging} with $B = I$ gives
\begin{align}
\ex[h(\bar{\theta}_k)] - h(\theta^*) \leq \frac{\tr\left( {H^*}^{-1} \Sigma_g \right)}{2(k+1)} + \frac{\left\| {H^*}^{-1/2}(\theta_0 - \theta^*) \right\|^2}{2(k+1)^2\alpha^2} \label{eqn:bound_avg_B=I}
\end{align} 
under the assumption that $\alpha \lambda_1({H^*}) < 1$.  For the maximum allowable value of $\alpha$ this becomes
\begin{align*}
\ex[h(\bar{\theta}_k)] - h(\theta^*) \leq \frac{\tr\left( {H^*}^{-1} \Sigma_g \right)}{2(k+1)} + \frac{ \lambda_1({H^*})^2 \left\| {H^*}^{-1/2}(\theta_0 - \theta^*) \right\|^2}{2(k+1)^2} \ep{.}
\end{align*} 
%To compare this bound to the $B = H^*$ case we note that $\lambda_1({H^*})^2 \| {H^*}^{-1/2}(\theta_0 - \theta^*) \|^2 \geq 2h(\theta_0)$.

An interesting observation we can make about these bounds is that they \emph{do not} demonstrate any improvement through the use of 2nd-order optimization on the \emph{asymptotically dominant} noise-dependent term in the bound (a phenomenon first observed by \citet{Polyak_averaging} in a more general setting).  Moreover, in the case where the stochastic gradients (the $g_k(\theta_k)$'s) are sampled using random training cases in the usual way so that $\Sigma_g = \bar{F}(\theta)$, and the realizability hypothesis is satisfied so that $H^* = F(\theta^*) = \bar{F}(\theta^*)$ (see Section \ref{sec:Fisher_efficient}), we can see that simple stochastic gradient descent with averaging achieves a similar \emph{asymptotic} convergence speed (given by $n/(k+1) + o(1/k)$) to that possessed by Fisher efficient methods like stochastic natural gradient descent (c.f. eqn.~\ref{eqn:expected_h_simple}), despite not involving the use of curvature matrices!

Moreover, in the non-realizable case, $1/(2k) \tr\left( {H^*}^{-1} \Sigma_g\right)$ turns out to be the same asymptotic rate as that achieved by the ``empirical risk minimizer" (i.e. the estimator of $\theta$ that minimizes the expected loss over the training cases processed thus far) and is thus ``optimal" in a certain strong sense for infinite data sets. See \citet{frostig2014competing} for a good recent discussion of this.

However, despite these observations, these bounds \emph{do} demonstrate an improvement to the noise-independent term (which depends on the starting point $\theta_0$) through the use of 2nd-order optimization. When $H^*$ is ill-conditioned and $\theta_0 - \theta^*$ has a large component in the direction of eigenvectors of $H^*$ with small eigenvalues, we will have
\begin{align*}
\lambda_1({H^*})^2 \| {H^*}^{-1/2}(\theta_0 - \theta^*) \|^2 \gg h(\theta_0) \ep{.}
\end{align*}
Crucially, this noise-independent term may often matter more in practice (despite being \emph{asymptotically} negligible), as the LHS expression may be very large compared to $\Sigma_g$, and we may be interested in stopping the optimization long before the more slowly shrinking noise-dependent term begins to dominate asymptotically (e.g. if we have a fixed iteration budget, or are employing early-stopping).  This is especially likely to be the case if the gradient noise is low due to the use of large mini-batches.%, and the local asymptotic convergence setting may only serve as a rough analogy to what is actually going on (see the discussion in Subsection \ref{sec:bounds_useful?}.

It is also worth pointing out that compared to standard stochastic 2nd-order optimization with a fixed step-size (as considered by the first part of Theorem \ref{thm:asymptotic_main}), the noise-independent term shrinks much more slowly when we use averaging (quadratically vs exponentially), or for that matter when we use an annealed step-size $\alpha_k = 1/(k+1)$ with $b > 1$.  This seems to be the price one has to pay in order to ensure that the noise-dependent term shrinks as $1/k$.  (Although in practice one can potentially obtain a more favorable dependence on the starting point by adopting the ``forgetful" exponentially-decaying variant of averaging discussed previously.)

%[[\textbf{perhaps we should comment on how averaging negatively affects the ``deterministic" part of the convergence, similarly to how the annealed step-size does this.  And both things are designed to aid in the stochastic convergence part.  Then maybe we should point out that the use of a window in averaging will mitigate this problem, at the cost of losing strict theoretical guarantees about the stochastic convergence}]]

\subsubsection{Related Results}

\label{sec:related_bound_results_2}

Under weaker assumptions about the nature of the stochastic gradient noise, \citet{Polyak_averaging} showed that
\begin{align*}
\ex \left[ (\bar{\theta}_k - \theta^*)(\bar{\theta}_k - \theta^*)^\top \right ] = \frac{1}{k+1} {H^*}^{-1/2} \Sigma_g {H^*}^{-1/2}  +   o\left(\frac{1}{k}\right) \ep{,}
\end{align*}
which using the first line of eqn.~\ref{eqn:obj_expression} yields
\begin{align*}
\ex[h(\bar{\theta}_k)] - h(\theta^*) = \frac{\tr\left( {H^*}^{-1} \Sigma_g \right)}{2(k+1)} + o\left(\frac{1}{k}\right) \ep{.}
\end{align*}
While consistent with Theorem \ref{thm:averaging}, this bound gives a less detailed picture of convergence, and in particular fails to quantify the relative contribution of the noise-dependent and independent terms, and thus doesn't properly distinguish between the behavior of stochastic 1st or 2nd-order optimization methods (i.e. $B = I$ vs $B = {H^*}$).

%\ex[h(\theta_k)] - h(\theta^*) &= \frac{1}{2} \tr\left( {H^*} \ex \left[ (\theta_k - \theta^*)(\theta_k - \theta^*)^\top \right ]

Assuming a model for the gradient noise which is consistent with linear least-squares regression and $B = I$, \citet{defossez2014constant} showed that 
\begin{align*}
\ex[h(\bar{\theta}_k)] - h(\theta^*) \approx \frac{\tr\left( {H^*}^{-1} \Sigma_g \right)}{k+1} + \frac{\left\| {H^*}^{-1/2}(\theta_0 - \theta^*) \right\|^2}{(k+1)^2\alpha^2}
\end{align*}
holds in the asymptotic limit as $\alpha \to 0$ and $k \to \infty$.  

This expression is similar to the one generated by Theorem \ref{thm:averaging} (see eqn.~\ref{eqn:bound_avg_B=I}), although it only holds in the asymptotic limit of small $\alpha$ and large $k$, and assumes a particular $\theta$-dependent form for the noise (arising in least-squares linear regression), which represents neither a subset nor a super-set of our general $\theta$-independent formulation. An interesting question for future research is whether Theorem \ref{thm:murata_replace} could be extended in way that would allow $\Sigma_g$ to vary with $\theta$, and whether this would allow us to prove a more general version of Theorem \ref{thm:averaging} that would cover the case of linear least-squares.

A result which is more directly comparable to our Theorem \ref{thm:averaging} is ``Theorem 3" of \citet{stepsize_flammarion}, which when applied to the same general case considered in Theorem \ref{thm:averaging} gives the following upper bound (assuming that $B = I$\footnote{Note that the assumption that $B = I$ doesn't actually limit this result since stochastic 2nd-order optimization of a quadratic using a \emph{fixed} $B$ can be understood as stochastic gradient descent applied to a transformed version of the original quadratic (with an appropriately transformed gradient noise matrix $\Sigma_g$).} and $\alpha \lambda_1({H^*}) \leq 1$):
\begin{align*}
\ex[h(\bar{\theta}_k)] - h(\theta^*) \leq 4 \alpha \tr(\Sigma_g) + \frac{\|\theta_0 - \theta^*\|^2}{(k+1)\alpha } \ep{.}
\end{align*}
Unlike the bound proved in Theorem \ref{thm:averaging}, this bound fails to establish that $\ex[h(\bar{\theta_k})]$ even converges, since the term $4 \alpha \tr(\Sigma_g)$ is constant in $k$. 

There are other older results in the literature analyzing averaging in general settings, such as those contain in \citet{kushner2003stochastic} and the references therein. However, to the best of our knowledge there is no result in the literature which would totally subsume Theorem \ref{thm:averaging} for the case of convex quadratics, particularly with regard to the level of detail in the non-asymptotically-dominant terms of the bound (which is very important for our purposes here).

%Unlike the situation without averaging, we will find that $\ex[h(\bar{\theta}_k)]$ does indeed converge to $h(\theta^*)$, although at the cost of making some of the scalar multipliers on these terms worse.

\section{Conclusions and Open Questions}
\label{sec:future_work}

In this report we have examined several aspects of the natural gradient method, such as its relationship to 2nd-order methods, its local convergence speed, and its invariance properties.  

The link we have established between natural gradient descent and (stochastic) 2nd-order optimization with the Generalized Gauss-Newton matrix (GGN) provides intuition for why it might work well with large step-sizes, and gives prescriptions for how to make it work robustly in practice. In particular, we advocate viewing natural gradient descent as a GGN-based 2nd-order method in disguise (assuming the equivalence between the Fisher and GGN holds), and adopting standard practices from the optimization literature to ensure fast and robust performance, such as trust-region/damping/Tikhonov regularization, and Levenberg-Marquardt adjustment heuristics \citep{levenberg_marquardt}.  

However, even in the case of squared loss, where the GGN becomes the standard Gauss-Newton matrix, we don't yet have a completely rigorous understanding of 2nd-order optimization with the GGN.  A completely rigorous account of its global convergence remains elusive (even if we can assume convexity), and convergence rate bounds such as those proved in Section \ref{sec:asymptotic_speed} don't even provide a complete picture of its \emph{local} convergence properties. 

Another issue with these kinds of local convergence bounds, which assume the objective is quadratic (or is well-approximated as such), is that they are always improved by using the Hessian instead of the GGN, and thus fail to explain the empirically observed superiority of the GGN over the Hessian for neural network training. This is because they assume that the objective function has constant curvature (given by the Hessian at the optimum), so that optimization could not possibly be helped by using an alternative curvature matrix like the GGN. % (e.g. for the reasons discussed in Section \ref{sec:GGN_speculation}), even assuming it is distinct from the Hessian for such objectives (which it may not be). 
 And for this reason they also fail to explain why damping methods are so crucial in practice.

Moreover, even just interpreting the constants in these bounds can be hard when using the Fisher or GGN. For example, if we pay attention only to the noise-independent/starting point-dependent term in the bound from Theorem \ref{thm:averaging}, which is given by
\begin{align*}
    \frac{1}{2(k+1)^2\alpha^2} \left\| {H^*}^{-1/2} B (\theta_0 - \theta^*) \right\|^2 \ep{,}
\end{align*}
and plug in $\B = F$ and the maximum-allowable learning rate $\alpha = 1/\lambda_1(B^{-1} {H^*})$, we get the somewhat opaque expression
\begin{align*}
    \frac{\lambda_1(F^{-1} {H^*})^2}{2(k+1)^2} \left\| {H^*}^{-1/2} F (\theta_0 - \theta^*) \right\|^2 \ep{.}
\end{align*}
It's not immediately obvious how we can further bound this expression in the non-realizable case (i.e. where we don't necessarily have $F = H^*$) using easily accessible/interpretable properties of the objective function. This is due to the complicated nature of the relationship between the GGN and Hessian, which we haven't explored in this report beyond the speculative discussion in Section \ref{sec:GGN_speculation}.

Finally, we leave the reader with a few open questions.
\begin{itemize}
    \item Can the observed advantages of the GGN vs the Hessian be rigorously justified for neural networks (assuming proper use of damping/trust-regions in both cases to deal with issues like negative curvature)?
    \item Are there situations where the Fisher and GGN are distinct and one is clearly preferable over the other?
    \item When will be pre-asymptotic advantage of stochastic 2nd-order methods vs SGD with Polyak averaging matter in practice? And can this be characterized rigorously using accessible properties of the target objective?
\end{itemize}
% (although some new work by \citet{zhang2019fast} and \citet{cai2019gram} sheds some light on this)

\section*{Acknowledgments}
We gratefully acknowledge support from Google, DeepMind, and the University of Toronto.  We would like to thank L\'{e}on Bottou, Guillaume Desjardins, Alex Botev, and especially the very thorough anonymous JMLR reviewers for their useful feedback on earlier versions of this manuscript.

\newpage

\renewcommand{\theHsection}{A\arabic{section}}
\appendix

\section{Proof of Basic Identity from \citet{Murata}}
\label{app:basic_ident_proof}

\begin{proposition}
Given the definitions of Section \ref{sec:asymptotic_speed} we have
\begin{align*}
\ex\left[ (\theta_k - \theta^*)(\theta_k - \theta^*)^\top \right] &= V_k + (\ex[\theta_k] - \theta^*)(\ex[\theta_k] - \theta^*)^\top \ep{.}
\end{align*}
\end{proposition}

\begin{proof}
We have
\begin{align*}
\ex\left[ (\theta_k - \theta^*)(\theta_k - \theta^*)^\top \right] &= \ex\left[ (\theta_k - \ex[\theta_k] + \ex[\theta_k] - \theta^*)(\theta_k - \ex[\theta_k] + \ex[\theta_k] - \theta^*)^\top \right] \\
&= \ex \left[ (\theta_k - \ex[\theta_k])(\theta_k - \ex[\theta_k])^\top \right] + \ex \left[ (\ex[\theta_k] - \theta^*)(\theta_k - \ex[\theta_k])^\top \right] \\
&\quad \quad + \ex \left[ (\theta_k - \ex[\theta_k])(\ex[\theta_k] - \theta^*)^\top \right] + \ex \left[ (\ex[\theta_k] - \theta^*)(\ex[\theta_k] - \theta^*)^\top \right] \\
&= V_k + (\ex[\theta_k] - \theta^*) \ex \left[ (\theta_k - \ex[\theta_k]) \right]^\top \\
&\quad \quad + \ex \left[ (\theta_k - \ex[\theta_k]) \right](\ex[\theta_k] - \theta^*)^\top + (\ex[\theta_k] - \theta^*)(\ex[\theta_k] - \theta^*)^\top \\
&= V_k + (\ex[\theta_k] - \theta^*)(\ex[\theta_k] - \theta^*)^\top \ep{,}
\end{align*}
where we have used $\ex \left[ (\theta_k - \ex[\theta_k]) \right] = \ex[\theta_k] - \ex[\theta_k] = 0$.
\end{proof}

\section{Proofs of Convergence Theorems}
\label{app:convergence_proofs}

In this section we will prove Theorems \ref{thm:murata_replace}, \ref{thm:asymptotic_main}, and \ref{thm:averaging}.

To begin, we recall the following linear operator definitions from the beginning of Section \ref{sec:speed_analysis}:
\begin{align*}
\Xi (X) &= \B^{-1} H^* X + \left(\B^{-1} H^* X\right)^\top = \B^{-1} H^* X + X H^* \B^{-1} \ep{,} \\
\Psi_\beta(X) &= \left(I - \beta \B^{-1} H^*\right) X \left(I - \beta \B^{-1} H^*\right)^\top = I - \beta \Xi (X) + \beta^2 \Omega (X) \ep{,}
\end{align*}
to which we add
\begin{align*}
\Omega (X) &= \B^{-1} H^* X \left(\B^{-1} H^*\right)^\top = \B^{-1} H^* X H^* \B^{-1} \ep{.}
\end{align*}
We note that
\begin{align*}
\Xi (\Omega (X)) &= \B^{-1} H^* \left(\B^{-1} H^* X H^* \B^{-1}\right) + \left(\B^{-1} H^* X H^* \B^{-1}\right) H^* \B^{-1} \\
&= \B^{-1} H^* \left(\B^{-1} H^* X  + X H^* \B^{-1}\right) H^* \B^{-1} \\
&= \Omega (\Xi (X))
\end{align*}
and so $\Omega$ and $\Xi$ commute as operators. And because the remaining operators defined above are all linear combinations of these, it follows that they all commute with each other.

According to eqn.~\ref{eqn:general_iter} we have
\begin{eqnarray}
\label{eqn:thetak_exrp}
  \theta_{k + 1} - \theta^{\ast} & = & \theta_k - \alpha_k B^{- 1} g_k
  (\theta_k) - \theta^{*} \notag \\
  & = & \theta_k - \theta^{*} - \alpha_k B^{- 1} (\nabla h(\theta_{k}) + \epsilon_k(\theta_k)) \notag \\
  & = & \theta_k - \theta^{\ast} - \alpha_k B^{- 1} H^{\ast} (\theta_k -
  \theta^{*}) - \alpha_k B^{- 1} \epsilon_k(\theta_k) \notag \\
  & = & (I - \alpha_k B^{- 1} H^{\ast}) (\theta_k - \theta^{*}) - \alpha_k
  B^{- 1} \epsilon_k(\theta_k) \ep{,}
\end{eqnarray}
where we have defined $\epsilon_k(\theta_k) = g_k (\theta_k) - \nabla h(\theta_{k})$ and used $\nabla h(\theta_{k}) = {H^*} (\theta_{k} - \theta^*)$. Taking the expectation of both sides while using $\ex \left[ \epsilon_k(\theta_k) \right] = 0$ we get
\begin{eqnarray*}
\ex[\theta_{k + 1}] - \theta^{\ast} = (I - \alpha_k B^{- 1} H^{\ast}) (\ex[\theta_k] - \theta^{*}) \ep{.}
\end{eqnarray*}
Iterating this we get
\begin{align}
\ex[\theta_k] - \theta^* = \prod_{j=0}^{k-1} \left(I - \alpha_j B^{-1} {H^*} \right) (\theta_0 - \theta^*) \ep{.} \label{eqn:exthetak_exrp}
\end{align}

Then observing that $\var(\theta_k - \theta^*) = \var(\theta_k) = V_k$ for all $k$, and exploiting the uncorrelatedness of $\theta_k$ and $\epsilon_k(\theta_k)$, we can take the variance of both sides of eqn.~\ref{eqn:thetak_exrp} to get that $V_k$ evolves according to
\begin{align*}
V_{k+1} &= \var((I - \alpha_k \B^{-1} H^*)(\theta_k - \theta^*)) + \var(\alpha_k \B^{-1} \epsilon_k(\theta_k)) \\
&= (I - \alpha_k \B^{-1} H^*) V_k (I - \alpha_k \B^{-1} H^*)^\top + \alpha_k^2 \B^{-1} \Sigma_g \B^{-1} \\
&= \Lambda_k (V_k) + \alpha_k^2 U \ep{,}
\end{align*}
where we have defined
\begin{align*}
U &= \B^{-1} \Sigma_g \B^{-1} \ep{, and} \\
\Lambda_k &= \Psi_{\alpha_k} \ep{.}
\end{align*}
This recursion for $V_k$ will be central in our analysis.

\subsection{The Constant $\alpha_k = \alpha$ Case}

In this subsection we will prove the claims made in Theorems \ref{thm:murata_replace} and \ref{thm:asymptotic_main} pertaining to the case where:
\begin{itemize}
    \item $\alpha_k = \alpha$ for a constant $\alpha$, and
    \item $\alpha \lambda_1(\B^{-1} H^*) \leq 1$.
\end{itemize}

To begin, we define the notation $\Lambda = \Psi_{\alpha}$.

\begin{corollary}
\label{cor:I_minus_Lambda_properties}
The operator $I - \Lambda$ has all positive eigenvalues and is thus invertible.  Moreover, we have 
\begin{align*}
(I - \Lambda)^{-1} = \sum_{i = 0}^{\infty} \Lambda^i \ep{.}
\end{align*}
And so if $X$ is a PSD matrix then $\left(I - \Lambda\right)^{-1}(X)$ is as well.
\end{corollary}

\begin{proof}
Let $D = I - \alpha \B^{-1} H^*$, so that $\Lambda(X) = D X D^\top$.

The eigenvalues of $\B^{-1} H^*$ are the same as those of $\B^{-1/2} H^* \B^{-1/2}$ (since the matrices differ by a similarity transform), and are thus real-valued and positive. Moreover, the minimum eigenvalue of $D$ is $\lambda_n(D) = 1 - \alpha \lambda_1(\B^{-1} H^*) \geq 0$, and the maximum eigenvalue is $\lambda_1(D) = 1 - \alpha \lambda_n(\B^{-1} H^*) < 1$, where we have used $\lambda_n(\B^{-1} H^*) = \lambda_n(\B^{-1/2} H^* \B^{-1/2}) > 0$ (both $H^*$ and $\B$ are positive definite).

The claim then follows by an application of Lemma \ref{lem:op_props}.
\end{proof}

\begin{lemma}
\label{lem:variance_formula}
The variance matrix $V_k$ is given by the formula
\begin{align*}
V_k = \left(I - \Lambda^k\right)(V_\infty) \ep{,}
\end{align*}
where $V_\infty = \alpha^2(I-\Lambda)^{-1}(U)$.
\end{lemma}
\begin{proof}
By expanding the recursion for $V_k$ we have
\[ V_k = \alpha^2  \sum_{i = 0}^{k - 1} \Lambda^i (U) . \]
And Corollary \ref{cor:I_minus_Lambda_properties} tells us that
\[ (I - \Lambda)^{- 1} = \sum_{i = 0}^{\infty} \Lambda^i . \]
Thus, we can write:
\begin{eqnarray*}
  \sum_{i = 0}^{k-1} \Lambda^i & = & \sum_{i = 0}^{\infty} \Lambda^i - \sum_{i =
  k}^{\infty} \Lambda^i\\
  & = & \sum_{i = 0}^{\infty} \Lambda^i - \Lambda^k  \sum_{i = 0}^{\infty} \Lambda^i\\
  & = & (I - \Lambda^k)  (I - \Lambda)^{- 1} .
\end{eqnarray*}
We thus conclude that
\[ V_k = \alpha^2  \sum_{i = 0}^{k-1} \Lambda^i (U) = \alpha^2  (I - \Lambda^k) 
   (I - \Lambda)^{- 1} (U) = (I - \Lambda^k)  (V_{\infty}), \]
as claimed.
\end{proof}

\begin{proposition}
\label{prop:trace_formula_2}
For any appropriately sized matrix $X$ we have
\begin{align*}
\tr\left(H^* \left(I - \Lambda\right)^{-1}(X)\right) = \frac{1}{2\alpha} \tr\left( \left(\B - \frac{\alpha}{2} {H^*}\right)^{-1} \B X \B \right) \ep{.}
\end{align*}

\end{proposition}

\begin{proof}

Let $Y = \left(I - \Lambda\right)^{-1}(X)$, so that $\left(I - \Lambda\right)(Y) = X$. Written as a matrix equation this is
\begin{align*}
\alpha \B^{-1} {H^*} Y + \alpha Y {H^*} \B^{-1} - \alpha^2 \B^{-1} {H^*} Y {H^*} \B^{-1} = X \ep{.}
\end{align*}

Left and right multiplying both sides by $\B$ we have
\begin{align*}
\alpha {H^*} Y \B + \alpha \B Y {H^*} - \alpha^2 {H^*} Y {H^*} = \B X \B \ep{.}
\end{align*}
This can be written as $A^\top P + PA + Q = 0$, where
\begin{align*}
A &= Y {H^*} \\
P &= \alpha\left(\B - \frac{\alpha}{2} {H^*}\right) \\
Q &= -\B X \B \ep{.}
\end{align*}

In order to compute $\tr({H^*} Y)$ we can thus apply Lemma \ref{lemma:pseudo-CALE}. However, we must first verify that our $P$ is invertible. To this end we will show that $\B - \frac{\alpha}{2} {H^*}$ is positive definite.  This is equivalent to the condition that $B^{-1/2}(\B - \frac{\alpha}{2} {H^*})B^{-1/2} = I - \frac{\alpha}{2} B^{-1/2} {H^*} B^{-1/2}$ is positive definite, or in other words that $\lambda_1(B^{-1/2} {H^*} B^{-1/2}) = \lambda_1(B^{-1} {H^*}) < 2$.  This is true by hypothesis (indeed, we are assuming $\alpha \lambda_1(\B^{-1} {H^*}) < 1$).

By Lemma \ref{lemma:pseudo-CALE} it thus follows that
\begin{align*}
\tr({H^*} Y) = \tr(Y {H^*}) = \tr(A) = -\frac{1}{2}\tr(P^{-1}Q) = \frac{1}{2\alpha} \tr\left( \left(\B - \frac{\alpha}{2} {H^*}\right)^{-1} \B X \B \right) \ep{.}
\end{align*}

\end{proof}

\begin{lemma}
The expected objective value satisfies
\begin{align*}
l(k) \leq  \ex[h(\theta_k)] - h(\theta^*) \leq u(k) \ep{,}
\end{align*}
where
\begin{align*}
u(k) &= \left[ 1 - \left(1 - \epsilon_1 \right)^{2k} \right] \frac{\alpha}{4} \tr\left( \left(\B - \frac{\alpha}{2} {H^*}\right)^{-1} \Sigma_g \right) + \left(1 - \epsilon_2 \right)^{2k} h(\theta_0)
\end{align*}
and
\begin{align*}
l(k) &= \left[ 1 - \left(1 - \epsilon_2 \right)^{2k} \right] \frac{\alpha}{4} \tr\left( \left(\B - \frac{\alpha}{2} {H^*}\right)^{-1} \Sigma_g \right) + \left(1 - \epsilon_1 \right)^{2k} h(\theta_0) \ep{,}
\end{align*}
with $\epsilon_1 = \lambda_1(C) = \alpha \lambda_1\left(B^{-1} {H^*}\right)$ and $\epsilon_2 = \lambda_n(C) = \alpha \lambda_n\left(B^{-1} {H^*}\right)$.

\end{lemma}

\begin{proof}
Note that for any appropriately sized matrix $X$ we have
\begin{align*}
{H^*}^{1/2} &\Lambda (X) \: {H^*}^{1/2} = {H^*}^{1/2} \left( (I - \alpha B^{-1} {H^*}) X (I - \alpha {H^*} B^{-1}) \right) \: {H^*}^{1/2} \\
&= \left( (I - \alpha {H^*}^{1/2} B^{-1} {H^*}^{1/2}) {H^*}^{1/2}X{H^*}^{1/2} (I - \alpha {H^*}^{1/2} B^{-1}{H^*}^{1/2}) \right) \\
&= \tilde{\Lambda}\left( {H^*}^{1/2}X{H^*}^{1/2} \right) \ep{,}
\end{align*}
where we have defined
\begin{align*}
\tilde{\Lambda}(Y) = (I - C) Y (I - C)^\top = (I - C) Y (I - C)
\end{align*}
with
\begin{align*}
C = \alpha {H^*}^{1/2} B^{-1} {H^*}^{1/2} \ep{.}
\end{align*}

Applying this repeatedly we obtain
\begin{align*}
{H^*}^{1/2} \Lambda^k (X) \: {H^*}^{1/2} = \tilde{\Lambda}^k ({H^*}^{1/2} X {H^*}^{1/2}) \ep{.}
\end{align*}

Then, using the expression for $V_k$ from Lemma \ref{lem:variance_formula}, it follows that
\begin{align}
{H^*}^{1/2} V_k {H^*}^{1/2} &= {H^*}^{1/2} \left( V_\infty - \Lambda^k (V_\infty) \right)  {H^*}^{1/2} \notag \\
&= {H^*}^{1/2} \left( I - \Lambda^k \right)(V_\infty)  {H^*}^{1/2} \notag \\
&= \left(I - \tilde{\Lambda}^k\right) ({H^*}^{1/2} V_\infty {H^*}^{1/2}) \ep{.} \label{eqn:HVkH}
\end{align}

And thus
\begin{align}
\frac{1}{2} \tr \left( {H^*} V_k \right) &= \frac{1}{2} \tr \left( {H^*}^{1/2} V_k {H^*}^{1/2} \right) \notag \\
&= \frac{1}{2} \tr\left( {H^*}^{1/2} V_\infty {H^*}^{1/2} \right ) - \frac{1}{2} \tr\left( \tilde{\Lambda}^k ({H^*}^{1/2} V_\infty {H^*}^{1/2}) \right) \ep{.} \label{eqn:obj_expression_fixeda_2}
\end{align}

Next, we observe that for any matrix $X$
\begin{align}
\label{eqn:XiC_ident}
\tr\left(\tilde{\Lambda}^k (X)\right) = \tr\left( (I - C)^k X (I - C)^k \right) = \tr\left( (I - C)^{2k} X \right) \ep{.}
\end{align}

Because the eigenvalues of a product of square matrices are invariant under cyclic permutation of those matrices, we have $\lambda_1(C) = \lambda_1( \alpha  {H^*}^{1/2} B^{-1} {H^*}^{1/2}) = \alpha \lambda_1( B^{-1} {H^*} ) \leq 1$ so that $I - C$ is PSD, and it thus follows that $\lambda_i((I - C)^{2k}) = (1 - \lambda_{n-i+1}(C))^{2k}$.  Then, assuming that $X$ is also PSD, we can use Lemma \ref{lemma:wang_basic_bound} to get
\begin{align*}
(1 - \lambda_1(C))^{2k} \tr(X) \leq \tr\left((I - C)^{2k} X\right) \leq (1 - \lambda_n(C))^{2k} \tr(X) \ep{.}
\end{align*}

Applying this to eqn.~\ref{eqn:obj_expression_fixeda_2} we thus have the upper bound
\begin{align}
\frac{1}{2} \tr \left( {H^*} V_k \right) \leq \left(1 - (1 - \lambda_1(C))^{2k}\right) \frac{1}{2} \tr({H^*} V_\infty) \ep{,} \label{eqn:trHVk_upper}
\end{align}
and the lower bound
\begin{align}
\frac{1}{2} \tr \left( {H^*} V_k \right) \geq \left(1 - (1 - \lambda_n(C))^{2k}\right) \frac{1}{2} \tr({H^*} V_\infty) \ep{.} \label{eqn:trHVk_lower}
\end{align}

% where we have used the following equality:
% \begin{align*}
% \frac{1}{2} \tr\left ( {H^*}^{1/2}(\theta_0 - \theta^*)(\theta_0 - \theta^*)^\top {H^*}^{1/2} \right) &= \tr\left ( (\theta_0 - \theta^*)^\top {H^*}(\theta_0 - \theta^*) \right) \\
% &= \frac{1}{2} (\theta_0 - \theta^*)^\top {H^*}(\theta_0 - \theta^*) = h(\theta_0) \ep{.}
% \end{align*}

From Lemma \ref{lem:variance_formula}, $V_\infty$ is given by $V_\infty = \alpha^2(I-\Lambda)^{-1} (B^{-1} \Sigma_g B^{-1})$. Thus, by Proposition \ref{prop:trace_formula_2} we have
\begin{align}
\label{eqn:trHVinf}
\tr({H^*}V_\infty ) = \alpha^2 \frac{1}{2\alpha} \tr\left( \left(\B - \frac{\alpha}{2} {H^*}\right)^{-1} \B \B^{-1} \Sigma_g \B^{-1} \B \right) \notag \\
= \frac{\alpha}{2} \tr\left( \left(\B - \frac{\alpha}{2} {H^*}\right)^{-1} \Sigma_g \right) \ep{.}
\end{align}

\vspace{0.3in}

Next, we will compute/bound the term $\tr\left( {H^*} (\ex[\theta_k] - \theta^*)(\ex[\theta_k] - \theta^*)^\top \right)$. 

From eqn.~\ref{eqn:exthetak_exrp} we have
\begin{align*}
\ex[\theta_k] - \theta^* &= (I - \alpha B^{-1} {H^*})^k (\theta_0 - \theta^*) \ep{.}
\end{align*}
Then observing
\begin{align*}
{H^*}^{1/2} (I - \alpha B^{-1} {H^*}) = \left(I - \alpha {H^*}^{1/2} B^{-1} {H^*}^{1/2}\right) {H^*}^{1/2} = \left(I - C\right) {H^*}^{1/2}
\end{align*}
and applying this recursively, it follows that
\begin{align}
{H^*}^{1/2} (\ex[\theta_k] - \theta^*) &= {H^*}^{1/2} (I - \alpha B^{-1} {H^*})^k (\theta_0 - \theta^*) \notag \\
&= \left(I - C\right)^k  {H^*}^{1/2} (\theta_0 - \theta^*) \ep{.} \label{eqn:errexp_1}
\end{align}
Thus we have
\begin{align*}
\frac{1}{2} \tr \left( {H^*} (\ex[\theta_k] - \theta^*)(\ex[\theta_k] - \theta^*)^\top \right) &= \frac{1}{2} \tr \left( {H^*}^{1/2} (\ex[\theta_k] - \theta^*)(\ex[\theta_k] - \theta^*)^\top {H^*}^{1/2} \right)\\
&= \frac{1}{2} \tr \left(  \left(I - C \right)^k {H^*}^{1/2} (\theta_0 - \theta^*)(\theta_0 - \theta^*)^\top {H^*}^{1/2} {\left(I - C \right)^k}^\top  \right) \\
&= \frac{1}{2} \tr \left(  \left(I - C \right)^{2k}  \left ({H^*}^{1/2} (\theta_0 - \theta^*)(\theta_0 - \theta^*)^\top {H^*}^{1/2} \right) \right) \ep{.}
\end{align*}

Applying Lemma \ref{lemma:wang_basic_bound} in a similar manner to before, and using the fact that
\begin{align*}
h(\theta) = \frac{1}{2} (\theta-\theta^*)^\top H^* (\theta-\theta^*) = \frac{1}{2} \tr \left( {H^*}^{1/2} (\theta_0 - \theta^*)(\theta_0 - \theta^*)^\top {H^*}^{1/2} \right) \ep{,}
\end{align*}
we have the upper bound
\begin{align}
\label{eqn:trHdiff2_upper}
\frac{1}{2} \tr \left( {H^*} (\ex[\theta_k] - \theta^*)(\ex[\theta_k] - \theta^*)^\top \right) \leq \left(1 - \lambda_n\left(C\right)\right)^{2k} h(\theta_0) \ep{,}
\end{align}
and the lower bound
\begin{align}
\label{eqn:trHdiff2_lower}
\frac{1}{2} \tr \left( {H^*} (\ex[\theta_k] - \theta^*)(\ex[\theta_k] - \theta^*)^\top \right) \geq \left(1 - \lambda_1\left(C \right)\right)^{2k} h(\theta_0) \ep{.}
\end{align}

Combining eqn.~\ref{eqn:obj_expression}, eqn.~\ref{eqn:trHVk_upper}, eqn.~\ref{eqn:trHVk_lower}, eqn.~\ref{eqn:trHVinf}, eqn.~\ref{eqn:trHdiff2_upper}, and eqn.~\ref{eqn:trHdiff2_lower} yields the claimed bound.

\end{proof}

\subsection{The $\alpha_k = 1/(k+a+1)$ Case}

In this subsection we will prove the claims made in Theorems \ref{thm:murata_replace} and \ref{thm:asymptotic_main} pertaining to the case where
\begin{itemize}
    \item $\alpha_k = 1/(k+a+1)$,
    \item $b = \lambda_n\left( B^{-1} {H^*} \right) > 1/2$, and
    \item $\lambda_1\left( B^{-1} {H^*} \right) \leq a+1$.
\end{itemize}

\begin{corollary}
\label{cor:Xi_minus_I_properties}
The operator $\Xi - I$ has all positive eigenvalues and is thus invertible.  Moreover, if $X$ is a PSD matrix then $\left(\Xi - I\right)^{-1}(X)$ is as well.
\end{corollary}

\begin{proof}
We note that the operator $\Omega(X) = \B^{-1} H^* X H^* \B^{-1}$ is represented by the matrix $\B^{-1} H^* \otimes \B^{-1} H^*$. Because Kronecker products respect eigenvalue decompositions, the eigenvalues of this matrix are given by $\{ \lambda_i(\B^{-1} H^*) \lambda_j(\B^{-1} H^*) \: | \: 0 \leq i, j\leq n\}$, and are thus all positive.  (Because both $H^*$ and $\B$ are positive definite, the eigenvalues of $\B^{-1} H^*$ are all positive.)  Thus, $\Omega^{-1}$ exists and has all positive eigenvalues. And observing that $(\Omega^{-1})(X) = {H^*}^{-1} \B X \B {H^*}^{-1}$, we see that $\Omega^{-1}$ preserves PSD matrices.

Define the operator $\Phi(X) = \Omega^{-1} \left( (\Xi - I)(X) \right)$.  Because $\Omega^{-1}$ has all positive eigenvalues and commutes with $\Xi$, $\Phi$ will have all positive eigenvalues if and only if $\Xi - I$ does.  Moreover, because $\Omega^{-1}$ is a bijection from the set of PSD matrices to itself, $\Phi^{-1}$ will map PSD matrices to PSD matrices if and only if $(\Xi - I)^{-1}$ does.

With these facts in hand it suffices to prove our various claims about the operator $\Phi$.

Note that
\begin{align*}
\Omega^{-1} \left( (\Xi - I)(X) \right) &= \Omega^{-1} \left( \B^{-1} H^* X + X H^* \B^{-1} - X \right) \\
&= X \B {H^*}^{-1}  +  {H^*}^{-1} \B X - {H^*}^{-1} \B X \B {H^*}^{-1} \\
&= X - (I - {H^*}^{-1} \B) X (I - {H^*}^{-1} \B)^\top \\
&= X - D X D^\top \ep{,}
\end{align*}
where we have defined $D = I - {H^*}^{-1} \B$.

The eigenvalues of ${H^*}^{-1} \B$ are the same as those of $\B^{1/2} {H^*}^{-1} \B^{1/2}$ (since the matrices differ by a similarity transform), and are thus real-valued and positive.  Moreover, they are the inverse of those of $\B^{-1} H^*$. So the minimum eigenvalue of $D$ is $\lambda_n(D) = 1 - \lambda_1({H^*}^{-1} \B) = 1 - 1/\lambda_n(\B^{-1} {H^*} ) \geq 1 - 1/0.5 = -1$.  And the maximum eigenvalue of $D$ is just $\lambda_n(D) = 1 - \lambda_n({H^*}^{-1} \B) \leq 1$.

Lemma \ref{lem:op_props} is therefore applicable to $\Phi$ and the claim follows.
\end{proof}

\begin{lemma}
\label{lemma:var_expression_decay}
For all $k \geq 0$
\begin{align*}
V_k = \frac{1}{k+a} \left( \Xi - I\right)^{-1}(U) + E_k \ep{,}
\end{align*}
where $E_k$ is a matrix-valued ``error" defined by the recursion
\begin{align*}
E_{k+1} = \Lambda_k (E_k) + \frac{1}{(k+a)(k+a+1)^2} Z \quad \quad E_0 = -\frac{1}{a} \left( \Xi - I\right)^{-1}(U) \ep{,}
\end{align*}
with $Z = \left(\Xi - I\right)^{-1} \Psi_1 (U)$.
\end{lemma}

\begin{proof}
We will proceed by induction on $k$.

Observe that $V_0 = \var(\theta_0) = 0$, and that this agrees with our claimed expression for $V_k$ when evaluated at $k=0$:
\begin{align*}
V_0 &= \frac{1}{0+a} \left(\Xi - I\right)^{-1}(U) + E_0 \\
&= \frac{1}{a} \left(\Xi - I\right)^{-1}(U) - \frac{1}{a} \left(\Xi - I\right)^{-1}(U) = 0 \ep{.}
\end{align*}

For the inductive case, suppose that $V_k = \frac{1}{k+a} \left( \Xi - I\right)^{-1}(U) + E_k$ for some $k$. 

Observe that for any $i > 0$
\begin{align*}
\frac{1}{i} - \frac{1}{i(i+1)} = \frac{(i+1) - 1}{i(i+1)} = \frac{i}{i(i+1)} = \frac{1}{i+1} \ep{,}
\end{align*}
from which it follows that
\begin{align*}
\frac{1}{k+a} - \frac{1}{(k+a)(k+a+1)} &= \frac{1}{k+a+1} \ep{,}
\end{align*}
and thus also
\begin{align*}
\frac{1}{(k+a)(k+a+1)} - \frac{1}{(k+a)(k+a+1)^2} &= \frac{1}{(k+a+1)^2} \ep{.}
\end{align*}

By definition, we have $\Psi_1 = I - \Xi + \Omega$, which implies
\begin{align*}
\left(\Xi - I\right)^{-1} \Psi_1 = \left(\Xi - I\right)^{-1} \Omega - I \ep{,}
\end{align*}
and so $Z = \left(\left(\Xi - I\right)^{-1} \Omega - I\right)(U)$.

Using the above expressions and the fact that the various linear operators commute we have
\begin{align*}
V_{k+1} &= \Lambda_k (V_k) + \alpha_k^2 U \\
&= \Lambda_k (V_k) + \frac{1}{(k+a+1)^2} U \\
&= \Lambda_k \left(\frac{1}{k+a} \left( \Xi - I\right)^{-1}(U) + E_k\right) + \frac{1}{(k+a+1)^2} U \\
&= \Lambda_k \left(\frac{1}{k+a} \left( \Xi - I\right)^{-1}(U) \right) + \left(\frac{1}{(k+a)(k+a+1)} - \frac{1}{(k+a)(k+a+1)^2}\right) (U) + \Lambda_k (E_k)\\
&= \left(\Xi - I\right)^{-1} \left(\frac{1}{k+a}\Lambda_k + \frac{1}{(k+a)(k+a+1)} \left( \Xi - I\right) - \frac{1}{(k+a)(k+a+1)^2}\Omega \right)(U) \\
&\quad \quad + \frac{1}{(k+a)(k+a+1)^2}\left(\left(\Xi - I\right)^{-1}\Omega - I\right)(U) + \Lambda_k (E_k) \\
&= \left(\Xi - I\right)^{-1} \left(\frac{1}{k+a}\Lambda_k + \frac{1}{(k+a)(k+a+1)} \left( \Xi - I\right) - \frac{1}{(k+a)(k+a+1)^2}\Omega \right)(U) + E_{k+1} \ep{.}
\end{align*}

Next, observing that
\begin{align*}
\frac{1}{k+a}\Lambda_k &+ \frac{1}{(k+a)(k+a+1)} \left( \Xi - I\right) - \frac{1}{(k+a)(k+a+1)^2}\Omega \\
&= \frac{1}{k+a}I - \frac{1}{(k+a)(k+a+1)} \Xi + \frac{1}{(k+a)(k+a+1)^2} \Omega \\
&\quad \quad + \frac{1}{(k+a)(k+a+1)} \left( \Xi - I\right) - \frac{1}{(k+a)(k+a+1)^2}\Omega \\
&= \left(\frac{1}{k+a} - \frac{1}{(k+a)(k+a+1)}\right) I \\
&= \frac{1}{k+a+1} I \ep{,}
\end{align*}
our previous expression for $V_{k+1}$ simplifies to
\begin{align*}
V_{k+1} &= \left( \Xi - I\right)^{-1}\left( \frac{1}{k+a+1} I \right)(U) + E_{k+1} \\
&= \frac{1}{k+a+1}\left( \Xi - I\right)^{-1}(U) + E_{k+1} \ep{.}
\end{align*}
\end{proof}

\begin{proposition}
\label{prop:trace_formula_1}
For any appropriately sized matrix $X$ we have
\begin{align*}
\tr\left(H^* \left(\Xi - I\right)^{-1}(X)\right) = \frac{1}{2}\tr\left(\left(\B^{-1} - \frac{1}{2} {H^*}^{-1} \right)^{-1} X\right) \ep{.}
\end{align*}
\end{proposition}

\begin{proof}

Let $Y = \left(\Xi - I\right)^{-1}(X)$, so that $\left(\Xi - I\right)(Y) = X$. Written as a matrix equation this is
\begin{align*}
\B^{-1} {H^*} Y + Y {H^*} \B^{-1} - Y = X \ep{.}
\end{align*}
And rearranging this becomes
\begin{align*}
Y {H^*} \left(\B^{-1} - \frac{1}{2}{H^*}^{-1} \right) + \left(\B^{-1} - \frac{1}{2} {H^*}^{-1} \right) {H^*} Y - X = 0 \ep{,}
\end{align*}
which is of the form $A^\top P + PA + Q = 0$ with
\begin{align*}
A &= {H^*} Y \\
P &= \left(\B^{-1} - \frac{1}{2} {H^*}^{-1} \right) \\
Q &= -X \ep{,} 
\end{align*}

In order to compute $\tr({H^*} Y)$ we can thus apply Lemma \ref{lemma:pseudo-CALE}. However, we must first verify that our $P$ is invertible. To this end we will show that $\B^{-1} - \frac{1}{2} {H^*}^{-1}$ is positive definite.  This is equivalent to the condition that ${H^*}^{1/2}(\B^{-1} - \frac{1}{2} {H^*}^{-1}){H^*}^{1/2} = {H^*}^{1/2}\B^{-1}{H^*}^{1/2} - \frac{1}{2}I$ is positive definite, or in other words that $\lambda_n({H^*}^{1/2}\B^{-1}{H^*}^{1/2}) = \lambda_n(\B^{-1} {H^*}) > 1/2$, which is true by hypothesis.

Thus Lemma \ref{lemma:pseudo-CALE} is applicable, and yields
\begin{align*}
\tr({H^*} Y) &= \tr(A) = -\frac{1}{2}\tr(P^{-1}Q) = \frac{1}{2}\tr\left(\left(\B^{-1} - \frac{1}{2} {H^*}^{-1} \right)^{-1} X\right) \ep{.}
\end{align*}

\end{proof}

\begin{lemma}
\label{lem:tr_H_E_k_bound}
For $E_k$ as defined in Lemma \ref{lemma:var_expression_decay} we have the following upper and lower bounds:
\begin{align*}
\tr(H^* E_k) \leq \frac{\nu(a) k}{2(k+a)^3} \tr\left(\left(\B^{-1} - \frac{1}{2} {H^*}^{-1} \right)^{-1} \Psi_1(U) \right)
\end{align*}
and
\begin{align*}
\tr(H^* E_k) \geq -\frac{1}{2a} \left(\frac{a+1}{k+a}\right)^{2b} \tr\left(\left(\B^{-1} - \frac{1}{2} {H^*}^{-1} \right)^{-1} U\right) \ep{,}
\end{align*}
where $\nu(a) = (a+2)^3/(a(a+1)^2)$.
\end{lemma}

\begin{proof}

Recall that $E_0 = -\frac{1}{a} \left( \Xi - I\right)^{-1}(U)$ and $Z = \left(\Xi - I\right)^{-1} \Psi_1 (U)$. 

Observing that $\Psi_1$ preserves PSD matrices, that $U$ is PSD, and that if a linear operator preserves PSD matrices it also preserves negative semi-definite (NSD) matrices (which follows directly from linearity), we can then apply Corollary \ref{cor:Xi_minus_I_properties} to get that $E_0$ is NSD and $Z$ is PSD.

Define $F_k$ by 
\begin{align*}
F_{k+1} = \Lambda_k (F_k) + \frac{1}{(k+a)(k+a+1)^2} Z \quad \quad F_0 = 0 \ep{,}
\end{align*}
and $D_k$ by
\begin{align*}
D_{k+1} = \Lambda_k (D_k) \quad \quad D_0 = E_0 \ep{.}
\end{align*}

Because $\Lambda_k$ preserves PSD (and NSD) matrices like $\Psi_1$ does, and PSDness is preserved under non-negative linear combinations, it's straightforward to show that  $D_k \preceq E_k \preceq F_k$ and $F_k \succeq 0 \succeq D_k$ for all $k$ (where $X \preceq Y$ means that $Y-X$ is PSD). It thus follows that $\tr(H^* D_k) \leq \tr(H^* E_k) \leq \tr(H^* F_k)$.

Note that for any appropriately sized matrix $X$ we have
\begin{align*}
{H^*}^{1/2} \Lambda_k (X) \: {H^*}^{1/2} &= {H^*}^{1/2} \left( (I - \alpha_k B^{-1} {H^*}) X (I - \alpha_k {H^*} B^{-1}) \right) \: {H^*}^{1/2} \notag \\
&= \left( (I - \alpha_k {H^*}^{1/2} B^{-1} {H^*}^{1/2}) {H^*}^{1/2}X{H^*}^{1/2} (I - \alpha_k {H^*}^{1/2} B^{-1}{H^*}^{1/2}) \right) \notag \\
&= \tilde{\Lambda}_k\left( {H^*}^{1/2}X{H^*}^{1/2} \right) \ep{,}
\end{align*}
where we have defined
\begin{align*}
\tilde{\Lambda}(Y) = (I - C_k) Y (I - C_k)^\top = (I - C_k) Y (I - C_k)
\end{align*}
with
\begin{align*}
C_k = \alpha_k {H^*}^{1/2} B^{-1} {H^*}^{1/2} \ep{.}
\end{align*}

It thus follows that
\begin{align*}
\tr(H^* \Lambda_k(X)) &= \tr\left({H^*}^{1/2} \Lambda_k(X) {H^*}^{1/2}\right) \\
&= \tr\left((I - C_k) {H^*}^{1/2} X {H^*}^{1/2}(I - C_k) \right) \\
&= \tr\left((I - C_k)^2 {H^*}^{1/2} X {H^*}^{1/2} \right) \ep{.}
\end{align*}

Because the eigenvalues of a product of square matrices are invariant under cyclic permutation of those matrices, we have $\lambda_1(C_k) = \lambda_1( \alpha_k {H^*}^{1/2} B^{-1} {H^*}^{1/2}) = \alpha_k \lambda_1( B^{-1} {H^*} ) \leq \frac{1}{a+1} \lambda_1( B^{-1} {H^*} ) \leq 1$ so that $I - C_k$ is PSD, and it thus follows that $\lambda_i((I - C_k)^{2}) = (1 - \lambda_{n-i+1}(C_k))^{2}$.  Then assuming $X$ is also PSD we can use Lemma \ref{lemma:wang_basic_bound} to get
\begin{align*}
\tr\left((I - C_k)^2 {H^*}^{1/2} X {H^*}^{1/2} \right) &\leq \lambda_1\left((I - C_k)^2\right) \tr({H^*}^{1/2} X {H^*}^{1/2}) \\
&= \left(1 - \frac{b}{k+a+1} \right)^2 \tr(H^* X) \ep{,}
\end{align*}
where we have defined $b = \lambda_n(B^{-1} H^*)$.

From this it follows that
\begin{align*}
\tr(H^* F_{k+1}) &= \tr\left(H^* \left(\Lambda_k (F_k) + \frac{1}{(k+a)(k+a+1)^2} Z \right)\right) \\
&= \tr(H^* \Lambda_k(F_k)) + \frac{1}{(k+a)(k+a+1)^2} \tr(H^* Z) \\
&\leq \left(1 - \frac{b}{k+a+1} \right)^2 \tr(H^* F_k) + \frac{1}{(k+a)(k+a+1)^2} \tr(H^* Z) \ep{.}
\end{align*}

Iterating this inequality and using the fact that $F_0 = 0$ we thus have
\begin{align*}
\tr(H^* F_k) &\leq \tr(H^* F_0) \prod_{j=0}^{k-1} \left(1 - \frac{b}{j+a+1} \right)^2 \\
&\quad \quad + \tr(H^* Z) \sum_{i=0}^{k-1} \frac{1}{(i+a)(i+a+1)^2} \prod_{j=i+1}^{k-1} \left(1 - \frac{b}{j+a+1} \right)^2 \\
&= \tr(H^* Z) \sum_{i=0}^{k-1} \frac{1}{(i+a)(i+a+1)^2} \prod_{j=i+1}^{k-1} \left(1 - \frac{b}{j+a+1} \right)^2 
\ep{.}
\end{align*}
Then applying Proposition \ref{prop:prod_ineq} we can further upper bound this by
\begin{align*}
\frac{\nu(a) k}{(k+a)^3} \tr(H^* Z) \ep{,}
\end{align*}
where $\nu(a) = (a+2)^3/(a(a+1)^2)$.

Recalling the definition $Z = \left(\Xi - I\right)^{-1} \Psi_1 (U)$ and applying Proposition \ref{prop:trace_formula_1}, we have
\begin{align*}
\tr\left(H^* Z\right) &= \tr\left(H^* \left(\Xi - I\right)^{-1} \Psi_1 (U)\right) = \frac{1}{2}\tr\left(\left(\B^{-1} - \frac{1}{2} {H^*}^{-1} \right)^{-1} \Psi_1(U) \right) \ep{.}
\end{align*}

And so in summary we have
\begin{align*}
\tr(H^* E_k) \leq \tr(H^* F_k) \leq \frac{\nu(a) k}{2(k+a)^3} \tr\left(\left(\B^{-1} - \frac{1}{2} {H^*}^{-1} \right)^{-1} \Psi_1(U) \right) \ep{.}
\end{align*}

It remains to establish the claimed lower bound. 

For NSD matrices $X$, Corollary \ref{cor:wang_basic_bound} tells us that 
\begin{align*}
\tr\left((I - C_k)^2 {H^*}^{1/2} X {H^*}^{1/2} \right) &\geq \lambda_1\left((I - C_k)^2\right) \tr({H^*}^{1/2} X {H^*}^{1/2}) \\
&= \left(1 - \frac{b}{k+a+1} \right)^2 \tr(H^* X) \ep{.}
\end{align*}
From this it follows that
\begin{align*}
\tr(H^* D_{k+1}) = \tr\left(H^* \left(\Lambda_k (D_k)  \right)\right) \geq \left(1 - \frac{b}{k+a+1} \right)^2 \tr(H^* D_k) \ep{.}
\end{align*}
Iterating this inequality and using the fact that $D_0 = E_0$ we thus have
\begin{align*}
\tr(H^* D_k) &\geq \tr(H^* D_0) \prod_{j=0}^{k-1} \left(1 - \frac{b}{j+a+1} \right)^2 \\
&= \tr(H^* E_0) \prod_{j=0}^{k-1} \left(1 - \frac{b}{j+a+1} \right)^2
\ep{.}
\end{align*}
Then applying Proposition \ref{prop:prod_ineq} we can further lower bound this by
\begin{align*}
\left(\frac{a+1}{k+a}\right)^{2b} \tr(H^* E_0) \ep{.}
\end{align*}

By the definition of $E_0$ and Proposition \ref{prop:trace_formula_1} we have
\begin{align*}
\tr(H^* E_0) = -\frac{1}{a} \tr\left(H^* \left(\Xi - I\right)^{-1}(U)\right) = - \frac{1}{2a}\tr\left(\left(\B^{-1} - \frac{1}{2} {H^*}^{-1} \right)^{-1} U\right) \ep{,}
\end{align*}
and so in summary our lower bound is
\begin{align*}
\tr(H^* E_k) \geq \tr(H^* D_k) \geq -\frac{1}{2a} \left(\frac{a+1}{k+a}\right)^{2b} \tr\left(\left(\B^{-1} - \frac{1}{2} {H^*}^{-1} \right)^{-1} U\right) \ep{.}
\end{align*}

\end{proof}

\begin{lemma}
The expected objective value satisfies
\begin{align*}
l(k) \leq  \ex[h(\theta_k)] - h(\theta^*) \leq u(k) \ep{,}
\end{align*}
where
\begin{align*}
u(k) = \frac{1}{4(k+a)} \tr\left(\left(\B^{-1} - \frac{1}{2} {H^*}^{-1} \right)^{-1} U\right) &+ \frac{\nu(a) k}{4(k+a)^3} \tr\left(\left(\B^{-1} - \frac{1}{2} {H^*}^{-1} \right)^{-1} \Psi_1(U) \right) \\
&\quad + h(\theta_0) \left(\frac{a+1}{k+a}\right)^{2 b}
\end{align*}
and
\begin{align*}
l(k) = \frac{1}{4(k+a)} \tr\left(\left(\B^{-1} - \frac{1}{2} {H^*}^{-1} \right)^{-1} U\right) - \frac{1}{4a} \left(\frac{a+1}{k+a}\right)^{2b} \tr\left(\left(\B^{-1} - \frac{1}{2} {H^*}^{-1} \right)^{-1} U\right) \ep{,}
\end{align*}
where $\nu(a) = (a+2)^3/(a(a+1)^2)$.
\end{lemma}

\begin{proof}

From eqn.~\ref{eqn:obj_expression} we have
\begin{align}
\label{eqn:obj_expression_2}
\ex&[h(\theta_k)] - h(\theta^*) = \frac{1}{2} \tr \left( {H^*} V_k \right) + \frac{1}{2} \tr\left( {H^*} (\ex[\theta_k] - \theta^*)(\ex[\theta_k] - \theta^*)^\top \right) \ep{.} %+ o\left(\frac{1}{k}\right) 
\end{align}

By the expression for $V_k$ from Lemma \ref{lemma:var_expression_decay} we have that
\begin{align}
\frac{1}{2} \tr \left( {H^*} V_k \right) &= \frac{1}{2} \tr \left( {H^*} \left( \frac{1}{k+a} \left( \Xi - I\right)^{-1}(U) + E_k \right) \right)  \notag \\
&= \label{eqn:trHVk_2}
\frac{1}{2(k+a)} \tr \left( {H^*} \left( \left( \Xi - I\right)^{-1}(U) \right)\right) + \frac{1}{2}\tr\left(H^* E_k \right) \ep{.}
\end{align}

Applying Proposition \ref{prop:trace_formula_1} we have that the first term above is given by
\begin{align}
\label{eqn:term1}
\frac{1}{2(k+a)} \tr \left( {H^*} \left( \left( \Xi - I\right)^{-1}(U) \right)\right) = \frac{1}{4(k+a)} \tr\left(\left(\B^{-1} - \frac{1}{2} {H^*}^{-1} \right)^{-1} U\right) \ep{.} 
\end{align}

And by Lemma \ref{lem:tr_H_E_k_bound}, the second term is lower and upper bounded as follows:
\begin{align}
\label{eqn:tr_H_E_k_bound}
-\frac{1}{4a} \left(\frac{a+1}{k+a}\right)^{2b} \tr\left(\left(\B^{-1} - \frac{1}{2} {H^*}^{-1} \right)^{-1} U\right) &\leq \frac{1}{2} \tr(H^* E_k) \notag \\
&\quad \leq \frac{\nu(a) k}{4(k+a)^3} \tr\left(\left(\B^{-1} - \frac{1}{2} {H^*}^{-1} \right)^{-1} \Psi_1(U) \right) \ep{,}
\end{align}
where $\nu(a) = (a+2)^3/(a(a+1)^2)$.

It remains to compute/bound the term $\tr\left( {H^*} (\ex[\theta_k] - \theta^*)(\ex[\theta_k] - \theta^*)^\top \right)$.  From eqn.~\ref{eqn:exthetak_exrp} we have
\begin{align*}
\ex[\theta_k] - \theta^* = \prod_{j=0}^{k-1} \left(I - \alpha_j B^{-1} {H^*} \right) (\theta_0 - \theta^*) \ep{.}
\end{align*}

Observing
\begin{align*}
{H^*}^{1/2} (I - \alpha_i B^{-1} {H^*}) = (I - \alpha_i {H^*}^{1/2} B^{-1} {H^*}^{1/2}) {H^*}^{1/2}
\end{align*}
it follows that
\begin{align*}
{H^*}^{1/2} (\ex[\theta_k] - \theta^*) &= {H^*}^{1/2} \prod_{j=0}^{k-1} \left(I - \alpha_j B^{-1} {H^*} \right) (\theta_0 - \theta^*) \\
&= \prod_{j=0}^{k-1} \left(I - \alpha_j {H^*}^{1/2} B^{-1} {H^*}^{1/2} \right) {H^*}^{1/2} (\theta_0 - \theta^*) \\
&= \psi_k\left( {H^*}^{1/2} B^{-1} {H^*}^{1/2} \right) {H^*}^{1/2} (\theta_0 - \theta^*) \ep{,}
\end{align*}
where $\psi_k$ is a polynomial defined by
\begin{align*}
\psi_k(x) = \prod_{j=0}^{k-1} \left( 1 - \alpha_j x \right) = \prod_{j=0}^{k-1} \left( 1 - \frac{x}{j+a+1} \right)  \ep{.}
\end{align*}

In the domain $0 \leq x \leq a+1$ we have that $\psi_k(x)$ is a decreasing function of $x$. Moreover, for $x$'s in this domain Proposition \ref{prop:prod_ineq} says that
\begin{align*}
\psi_k(x) \leq \left(\frac{a+1}{k+a}\right)^{x} \ep{.}
\end{align*}
So because the eigenvalues of $\psi_k(X)$ for any matrix $X$ are given by $\{\psi_k(\lambda_i(X))\}_i$, and $\lambda_1(B^{-1}{H^*}) \leq a+1$ by hypothesis, it thus follows that 
\begin{align*}
\lambda_1\left( \psi_k\left( {H^*}^{1/2} B^{-1} {H^*}^{1/2} \right) \right) = \psi_k\left( \lambda_n\left( {H^*}^{1/2} B^{-1} {H^*}^{1/2} \right) \right) = \psi_k\left( \lambda_n\left( B^{-1} {H^*} \right) \right) \leq  \left(\frac{a+1}{k+a}\right)^b \ep{.}
\end{align*}

And so by Lemma \ref{lemma:wang_basic_bound} we have that
\begin{align}
\frac{1}{2} \tr & \left( {H^*} (\ex[\theta_k] - \theta^*)(\ex[\theta_k] - \theta^*)^\top \right) = \frac{1}{2} \tr \left( {H^*}^{1/2} (\ex[\theta_k] - \theta^*)(\ex[\theta_k] - \theta^*)^\top {H^*}^{1/2} \right) \notag \\
&= \frac{1}{2} \tr \left(  \psi_k\left( {H^*}^{1/2} B^{-1} {H^*}^{1/2} \right)^2 \left ({H^*}^{1/2} (\theta_0 - \theta^*)(\theta_0 - \theta^*)^\top {H^*}^{1/2} \right) \right) \notag  \\
&\leq \lambda_1\left( \psi_k\left( {H^*}^{1/2} B^{-1} {H^*}^{1/2} \right) \right)^2 \frac{1}{2} \tr \left ({H^*}^{1/2} (\theta_0 - \theta^*)(\theta_0 - \theta^*)^\top {H^*}^{1/2} \right) \label{eqn:crude_bound_step} \\
&= \lambda_1\left( \psi_k\left( {H^*}^{1/2} B^{-1} {H^*}^{1/2} \right) \right)^2 \frac{1}{2} (\theta_0 - \theta^*)^\top{H^*}(\theta_0 - \theta^*) \notag \\
&\leq \left(\frac{a+1}{k+a}\right)^{2 b} h(\theta_0) \notag \ep{.}
\end{align}

Combining eqn.~\ref{eqn:obj_expression_2}, eqn.~\ref{eqn:trHVk_2}, eqn.~\ref{eqn:term1}, eqn.~\ref{eqn:tr_H_E_k_bound} and the above bound the claimed result follows.

\end{proof}

\subsection{Proof of Theorem \ref{thm:averaging}}
\label{app:thm_avg_proof}

\begin{theorem*}
Suppose that $\theta_k$ is generated by the stochastic iteration in eqn.~\ref{eqn:general_iter} with constant step-size $\alpha_k = \alpha$ while optimizing a quadratic objective $h(\theta) = \frac{1}{2} (\theta - \theta^*)^\top {H^*} (\theta - \theta^*)$. Further more, suppose that $\alpha \lambda_1(B^{-1} {H^*}) < 1$, and define $\bar{\theta}_k = \frac{1}{k+1} \sum_{i = 0}^k \theta_i$. Then we have the following bound:

\begin{align*}
\ex[h(\bar{\theta}_k)] - h(\theta^*) &\leq \min \left \{\frac{1}{2(k+1)} \tr\left({H^*}^{-1} \Sigma_g\right) , \:\: \frac{\alpha}{4} \tr\left( \left(\B - \frac{\alpha}{2} {H^*}\right)^{-1} \Sigma_g \right) \right\}  \\
&\quad + \min \left \{ \frac{1}{2(k+1)^2\alpha^2} \left\| {H^*}^{-1/2} B (\theta_0 - \theta^*) \right\|^2, \right. \\ &\quad\quad\quad\quad\quad\quad \left. \frac{1}{2(k+1) \alpha} \left\| B^{1/2} (\theta_0 - \theta^*) \right\|^2, \:\: h(\theta_0) \right \} \ep{.}
\end{align*}
\end{theorem*}

\begin{proof}

To begin, we observe that, analogously to eqn.~\ref{eqn:obj_expression},
\begin{align}
\label{eqn:obj_expression_averaging}
\ex[h(\bar{\theta}_k)] - h(\theta^*) = \frac{1}{2} \tr \left( {H^*} \bar{V}_k \right) + \frac{1}{2} \tr \left( {H^*} \left(\ex[\bar{\theta}_k] - \theta^*\right)\left(\ex[\bar{\theta}_k] - \theta^*\right)^\top \right) \ep{,}
\end{align}
where
\begin{align*}
\bar{V}_k = \var(\bar{\theta}_k) = \cov(\bar{\theta}_k, \bar{\theta}_k) = \ex \left[ \left(\bar{\theta}_k - \ex[\bar{\theta}_k]\right)\left(\bar{\theta}_k - \ex[\bar{\theta}_k]\right)^\top \right ] \ep{.}
\end{align*}

Our first major task is to find an expression for $\bar{V}_k$ in order to bound the term $\frac{1}{2} \tr \left( {H^*} \bar{V}_k \right)$.  To this end we observe that
\begin{align*}
\bar{V}_k = \frac{1}{(k+1)^2} \sum_{i = 0}^k \sum_{j = 0}^k \cov( \theta_i, \theta_j ) \ep{.}
\end{align*}
For $j > i$ we have
\begin{align*}
\cov \left( \theta_i, \theta_j \right) = \cov\left( \theta_i, \theta_{j-1} - \alpha \B^{-1} g_{j-1}(\theta_{j-1}) \right) = \cov\left( \theta_i, \theta_{j-1} \right) - \alpha \cov\left( \theta_i, g_{j-1}(\theta_{j-1}) \right) \B^{-1} \ep{,}
\end{align*}
where 
\begin{align*}
\cov\left( \theta_i, g_{j-1}(\theta_{j-1}) \right) &= \ex \left[ (\theta_i - \ex[\theta_i])(g_{j-1}(\theta_{j-1}) - \ex[g_{j-1}(\theta_{j-1})])^\top \right ] \\
&= \ex_{\theta_i,\theta_{j-1}} \left[ \ex_{g_{j-1}(\theta_{j-1}) | \theta_{j-1}} \left[ (\theta_i - \ex[\theta_i])(g_{j-1}(\theta_{j-1}) - \ex[g_{j-1}(\theta_{j-1})])^\top \right ] \right ] \\
&= \ex_{\theta_i,\theta_{j-1}} \left[ (\theta_i - \ex[\theta_i])( \nabla h(\theta_{j-1}) - \ex[g_{j-1}(\theta_{j-1})])^\top \right ] \\
&= \ex \left[ (\theta_i - \ex[\theta_i])( \nabla h(\theta_{j-1}) - \ex[g_{j-1}(\theta_{j-1})])^\top \right ] \ep{.}
\end{align*}
Here we have used the fact that $g_{j-1}(\theta_{j-1})$ is conditionally independent of $\theta_i$ given $\theta_{j-1}$ for $j-1 \geq i$ (which allows us to take the conditional expectation over $g_{j-1}(\theta_{j-1})$ inside), and is an unbiased estimator of $\nabla h(\theta_{j-1})$.  

Then, noting that $\ex[g_{j-1}(\theta_{j-1})] = \ex[\nabla h(\theta_{j-1})] = \ex[ {H^*} ( \theta_{j-1} - \theta^*)] = {H^*} ( \ex[\theta_{j-1}] - \theta^*)$, we have
\begin{align*}
\nabla h(\theta_{j-1}) - \ex[g_{j-1}(\theta_{j-1})] &= {H^*} (\theta_{j-1} - \theta^*) - {H^*} ( \ex[\theta_{j-1}] - \theta^*) \\
&= {H^*} (\theta_{j-1} - \ex[\theta_{j-1}] )
\end{align*}
so that
\begin{align*}
\cov\left( \theta_i, g_{j-1}(\theta_{j-1}) \right) &= \ex \left[ (\theta_i - \ex[\theta_i])( \nabla h(\theta_{j-1}) - \ex[g_{j-1}(\theta_{j-1})] )^\top \right ] \\
&= \ex \left[ (\theta_i - \ex[\theta_i])( {H^*} (\theta_{j-1} - \ex[\theta_{j-1}] ) )^\top \right ] \\
&= \ex \left[ (\theta_i - \ex[\theta_i])(\theta_{j-1} - \ex[\theta_{j-1}] )^\top \right ] {H^*} = \cov( \theta_i, \theta_{j-1} ) {H^*} \ep{.}
\end{align*}
From this we conclude that
\begin{align*}
\cov \left( \theta_i, \theta_j \right) &= \cov\left( \theta_i, \theta_{j-1} \right) - \alpha \cov\left( \theta_i, g_{j-1}(\theta_{j-1}) \right) \B^{-1} \\
&= \cov\left( \theta_i, \theta_{j-1} \right) - \alpha \cov( \theta_i, \theta_{j-1} ) {H^*} \B^{-1} \\
&= \cov\left( \theta_i, \theta_{j-1} \right) \left( I - \alpha \B^{-1} {H^*} \right)^\top \ep{.}
\end{align*}

Applying this recursively we have that for $j \geq i$
\begin{align}
\label{eqn:avg_decorr}
\cov \left( \theta_i, \theta_j \right) = V_i {\left( I - \alpha \B^{-1} {H^*} \right)^{j-i}}^\top \ep{.}
\end{align}
Taking transposes and switching the roles of $i$ and $j$ we similarly have for $i \geq j$ that
\begin{align*}
\cov \left( \theta_i, \theta_j \right) = \left( I - \alpha \B^{-1} {H^*} \right)^{i-j} V_j \ep{.}
\end{align*}

Thus, we have the following expression for the variance $\bar{V}_k$ of the averaged parameter $\bar{\theta}_k$:
\begin{align*}
\bar{V}_k &= \frac{1}{(k+1)^2} \sum_{i = 0}^k \sum_{j = 0}^k \cov( \theta_i, \theta_j ) \\
&= \frac{1}{(k+1)^2} \sum_{i = 0}^k \left( \sum_{j = 0}^{i} \left( I - \alpha \B^{-1} {H^*} \right)^{i-j} V_j +  \sum_{j = i+1}^k V_i {\left( I - \alpha \B^{-1} {H^*} \right)^{j-i}}^\top  \right) \ep{,}
\end{align*}
which by reordering the sums and re-indexing can be written as
\begin{align*}
\bar{V}_k = \frac{1}{(k+1)^2} \sum_{i = 0}^k \left( \sum_{j = 0}^{i} \left( I - \alpha \B^{-1} {H^*} \right)^j V_i  +  \sum_{j = 1}^{k-i} V_i {\left( I - \alpha \B^{-1} {H^*} \right)^j}^\top \right) \ep{.}
\end{align*}

Having computed $\bar{V}_k$ we now deal with the term $\frac{1}{2} \tr \left( {H^*} \bar{V}_k \right)$.  Observing that
\begin{align*}
{H^*}^{1/2} (I - \alpha B^{-1} {H^*}) = \left(I - \alpha {H^*}^{1/2} B^{-1} {H^*}^{1/2}\right) {H^*}^{1/2} = \left(I - C\right) {H^*}^{1/2} \ep{,}
\end{align*}
where $C = \alpha {H^*}^{1/2} B^{-1} {H^*}^{1/2}$, we have
\begin{align*}
{H^*}^{1/2} \bar{V}_k {H^*}^{1/2} = \frac{1}{(k+1)^2} \sum_{i = 0}^k \left(  \sum_{j = 0}^{i} \left( I - C \right)^j ({H^*}^{1/2} V_i {H^*}^{1/2}) + \sum_{j = 1}^{k-i} ({H^*}^{1/2} V_i {H^*}^{1/2}) \left( I - C \right)^j \right) \ep{.}
\end{align*}

It thus follows that
\begin{align*}
\frac{1}{2} \tr \left( {H^*} \bar{V}_k \right) &= \frac{1}{2} \tr \left( {H^*}^{1/2} \bar{V}_k {H^*}^{1/2} \right) \\
&= \frac{1}{2(k+1)^2} \sum_{i = 0}^k \tr \left( \left(\sum_{j = 0}^{i} \left(I - C\right)^j + \sum_{j = 1}^{k-i} \left(I - C\right)^j \right) {H^*}^{1/2} V_i {H^*}^{1/2} \right ) \ep{.}
\end{align*}

Recall that from eqn.~\ref{eqn:HVkH} we have
\begin{align*}
{H^*}^{1/2} V_i {H^*}^{1/2} = \left(I - \tilde{\Lambda}^i\right) ({H^*}^{1/2} V_\infty {H^*}^{1/2}) \ep{,}
\end{align*}
where $\tilde{\Lambda}(Y) = (I - C) Y (I - C)^\top = (I - C) Y (I - C)$.

Plugging this into the previous equation, using the definition of $\tilde{\Lambda}$, and the fact that various powers of $C$ commute, we have
\begin{align}
\frac{1}{2} \tr &\left( {H^*} \bar{V}_k \right) = \frac{1}{2(k+1)^2} \tr \left( \sum_{i = 0}^k \left(\sum_{j = 0}^{i} \left(I - C\right)^j + \sum_{j = 1}^{k-i} \left(I - C\right)^j \right) \left(I - \tilde{\Lambda}^i\right) \left({H^*}^{1/2} V_\infty {H^*}^{1/2}\right) \right ) \notag \\
&= \frac{1}{2(k+1)^2} \tr \left( \sum_{i = 0}^k \left(\sum_{j = 0}^{i} \left(I - C\right)^j + \sum_{j = 1}^{k-i} \left(I - C\right)^j \right) \left(I - (I-C)^{2i}\right) {H^*}^{1/2} V_\infty {H^*}^{1/2} \right ) \ep{.} \label{eqn:trHbarVk_finalexpression}
\end{align}

Because $C$ and $I-C$ are PSD (which follows from the hypothesis $\lambda_1(C) = \alpha \lambda_1( B^{-1} {H^*} ) < 1$), we have the following basic matrix inequalities:
\begin{align}
\sum_{j = 0}^{i} \left(I - C\right)^j + \sum_{j = 1}^{k-i} \left(I - C\right)^j &\preceq 2\sum_{j=0}^\infty \left(I - C \right)^j - I = 2 C^{-1} - I\label{eqn:matbound_1} \\
\sum_{j = 0}^{i} \left(I - C\right)^j + \sum_{j = 1}^{k-i} \left(I - C\right)^j &\preceq (k + 1)I \label{eqn:matbound_2} \\
\sum_{i = 0}^k \left(I - \left(I - C\right)^{2i}\right) &\preceq (k+1)I \ep{,} \label{eqn:matbound_5}
\end{align}
where $X \preceq Y$ means that $Y-X$ is PSD.

As the right and left side of all the previously stated matrix inequalities are commuting matrices (because they are all linear combinations of powers of $C$, and thus share their eigenvectors with $C$), we can apply Lemma \ref{lemma:PSDtrace_bound} to eqn.~\ref{eqn:trHbarVk_finalexpression} to obtain various simplifying upper bounds on $\frac{1}{2} \tr \left( {H^*} \bar{V}_k \right)$.

Applying Lemma \ref{lemma:PSDtrace_bound} using eqn.~\ref{eqn:matbound_1} and then eqn.~\ref{eqn:matbound_5} gives the upper bound
\begin{align*}
\frac{1}{2} \tr \left( {H^*} \bar{V}_k \right) &\leq \frac{1}{2(k+1)^2} \tr \left( \left(2C^{-1} - I\right) \: (k+1)I \: {H^*}^{1/2} V_\infty {H^*}^{1/2} \right ) \\
&= \frac{1}{k+1} \tr \left( \left(\frac{1}{\alpha}\B - \frac{1}{2}{H^*} \right) \: V_\infty \right ) \ep{,}
\end{align*}
where we have used ${H^*}^{1/2} C^{-1} {H^*}^{1/2} = {H^*}^{1/2} \left(\alpha {H^*}^{1/2} \B^{-1} {H^*}^{1/2} \right)^{-1} {H^*}^{1/2} = \frac{1}{\alpha}\B$.

Or we can apply the lemma using eqn.~\ref{eqn:matbound_2} and then  eqn.~\ref{eqn:matbound_5}, which gives a different upper bound of
\begin{align*}
\frac{1}{2} \tr \left( {H^*} \bar{V}_k \right) \leq \frac{1}{2(k+1)^2} \tr \left( (k+1)I \: (k+1)I \: {H^*}^{1/2} V_\infty {H^*}^{1/2} \right ) &\leq \frac{1}{2} \tr \left( {H^*} V_\infty \right ) \\
&= \frac{\alpha}{4} \tr\left( \left(\B - \frac{\alpha}{2} {H^*}\right)^{-1} \Sigma_g \right) \ep{,}
\end{align*}
where we have used eqn.~\ref{eqn:trHVinf} on the last line.

Applying these bounds to eqn.~\ref{eqn:trHbarVk_finalexpression} yields
\begin{align}
\label{eqn:trHbarVk_finalbound}
\frac{1}{2} \tr \left( {H^*} \bar{V}_k \right) &\leq \min \left \{\frac{1}{k+1} \tr \left( \left(\frac{1}{\alpha}\B - \frac{1}{2}{H^*} \right) \: V_\infty \right ) , \:\: \frac{\alpha}{4} \tr\left( \left(\B - \frac{\alpha}{2} {H^*}\right)^{-1} \Sigma_g \right) \right\} \ep{.}
\end{align}

%[[\textbf{tie together the above bounds in a big expression involving mins?}]]

To compute $\tr \left( \left(\frac{1}{\alpha}\B - \frac{1}{2}{H^*} \right) V_\infty \right )$, we begin by recalling the definition $V_\infty = \alpha^2 \left(I - \Lambda\right)^{-1} (U)$.  Applying the operator $\left(I - \Lambda\right)$ to both sides gives $\left(I - \Lambda\right) (V_\infty) = \alpha^2 (U)$, which corresponds to the matrix equation
\begin{align*}
\alpha \B^{-1} {H^*} V_\infty + \alpha V_\infty {H^*} \B^{-1} - \alpha^2 \B^{-1} {H^*} V_\infty {H^*} \B^{-1} = \alpha^2 U = \alpha^2 \B^{-1} \Sigma_g \B^{-1} \ep{.}
\end{align*}
Left and right multiplying both sides by $\frac{1}{\alpha} B$ gives
\begin{align*}
\frac{1}{\alpha} {H^*} V_\infty B + \frac{1}{\alpha} B V_\infty {H^*} - {H^*} V_\infty {H^*} = \Sigma_g \ep{,}
\end{align*}
which can be rewritten as
\begin{align*}
\left(\frac{1}{\alpha} B - \frac{1}{2} {H^*} \right) V_\infty {H^*} + {H^*} V_\infty \left(\frac{1}{\alpha} B - \frac{1}{2} {H^*} \right) = \Sigma_g \ep{.}
\end{align*}
This is of the form $A^\top P + PA + Q = 0$ where
\begin{align*}
A &= V_\infty \left(\frac{1}{\alpha} B - \frac{1}{2} {H^*}\right) \\
P &= {H^*} \\
Q &= -\Sigma_g \ep{.} 
\end{align*}

Applying Lemma \ref{lemma:pseudo-CALE} we get that
\begin{align}
\label{eqn:trBVinf}
\tr \left( \left(\frac{1}{\alpha}\B - \frac{1}{2}{H^*} \right) V_\infty \right ) = \tr(A^\top) = \tr(A) = \frac{1}{2}\tr(P^{-1}Q) = \frac{1}{2}\tr\left({H^*}^{-1} \Sigma_g\right) \ep{.}
\end{align}

\vspace{0.3in}

It remains to bound the term $\frac{1}{2} \tr \left( {H^*} (\ex[\bar{\theta}_k] - \theta^*)(\ex[\bar{\theta}_k] - \theta^*)^\top \right)$.  

First, we observe that by Theorem \ref{thm:murata_replace}
\begin{align*}
\ex[\bar{\theta}_k] - \theta^* = \frac{1}{k+1} \sum_{i = 0}^k (\ex[\theta_i] - \theta^*) = \frac{1}{k+1} \sum_{i = 0}^k \left (I - \alpha B^{-1} {H^*} \right)^i (\theta_0 - \theta^*) \ep{.}
\end{align*}
Applying eqn.~\ref{eqn:errexp_1} then gives
\begin{align*}
{H^*}^{1/2} \left(\ex[\bar{\theta}_k] - \theta^*\right) = \frac{1}{k+1} \sum_{i = 0}^k \left (I - C\right)^i {H^*}^{1/2} (\theta_0 - \theta^*) \ep{.}
\end{align*}
And thus we have
\begin{align}
\frac{1}{2} \tr &\left( {H^*} \left(\ex[\bar{\theta}_k] - \theta^*\right)\left(\ex[\bar{\theta}_k] - \theta^*\right)^\top \right) = \frac{1}{2} \tr \left( {H^*}^{1/2} \left(\ex[\bar{\theta}_k] - \theta^*\right)\left(\ex[\bar{\theta}_k] - \theta^*\right)^\top {H^*}^{1/2} \right) \notag \\
&= \frac{1}{2(k+1)^2} \tr \left( \left(\sum_{i=0}^k \left(I - C \right)^i\right) {H^*}^{1/2} (\theta_0 - \theta^*)(\theta_0 - \theta^*)^\top {H^*}^{1/2} \left(\sum_{i=0}^k \left(I - C \right)^i\right)  \right) \ep{.}\label{eqn:trHbardiff} %\\
%&= \frac{1}{2(k+1)^2} \tr \left( \left(\sum_{i=0}^k \left(I - C \right)^i\right)^2  {H^*}^{1/2} (\theta_0 - \theta^*)(\theta_0 - \theta^*)^\top {H^*}^{1/2} \right) \ep{.}
\end{align}

Similarly to eqn.~\ref{eqn:matbound_1}--\ref{eqn:matbound_5} we have the following matrix inequalities
\begin{align}
\sum_{i=0}^k \left(I - C \right)^i &\preceq \sum_{i=0}^\infty \left(I - C \right)^i = C^{-1} \label{eqn:matbound_6} \\
\sum_{i=0}^k \left(I - C\right)^i &\preceq (k+1)I \ep{.} \label{eqn:matbound_7}
\end{align}

Applying Lemma \ref{lemma:PSDtrace_bound} using eqn.~\ref{eqn:matbound_6} twice we obtain an upper bound on the RHS of eqn.~\ref{eqn:trHbardiff} of
\begin{align*}
\frac{1}{2(k+1)^2} \tr \left( C^{-1} \: {H^*}^{1/2} (\theta_0 - \theta^*)(\theta_0 - \theta^*)^\top {H^*}^{1/2} \: C^{-1} \right ) = \frac{1}{2(k+1)^2 \alpha^2} \left\| {H^*}^{-1/2} B (\theta_0 - \theta^*) \right\|^2 \ep{.}
\end{align*}
Applying the lemma using eqn.~\ref{eqn:matbound_6} and eqn.~\ref{eqn:matbound_7} gives a different upper bound of
\begin{align*}
\frac{1}{2(k+1)^2} &\tr \left( C^{-1} \: {H^*}^{1/2} (\theta_0 - \theta^*)(\theta_0 - \theta^*)^\top {H^*}^{1/2} \: (k+1)I \right ) &= \frac{1}{2(k+1) \alpha} \left\| B^{1/2} (\theta_0 - \theta^*) \right\|^2 \ep{.}
\end{align*}
And finally, applying the lemma using eqn.~\ref{eqn:matbound_7} twice gives an upper bound of
\begin{align*}
\frac{1}{2(k+1)^2} &\tr \left( (k+1) I \: {H^*}^{1/2} (\theta_0 - \theta^*)(\theta_0 - \theta^*)^\top {H^*}^{1/2} \: (k+1) I \right ) = h(\theta_0) \ep{.}
\end{align*}

Combining these various upper bounds gives us 
\begin{align}
\label{eqn:trHbardiff_finalbound}
\frac{1}{2} \tr & \left( {H^*} (\ex[\bar{\theta}_k] - \theta^*)(\ex[\bar{\theta}_k] - \theta^*)^\top \right) \notag \\
&\leq \min \left \{ \frac{1}{2(k+1)^2\alpha^2} \left\| {H^*}^{-1/2} B (\theta_0 - \theta^*) \right\|^2, \:\: \frac{1}{2(k+1) \alpha} \left\| B^{1/2} (\theta_0 - \theta^*) \right\|^2, \:\: h(\theta_0) \right \} \ep{.}
\end{align}

The result now follows from eqn.~\ref{eqn:obj_expression_averaging}, eqn.~\ref{eqn:trHbarVk_finalbound}, eqn.~\ref{eqn:trBVinf}, and eqn.~\ref{eqn:trHbardiff_finalbound}.

\end{proof}

\section{Derivations of Bounds for Section \ref{sec:consequences_asymptotic_main}}
\label{app:extra_deriv2}

By Lemma \ref{lemma:wang_basic_bound}
\begin{align*}
\tr\left( {H^*}^{-1} \Sigma_g\right) \geq \lambda_n \left({H^*}^{-1} \right) \tr(\Sigma_g) = \frac{1}{\lambda_1 \left({H^*} \right)}\tr(\Sigma_g)
\end{align*}
and
\begin{align*}
\tr\left( {H^*}^{-1} \Sigma_g\right) \leq \lambda_1 \left({H^*}^{-1} \right) \tr(\Sigma_g) = \frac{1}{\lambda_n \left({H^*} \right)}\tr(\Sigma_g) \ep{,}
\end{align*}
so that
\begin{align*}
\frac{1}{2\lambda_1 \left({H^*} \right)}\tr(\Sigma_g) \leq \frac{1}{2} \tr\left( {H^*}^{-1} \Sigma_g\right) \leq \frac{1}{2\lambda_n \left({H^*} \right)}\tr(\Sigma_g) \ep{.}
\end{align*}

Meanwhile, by Lemma \ref{lemma:wang_basic_bound}
\begin{align*}
\tr\left(\left(I - \frac{\lambda_n({H^*})}{2} {H^*}^{-1} \right)^{-1} \Sigma_g \right) &\geq \lambda_n\left(\left(I - \frac{\lambda_n({H^*})}{2} {H^*}^{-1} \right)^{-1}\right) \tr(\Sigma_g) \\
&= \frac{1}{\lambda_1\left(I - \frac{\lambda_n({H^*})}{2} {H^*}^{-1} \right)}\tr(\Sigma_g) \\
&= \frac{1}{1 - \frac{\lambda_n({H^*})}{2} \lambda_n({H^*}^{-1})}\tr(\Sigma_g) \\
&= \frac{1}{1-\frac{1}{2\kappa(H^*)}}\tr(\Sigma_g) \geq \tr(\Sigma_g) \ep{,}
\end{align*}
where $\kappa(H^*) = \lambda_1(H^*) / \lambda_n(H^*)$ is the condition number of $H^*$.  Similarly, by Lemma \ref{lemma:wang_basic_bound} we have 
\begin{align*}
\tr\left(\left(I - \frac{\lambda_n({H^*})}{2} {H^*}^{-1} \right)^{-1} \Sigma_g \right) &\leq \lambda_1\left(\left(I - \frac{\lambda_n({H^*})}{2} {H^*}^{-1} \right)^{-1}\right) \tr(\Sigma_g) \\
&= \frac{1}{\lambda_n\left(I - \frac{\lambda_n({H^*})}{2} {H^*}^{-1} \right)}\tr(\Sigma_g) \\
&= \frac{1}{1 - \frac{\lambda_n({H^*})}{2} \lambda_1({H^*}^{-1})}\tr(\Sigma_g) \\
&= \frac{1}{1-\frac{\lambda_n({H^*})}{2\lambda_n({H^*})}}\tr(\Sigma_g) = 2\tr(\Sigma_g) \ep{,}
\end{align*}
and thus
\begin{align*}
\frac{1}{4\lambda_n({H^*})} \tr(\Sigma_g) \leq \frac{1}{4\lambda_n({H^*})} \tr\left(\left(I - \frac{\lambda_n({H^*})}{2} {H^*}^{-1} \right)^{-1} \Sigma_g \right) \leq \frac{1}{2\lambda_n({H^*})} \tr(\Sigma_g) \ep{.}
\end{align*}

%Notably, because these bounds are all derived from Lemma \ref{lemma:wang_basic_bound}, they are tight in the sense that there exists examples the realize the upper and lower bounds with zero slack in each case.

\section{Some Self-contained Technical Results}

% We define matrix equations of the form
% \begin{align*}
% A^\top P + PA + Q = 0 \ep{.}
% \end{align*}
% as Continuous Algebraic Lyapunov Equations (CALEs) [[\textbf{but are we sure this is the def?}]], whenever $Q$ is PSD. The control theory community has developed efficient algorithms for solving such equations \citep[e.g][]{bartels1972solution}. 

% For our purposes we only need to compute certain traces associated with CALEs, which we can do via the following lemma.
\begin{lemma}
\label{lemma:pseudo-CALE}
Suppose $A^\top P + PA + Q = 0$ is a matrix equation where $P$ is invertible. Then we have
\begin{align*}
\tr(A) = -\frac{1}{2}\tr(P^{-1}Q) \ep{.}
\end{align*}
\end{lemma}
\begin{proof}
Pre-multiplying both sides of $A^\top P + P A + Q = 0$ by $P^{-1}$ and taking the trace yields $\tr(P^{-1} A^\top P ) + \tr(A) + \tr(P^{-1}Q) = 0$.  Then noting that $\tr(P^{-1} A^\top P ) = \tr( P P^{-1} A^\top ) = \tr( A^\top ) = \tr( A )$ this becomes $2\tr(A) + \tr(P^{-1}Q) = 0$, from which the claim follows.
\end{proof}

\begin{lemma}[Adapted from Lemma 1 from \citet{wang1986trace}]
\label{lemma:wang_basic_bound} 
Suppose $X$ and $S$ are $n\times n$ matrices such that $S$ is symmetric and $X$ is PSD.  Then we have
\begin{align*}
\lambda_n(S) \tr(X) \leq \tr(SX) \leq \lambda_1(S) \tr(X) \ep{.}
\end{align*}
%Moreover, these bounds are tight in the sense that for any such $X$, there exists example $S$'s for which either the upper or lower bounds are become equalities.
\end{lemma}

\begin{corollary}
\label{cor:wang_basic_bound}
Suppose $X$ and $S$ are $n\times n$ matrices such that $S$ is symmetric and $X$ is negative semi-definite (NSD).  Then we have
\begin{align*}
\lambda_1(S) \tr(X) \leq \tr(SX) \leq \lambda_n(S) \tr(X) \ep{.}
\end{align*}
\end{corollary}

\begin{proof}
Because $X$ is NSD, $-X$ is PSD. We can therefore apply Lemma \ref{lemma:wang_basic_bound} to get that
\begin{align*}
\lambda_n(S) \tr(-X) \leq \tr(S(-X)) \leq \lambda_1(S) \tr(-X) \ep{,}
\end{align*}
or in other words
\begin{align*}
-\lambda_n(S) \tr(X) \leq -\tr(SX) \leq -\lambda_1(S) \tr(X) \ep{.}
\end{align*}
Multiplying by $-1$ this becomes
\begin{align*}
\lambda_n(S) \tr(X) \geq \tr(SX) \geq \lambda_1(S) \tr(X) \ep{.}
\end{align*}

\end{proof}

\begin{lemma}
\label{lemma:PSDtrace_bound}
If $A$, $S$, $T$, and $X$ are matrices such that $A$, $S$ and $T$ commute with each other, $S \preceq T$ (i.e. $T-S$ is PSD), and $A$ and $X$ are PSD, then we have
\begin{align*}
\tr( ASX ) \leq \tr( ATX ) \ep{.}
\end{align*}
\end{lemma}
\begin{proof}
Since $A$, $S$ and $T$ are commuting PSD matrices they have the same eigenvectors, as does $A^{1/2}$ (which thus also commutes).

Meanwhile, $S \preceq T$ means that $T-S$ is PSD, and thus so is $A^{1/2} (T-S) A^{1/2}$.  Because the trace of the product of two PSD matrices is non-negative (e.g. by Lemma \ref{lemma:wang_basic_bound}), it follows that $\tr( (A^{1/2} (T-S) A^{1/2}) X) \geq 0$.  Adding $\tr(A^{1/2} S A^{1/2} X)$ to both sides of this we get $\tr( A^{1/2} T A^{1/2} X) \geq \tr(A^{1/2} S A^{1/2} X)$.  Because $A^{1/2}$ commutes with $T$ and $S$ we have $\tr( A^{1/2} T A^{1/2} X) = \tr( A T X)$ and $\tr( A^{1/2} S A^{1/2} X) = \tr( A S X)$, and so the result follows.
\end{proof}

\begin{lemma}
\label{lem:op_props}
Suppose $D$ is a matrix with real eigenvalues bounded strictly between $-1$ and $1$.  Define the operator $\Phi(X) = X - D X D^\top$. Then $\Phi$ has positive eigenvalues and is thus invertible.  Moreover, we have
\begin{align*}
\Phi^{-1}(X) = \sum_{i=0}^\infty D^i X (D^i)^\top \ep{.}
\end{align*}
And so if $X$ is a PSD matrix then $\Phi^{-1}(X)$ is as well.
\end{lemma}

\begin{proof}
The linear operator $\Phi$ can be expressed as a matrix using Kronecker product notation as $I - D \otimes D$. See \citet{van2000ubiquitous} for a discussion of Kronecker products and their properties. 

Because Kronecker products respect eigenvalue decompositions, the eigenvalues of $D \otimes D$ are given by $\{ \lambda_i(D) \lambda_j(D) \: | \: 0 \leq i, j\leq n\}$. By hypothesis, the eigenvalues of $D$ are real and bounded strictly between $-1$ and $1$, and it therefore follows that the eigenvalues of $D \otimes D$ have the same property.  From this it immediately follows that the eigenvalues of $I - D \otimes D$ are all $> 0$, and thus $I - D \otimes D$ is invertible.

Moreover, because of these bounds on the eigenvalues for $D \otimes D$, we have
\begin{align*}
\left(I - D \otimes D\right)^{-1} = \sum_{i=0}^\infty (D \otimes D)^i = \sum_{i=0}^\infty (D^i \otimes D^i) \ep{.}
\end{align*}
Translating back to operator notation this is
\begin{align*}
\Phi^{-1}(X) = \sum_{i=0}^\infty D^i X (D^i)^\top \ep{.}
\end{align*}

For any PSD matrix $X$ this is a sum (technically a convergent series) of self-evidently PSD matrices, and is therefore PSD itself.

\end{proof}

\begin{proposition}
\label{prop:harmonic}
Let $H_n$ be the $n$-th Harmonic number, defined by $H_n = \sum_{i=1}^n \frac{1}{i}$.  For any integers $n_1 \geq n_2 \geq 1$ we have
\begin{align*}
    H_{n_1} - H_{n_2} \geq \log(n_1) - \log(\min \{n_2 + 1, n_1\}) \ep{.}
\end{align*}
\end{proposition}

\begin{proof}
An inequality for $H_n$ due to \citet{young} is
\begin{align*}
\log(n) + \gamma + \frac{1}{2(n+1)} \leq H_n \leq \log(n) + \gamma + \frac{1}{2n} \ep{,}
\end{align*}
where $\gamma$ is the Euler-Mascheroni constant.

In particular, we have
\begin{align*}
H_{n_1} \geq \log(n_1) + \gamma + \frac{1}{2 (n_1+1)} \geq \log(n_1) + \gamma \ep{,}
\end{align*}
and
\begin{align*}
H_{n_2 + 1} \leq \log(n_2+1) + \gamma + \frac{1}{2 (n_2+1)} \leq \log(n_2+1) + \gamma + \frac{1}{n_2+1} \ep{,}
\end{align*}
which implies
\begin{align*}
H_{n_2} \leq \log(n_2+1) + \gamma \ep{.}
\end{align*}

Taking the difference of the two inequalities yields
\begin{align*}
H_{n_1} - H_{n_2} \geq \log(n_1) - \log(n_2 + 1) \ep{.}
\end{align*}

Noting that $\min \{n_2 + 1, n_1\} = n_1$ if and only if $n_1 = n_2$, and that in such a case we have $H_{n_1} - H_{n_2} = 0$, the result follows.
\end{proof}

\begin{proposition}
\label{prop:prod_ineq}
Suppose $0 \leq i \leq k-1$ for integers $i$ and $k$, and $b$ is a non-negative real number. 

For any non-negative integer $i$ such that $b \leq i+a+1$ we have
\begin{align*}
    \prod_{j=i}^{k-1} \left(1 - \frac{b}{j+a+1} \right) \leq \left(\frac{i+a+1}{k+a}\right)^b \ep{.}
\end{align*}
And for $b \leq a+2$ we have
\begin{align*}
\sum_{i=0}^{k-1} \frac{1}{(i+a)(i+a+1)^2} &\prod_{j=i+1}^{k-1} \left(1 - \frac{b}{j+a+1} \right)^2 \leq \frac{\nu(a) k}{(k+a)^3} \ep{,}
\end{align*}
where $\nu(a) = (a+2)^3/(a(a+1)^2)$.
\end{proposition}

\begin{proof}
It is a well-known fact that for $0 \leq y \leq 1$
\begin{align*}
    1 - y \leq \exp(-y) \ep{.}
\end{align*}
For all $j \geq i$ we have $0\leq \frac{b}{j+a+1} \leq 1$ (since $0 \leq b \leq i+a+1 \leq j+a+1$), and so
\begin{align*}
1 - \frac{b}{j+a+1} \leq \exp\left(-\frac{b}{j+a+1}\right) \ep{.}
\end{align*}

From this inequality and Proposition \ref{prop:harmonic} it follows that
\begin{align*}
    \prod_{j=i}^{k-1} \left(1 - \frac{b}{j+a+1} \right) &\leq \prod_{j=i}^{k-1} \exp\left(-\frac{b}{j+a+1}\right) \\
    &= \exp\left(-b \sum_{j=i}^{k-1} \frac{1}{j+a+1} \right) \\
    &= \exp\left(-b (H_{k+a} - H_{i+a}) \right) \\
    &\leq \exp\left(-b \left(\log(k+a) - \log(\min \{i+a+1, k+a\})\right) \right) \\
    &= \left(\frac{\min \{i+a+1, k+a\}}{k+a}\right)^b \\
    &\leq \left(\frac{i+a+1}{k+a}\right)^b \ep{.}
\end{align*}

Squaring both sides of the penultimate version of the above inequality it follows that
\begin{align*}
\sum_{i=0}^{k-1} \frac{1}{(i+a)(i+a+1)^2} &\prod_{j=i+1}^{k-1} \left(1 - \frac{b}{j+a+1} \right)^2 \\
&\leq \sum_{i=0}^{k-1} \frac{1}{(i+a)(i+a+1)^2} \left(\frac{\min\{i+a+2, k+a\}}{k+a}\right)^{2b} \\
&= \frac{1}{(k+a)^{2b}} \sum_{i=0}^{k-1} \frac{\left(\min\{i+a+2, k+a\}\right)^{3}}{(i+a)(i+a+1)^2} \left(\min\{i+a+2, k+a\}\right)^{2b-3} \\
&\leq \frac{1}{(k+a)^{2b}} \sum_{i=0}^{k-1} \nu(a) \left(\min\{i+a+2, k+a\}\right)^{2b-3} \\
&\leq \frac{\nu(a)}{(k+a)^{2b}} \sum_{i=0}^{k-1} (k+a)^{2b-3} = \frac{\nu(a)}{(k+a)^{2b}} k \cdot (k+a)^{2b-3} = \frac{\nu(a) k}{(k+a)^3} \ep{,}
\end{align*}
where $\nu(a) = (a+2)^3/(a(a+1)^2)$, which is the second claimed inequality.
\end{proof}

\section{Proof of Corollary \ref{cor:invar}}
\label{app:cor_invar_proof}

\begin{corollary*}
Suppose that $B_\theta$ and $B_\gamma$ are invertible matrices satisfying
\begin{align*}
J_\zeta^\top B_\theta J_\zeta = B_\gamma
\end{align*}
for all values of $\theta$.  Then the path followed by an iterative optimizer working in $\theta$-space and using additive updates of the form $d_\theta = -\alpha B_\theta^{-1} \nabla h$ is the same as the path followed by an iterative optimizer working in $\gamma$-space and using additive updates of the form $d_\gamma = -\alpha B_\gamma^{-1} \nabla_{\gamma} h$, provided that the optimizers use equivalent starting points (i.e. $\theta_0 = \zeta(\gamma_0)$), and that either
\begin{itemize}
    \item $\zeta$ is affine,
    \item or $d_\theta / \alpha$ is uniformly continuous as a function of $\theta$, $d_\gamma / \alpha$ is uniformly bounded (in norm), there is a $C$ as in the statement of Theorem \ref{thm:invar}, and $\alpha \to 0$.
\end{itemize}
Note that in the second case we allow the number of steps in the sequences to grow proportionally to $1/\alpha$ so that the continuous paths they converge to have non-zero length as $\alpha \to 0$.
\end{corollary*}

\begin{proof}
In the case where $\zeta$ is affine the result follows immediately from Theorem \ref{thm:invar}, and so it suffices to prove the second case.

We will denote by $\theta_0, \theta_1, ...$, and $\gamma_0, \gamma_1, ...$ the sequences of iterates produced by each optimizer. Meanwhile, $d_{\theta_0}, d_{\theta_1}, ...$ and $d_{\gamma_0}, d_{\gamma_1}, ...$ will denote the sequences of their updates.  %By hypothesis, the only relationship we have  between these two is that $\theta_0 = \zeta(\gamma_0)$.

By the triangle inequality
\begin{align*}
\|\zeta(\gamma_{k+1}) - \theta_{k+1}\| &= \|\zeta(\gamma_k + d_{\gamma_k}) - (\theta_k + d_{\theta_k})\| \\
&= \|\zeta(\gamma_k + d_{\gamma_k}) - (\zeta(\gamma_k) + d_{\zeta(\gamma_k)}) + (\zeta(\gamma_k) + d_{\zeta(\gamma_k)}) - (\theta_k + d_{\theta_k})\| \\
&= \|\zeta(\gamma_k + d_{\gamma_k}) - (\zeta(\gamma_k) + d_{\zeta(\gamma_k)}) + (\zeta(\gamma_k) - \theta_k) + (d_{\zeta(\gamma_k)} - d_{\theta_k})\| \\
&\leq \|\zeta(\gamma_k + d_{\gamma_k}) - (\zeta(\gamma_k) + d_{\zeta(\gamma_k)}) \| + \|\zeta(\gamma_k) - \theta_k\| + \|d_{\zeta(\gamma_k)} - d_{\theta_k}\| \ep{.}
\end{align*}

By Theorem \ref{thm:invar}, we can upper bound the first term on the RHS by $\frac{1}{2} C \sqrt{n} \|d_{\gamma_k}\|^2$. Using the hypothesis $\|d_\gamma / \alpha\| \leq D$ for all $\gamma$ for some universal constant $D$, this is further bounded by $\frac{1}{2} \alpha^2 C D^2 \sqrt{n} \equiv \alpha^2 E$, where $E$ is a universal constant. And by the hypothesized uniform continuity of $d_\theta / \alpha$ (as a function of $\theta$) there exists a universal constant $U$ such that $\|d_{\zeta(\gamma_k)}/\alpha - d_{\theta_k}/\alpha\| \leq U \|\zeta(\gamma_k) - \theta_k\|$, which gives a bound of $\alpha U \|\zeta(\gamma_k) - \theta_k\|$ on the third term.

In summary, we now have
\begin{align}
\label{eqn:it_diff_bound}
\|\zeta(\gamma_{k+1}) - \theta_{k+1}\| \leq \alpha^2 E + \|\zeta(\gamma_k) - \theta_k\| + \alpha U \|\zeta(\gamma_k) - \theta_k\| = \alpha^2 E + (1 + \alpha U)\|\zeta(\gamma_k) - \theta_k\| \ep{.}
\end{align}

Starting from $\|\zeta(\gamma_0) - \theta_0\| = 0$ (which is true by hypothesis) and applying this formula recursively, we end up with the geometric series formula
\begin{align*}
\|\zeta(\gamma_k) - \theta_k\| \leq \alpha^2 E \sum_{i=0}^{k-1} (1+\alpha U)^i = \alpha^2 E \left(\frac{(1+\alpha U)^k - 1}{\alpha U} \right) \ep{.}
\end{align*}

Because each step scales as $\alpha$, sequences of length $T/\alpha$ will converge to continuous paths of a finite non-zero length (that depends on $T$) as $\alpha \to 0$. Noting that $\lim_{\alpha \to 0} (1+\alpha U)^{T/\alpha} = \exp(UT)$ (which is a standard result), it follows that the RHS of eqn.~\ref{eqn:it_diff_bound} converges to zero as $\alpha \to 0$ for $k = T/\alpha$, and indeed for all natural numbers $k \leq T/\alpha$. Thus we have for all $k \leq T/\alpha$
\begin{align*}
\lim_{\alpha \to 0} \|\zeta(\gamma_k) - \theta_k\| = 0 \ep{,}
\end{align*}
which completes the proof.

\end{proof}

\bibliography{natural_grad_18}
%\addcontentsline{toc}{chapter}{Bibliography}

%\bibliographystyle{plainnat}
%\bibliographystyle{plainnat-eprints}
%\bibliographystyle{plainnat-eprints-authabrv}
%\bibliographystyle{abbrevnat} 

\end{document}